\documentclass[nohyperref]{article}
\usepackage{microtype}
\usepackage{graphicx}
\usepackage{subfigure}
\usepackage{booktabs}
\usepackage{hyperref}

\usepackage[accepted]{icml2022}

\usepackage{amsmath}
\usepackage{amssymb}
\usepackage{mathtools}
\usepackage{amsthm}
\usepackage{nicefrac}
\usepackage[capitalize,noabbrev]{cleveref}
\usepackage{tikz}
\newcommand*\circled[1]{\tikz[baseline=(char.base)]{
            \node[shape=circle,draw,inner sep=.2pt] (char) {#1};}}
            
\theoremstyle{plain}
\newtheorem{theorem}{Theorem}
\newtheorem{proposition}[theorem]{Proposition}
\newtheorem{lemma}[theorem]{Lemma}
\newtheorem{corollary}[theorem]{Corollary}
\theoremstyle{definition}
\newtheorem{definition}[theorem]{Definition}

\theoremstyle{remark}
\newtheorem{remark}[theorem]{Remark}
\newcommand{\notice}[1]{{\textcolor{red}{#1}}}

\usepackage[textsize=tiny]{todonotes}
\newif\ifspacehack
\usepackage{natbib}
\hypersetup{
    colorlinks = blue,
    breaklinks,
    linkcolor = blue,
    citecolor = blue,
    urlcolor  = blue,
}
\usepackage{url} 
\usepackage{graphicx}
\usepackage{mathtools}
\usepackage{footnote}
\usepackage{float}
\usepackage{xspace}
\usepackage{multirow}
\usepackage{xcolor}
\usepackage{wrapfig}
\usepackage{framed}
\usepackage{bbm}
\usepackage{footnote}
\usepackage{nicefrac}
\usepackage{makecell}
\usepackage[algo2e]{algorithm2e} 
\usepackage{algorithm}
\usepackage{amssymb}
\usepackage{amsthm}
\usepackage{bm}
\makesavenoteenv{tabular}
\makesavenoteenv{table}

\newcommand\innerp[2]{\langle #1, #2 \rangle}
\renewcommand{\tilde}{\widetilde}
\renewcommand{\hat}{\widehat}

\newcommand{\nereg}{\text{\rm NE-Reg}_T}
\newcommand{\iregx}{\text{\rm Reg}_T^x}
\newcommand{\iregy}{\text{\rm Reg}_T^y}
\newcommand{\dnereg}{\text{\rm DynNE-Reg}_T}
\newcommand{\gap}{\text{\rm Dual-Gap}_T}
\newcommand{\varNE}{P}
\newcommand{\varA}{V}
\newcommand{\variaA}{W}

\def \R {\mathbb{R}}

\newcommand{\calA}{{\mathcal{A}}}
\newcommand{\calB}{{\mathcal{B}}}
\newcommand{\calX}{{\mathcal{X}}}
\newcommand{\calS}{{\mathcal{S}}}

\newcommand{\calY}{{\mathcal{Y}}}

\newcommand{\Reg}{\text{\rm Reg}}

\newcommand{\p}{\prime}

\DeclareMathOperator*{\argmin}{argmin}
\DeclareMathOperator*{\argmax}{argmax}

\newcommand{\inner}[1]{ \left\langle {#1} \right\rangle }

\newcommand{\norm}[1]{\left\|{#1}\right\|}

\newcommand{\wh}{\widehat}
\newcommand{\wt}{\widetilde}

\newcommand{\order}{\ensuremath{\mathcal{O}}}
\newcommand{\otil}{\ensuremath{\tilde{\mathcal{O}}}}

\usepackage{lipsum,booktabs}
\usepackage{amsmath,mathrsfs,amssymb,amsfonts,bm}
\usepackage{rotating}
\usepackage{pdflscape}
\usepackage{hyperref,url}
\hypersetup{
    colorlinks,
    breaklinks,
    linkcolor = blue,
    citecolor = blue,
    urlcolor  = blue,
}
\allowdisplaybreaks
\usepackage{appendix}
\usepackage{multirow,makecell,tabularx}

\usepackage{algorithmic,algorithm}
\renewcommand{\algorithmicrequire}{ \textbf{Input:}}
\renewcommand{\algorithmicensure}{ \textbf{Output:}}

\renewcommand{\tilde}{\widetilde}
\renewcommand{\hat}{\widehat}
\newcommand \term[1]{\mathtt{term}~(\mathtt{#1})}

\def \H {\mathcal{H}}
\def \I {\mathcal{I}}

\def \O {\mathcal{O}}

\def \R {\mathbb{R}}

\def \T {\top}

\def \X {\mathcal{X}}
\def \Y {\mathcal{Y}}

\def \p {\mathbf{p}}

\def \u {\mathbf{u}}

\def \v {\mathbf{v}}
\def \x {\mathbf{x}}
\def \y {\mathbf{y}}

\def \xh {\hat{\x}}

\def \Ot {\tilde{\O}}

\usepackage{mathtools}
\let\norm\undefined 
\DeclarePairedDelimiter\norm{\lVert}{\rVert}
\DeclarePairedDelimiter\abs{\lvert}{\rvert}

\DeclareMathOperator*{\poly}{poly}

\newcommand \Div[2]{D_{\psi}(#1, #2)}

\usepackage{amsthm}

\theoremstyle{definition}

\usepackage{prettyref}
\newcommand{\pref}[1]{\prettyref{#1}}

\newcommand{\savehyperref}[2]{\texorpdfstring{\hyperref[#1]{#2}}{#2}}
\newrefformat{eq}{\savehyperref{#1}{Eq.~\textup{(\ref*{#1})}}}
\newrefformat{eqn}{\savehyperref{#1}{Eq.~(\ref*{#1})}}
\newrefformat{lem}{\savehyperref{#1}{Lemma~\ref*{#1}}}
\newrefformat{lemma}{\savehyperref{#1}{Lemma~\ref*{#1}}}
\newrefformat{def}{\savehyperref{#1}{Definition~\ref*{#1}}}
\newrefformat{line}{\savehyperref{#1}{Line~\ref*{#1}}}
\newrefformat{thm}{\savehyperref{#1}{Theorem~\ref*{#1}}}
\newrefformat{corr}{\savehyperref{#1}{Corollary~\ref*{#1}}}
\newrefformat{cor}{\savehyperref{#1}{Corollary~\ref*{#1}}}
\newrefformat{sec}{\savehyperref{#1}{Section~\ref*{#1}}}
\newrefformat{app}{\savehyperref{#1}{Appendix~\ref*{#1}}}
\newrefformat{appendix}{\savehyperref{#1}{Appendix~\ref*{#1}}}
\newrefformat{assum}{\savehyperref{#1}{Assumption~\ref*{#1}}}
\newrefformat{ex}{\savehyperref{#1}{Example~\ref*{#1}}}
\newrefformat{fig}{\savehyperref{#1}{Figure~\ref*{#1}}}
\newrefformat{table}{\savehyperref{#1}{Table~\ref*{#1}}}
\newrefformat{alg}{\savehyperref{#1}{Algorithm~\ref*{#1}}}
\newrefformat{rem}{\savehyperref{#1}{Remark~\ref*{#1}}}
\newrefformat{remark}{\savehyperref{#1}{Remark~\ref*{#1}}}
\newrefformat{conj}{\savehyperref{#1}{Conjecture~\ref*{#1}}}
\newrefformat{prop}{\savehyperref{#1}{Proposition~\ref*{#1}}}
\newrefformat{proto}{\savehyperref{#1}{Protocol~\ref*{#1}}}
\newrefformat{prob}{\savehyperref{#1}{Problem~\ref*{#1}}}
\newrefformat{claim}{\savehyperref{#1}{Claim~\ref*{#1}}}
\newrefformat{que}{\savehyperref{#1}{Question~\ref*{#1}}}
\newrefformat{op}{\savehyperref{#1}{Open Problem~\ref*{#1}}}
\newrefformat{fn}{\savehyperref{#1}{Footnote~\ref*{#1}}}

\def \p {\boldsymbol{p}}

\def \drvu {\textsc{DRVU}}
\def \epsilon {\varepsilon}
\def \cthree {c}

\def \base {\mathtt{base}\mbox{-}\mathtt{regret}}
\def \meta {\mathtt{meta}\mbox{-}\mathtt{regret}}

\def \p {p}

\def \x {x}
\def \y {y}
\def \u {u}
\def \v {v}
\usepackage{paralist}
\usepackage{threeparttable}

\icmltitlerunning{No-Regret Learning in Time-Varying Zero-Sum Games}

\begin{document}

\twocolumn[
\icmltitle{No-Regret Learning in Time-Varying Zero-Sum Games}

\icmlsetsymbol{equal}{*}

\begin{icmlauthorlist}
\icmlauthor{Mengxiao Zhang}{equal,usc}
\icmlauthor{Peng Zhao}{equal,nju}
\icmlauthor{Haipeng Luo}{usc}
\icmlauthor{Zhi-Hua Zhou}{nju}
\end{icmlauthorlist}

\icmlaffiliation{usc}{University of Southern California}
\icmlaffiliation{nju}{National Key Laboratory for Novel Software Technology, Nanjing University}

\icmlcorrespondingauthor{Mengxiao Zhang}{mengxiao.zhang@usc.edu}
\icmlcorrespondingauthor{Peng Zhao}{zhaop@lamda.nju.edu.cn}
\icmlkeywords{Online Learning, Zero-Sum Games, Dynamic Regret}

\vskip 0.3in
]

\printAffiliationsAndNotice{\icmlEqualContribution}

\begin{abstract}
Learning from repeated play in a fixed two-player zero-sum game is a classic problem in game theory and online learning.
We consider a variant of this problem where the game payoff matrix changes over time, possibly in an adversarial manner.
We first present three performance measures to guide the algorithmic design for this problem: 1) the well-studied \emph{individual regret}, 2) an extension of \emph{duality gap}, and 3) a new measure called \emph{dynamic Nash Equilibrium regret}, which quantifies the cumulative difference between the player's payoff and the minimax game value.
Next, we develop a single parameter-free algorithm that \emph{simultaneously} enjoys favorable guarantees under all these three performance measures.
These guarantees are adaptive to different non-stationarity measures of the payoff matrices and, importantly, recover the best known results when the payoff matrix is fixed. 
Our algorithm is based on a two-layer structure with a meta-algorithm learning over a group of black-box base-learners satisfying a certain property, along with several novel ingredients specifically designed for the time-varying game setting. 
Empirical results further validate the effectiveness of our algorithm.
\end{abstract}

\section{Introduction}
\label{sec: intro}
Repeated play in a fixed two-player zero-sum game, a fundamental problem in the interaction between game theory and online learning, has been extensively studied in recent decades. In particular, many efforts have been devoted to designing online algorithms such that both players achieve small individual regret (that is, difference between one's cumulative payoff and that of the best fixed action) while at the same time the dynamics of the players' strategy leads to a Nash equilibrium, a pair of strategies that neither player has incentive to deviate from; see for example~\citep{1999:Schapire-game-theory,conf/nips/RakhlinS13,SODA'11:fast-games,NIPS'15FastConv,NIPS'20:faster-hedge,ICLR'21-last-iteration,COLT'21:Hsieh,NIPS'21:general-game}.

In contrast to this large body of studies for learning over a \emph{fixed} zero-sum game,
repeated play over a sequence of \emph{time-varying} games, the focus of this paper and a ubiquitous scenario in practice, is much less explored.
While minimizing individual regret still makes perfect sense in this case,
it is not immediately clear what other desirable game-theoretic guarantees are that generalize the concept of approaching a Nash equilibrium when the game is fixed.
As far as we know, \citet{ICML'19:drift-game} are the first to explicitly consider this problem. They proposed the notion of Nash-Equilibrium regret (NE-regret) as the performance measure, which quantifies the difference between the learners' cumulative payoff and the minimax value of the cumulative payoff matrix. 
The authors proposed an algorithm with $\Ot(\sqrt{T})$ NE-regret after $T$ rounds of play and, importantly, proved that no algorithm can simultaneously achieve sublinear NE-regret and sublinear individual regret for both players.

Our work starts by questioning whether the NE-regret of~\citep{ICML'19:drift-game} is indeed a good performance measure for the problem of learning in time-varying games, especially given its incompatibility with the arguably most standard goal of having small individual regret.
We then discover that measuring performance with NE-regret can in fact be highly unreasonable: we show an example (in~\pref{sec:measure}) where even the two players perform perfectly (in the sense that they play the corresponding Nash equilibrium in every round), the resulting NE-regret is still \emph{linear} in $T$!

\begin{table*}[!t]
\centering
\caption{Summary of our results. The first column indicates the three performance measures considered in this work. The second column presents our main results for a sequence of time-varying payoff matrices $\{A_t\}_{t=1}^T$, and notably all results are simultaneously achieved by one single parameter-free algorithm. These guarantees are expressed in terms of three different (unknown) non-stationarity measures of the payoff matrices: $P_T, V_T$, and $W_T$, all of which are $\Theta(T)$ in the worst-case and zero in the stationary case (when $A_t = A$ for all $t \in [T]$);
see \pref{sec:measure} for definitions. Additionally, for duality gap we use notation $Q_T$ as a shorthand for $V_T + \min\{P_T, W_T\}$. Substituting all non-stationarity measures with zero leads to our corollaries for the stationary case shown in the third column, which match the state-of-the-art for all three performance measures (up to logarithmic factors) as shown in the last column.} \vspace{2mm}
\label{table:result-overview}
\renewcommand*{\arraystretch}{1.25}
\resizebox{0.98\textwidth}{!}{
\begin{tabular}{|c|c|c|c|}\hline
\textbf{Measure} &  \multicolumn{1}{c|}{\textbf{Time-Varying Game} ($\{A_t\}_{t=1}^T$, general)} & \multicolumn{2}{l|}{\textbf{Stationary Game} ($A_t = A$, fixed)} \\\hline
\multicolumn{1}{|c|}{\multirow{2}{*}{Individual Regret}} &  $\Ot\big(\sqrt{1 + V_T + \min\{P_T, W_T\}}\big)$	& $\Ot(1)$ & $\O(1)$ \\
\multicolumn{1}{|c|}{} & [\pref{thm:individual-regret}]   & [\pref{cor:individual-regret-stationary}]   & \citep{COLT'21:Hsieh}  \\ \hline
\multicolumn{1}{|c|}{\multirow{2}{*}{Dynamic NE-Regret}} & \multirow{2}{*}{}     $\Ot\big(\min\{\sqrt{(1+V_T)(1+P_T)} + P_T, 1 + W_T\}\big)$            & $\Ot(1)$ & $\O(1)$\\
\multicolumn{1}{|c|}{} &  [\pref{thm:dynamic-NE}]     & [\pref{cor:dynamic-NE-regret-stationary}]   & \citep{COLT'21:Hsieh}\footnotemark  \\ \hline
\multicolumn{1}{|c|}{\multirow{2}{*}{Duality Gap}} & \multirow{2}{*}{}     $\Ot\big(\min\{T^{\frac{3}{4}} ( 1 + Q_T)^{\frac{1}{4}}, T^{\frac{1}{2}}(1+Q_T^{\frac{3}{2}}+P_TQ_T)^{\frac{1}{2}}\}\big)$            & $\Ot(\sqrt{T})$  & $\O(\sqrt{T})$ \\
\multicolumn{1}{|c|}{} &   [\pref{thm:duality-gap-changing}]     & [\pref{cor:duality-gap-stationary}]   & \citep{ICLR'21-last-iteration}  \\\hline
\end{tabular}}
\end{table*}

Motivated by this observation, we revisit the basic problem of how to measure the algorithm's performance in such a time-varying game setting.
Concretely, we consider three performance measures that we believe are appropriate and natural:
1) the standard individual regret;
2) the direct generalization of cumulative duality gap from a fixed game to a varying game;
and 3) a new measure called \emph{dynamic NE-regret}, which quantifies the difference between the learner's cumulative payoff and the cumulative minimax game value (instead of the minimax value of the cumulative payoff matrix, as in NE-regret).
We argue that dynamic NE-regret is a better measure compared to NE-regret:
first, in the earlier example where both players play perfectly in each round using the corresponding Nash equilibrium, their dynamic NE-regret is exactly zero (while their NE-regret can be linear in $T$);
second, having small dynamic NE-regret does not prevent one from enjoying small individual regret or duality gap (as will become clear soon).

With these performance measures in mind, our main contribution is to develop \emph{one single parameter-free algorithm} that simultaneously enjoys favorable guarantees under all measures.
These guarantees are adaptive to some unknown non-stationarity measures of the payoff matrices --- naturally, the bounds worsen as the non-stationarity becomes larger.
More specifically, the individual regret is always at most $\Ot(\sqrt{T})$, the well-known worst-case bound, but could be much smaller if the non-stationarity measures are sublinear;
on the other hand, the duality gap and dynamic NE-regret are sublinear as long as the non-stationarity measures are sublinear.
In the special case of a fixed payoff matrix, all non-stationarity measures become zero and our results immediately recover the state-of-the-art results (up to logarithmic factors); see \pref{table:result-overview} for details.
Notice that the best known results for a fixed game are not necessarily achieved by the same algorithm, while again, our results are all achieved by one adaptive algorithm.
We also conduct empirical studies to further support our theoretical findings (\pref{appendix:experiments}).

\textbf{Techniques.}~~
For a fixed game, \citet{NIPS'15FastConv} proposed the ``Regret bounded by Variation in Utilities'' (RVU) property as the key condition for an algorithm to achieve good performance.
On the other hand, one of the key tools for achieving our results is to ensure a small gap between each player's cumulative payoff and that of a sequence of changing comparators, known as \emph{dynamic regret} in the literature~\citep{ICML'03:zinkvich}. 
Therefore, our first step is to generalize the RVU property to ``Dynamic Regret bounded by Variation in Utilities'' (\drvu) property, and to show that many existing algorithms indeed satisfy \drvu.

Furthermore, to achieve strong guarantees for all performance measures without any prior knowledge, we also need to deploy a two-layer structure, with a meta-algorithm learning over and combining decisions of a group of base-learners, each of which satisfies the DRVU property but uses a different step size.
Although such a framework has been used in many prior works in online learning (see for example the latest advances~\citep{COLT'21:impossible-tuning} and references therein), several new ingredients are required to achieve our results.
First, when updating the meta-algorithm, a correction term related to the stability of each base-algorithm is injected into the loss for the corresponding base-algorithm, which plays a key role in the analysis.
More specifically, we show (in \pref{lemma:stability-main}) an explicit bound for the stability of the meta-algorithm's decisions, whose proof requires a careful analysis using the correction terms above and the unique game structure.
Second, we also introduce a set of additional ``dummy'' base-algorithms that always play some fixed action.
This plays a key role in controlling the dynamics of the base-learners' outputs and turns out to be critical when bounding the duality gap.

\footnotetext{This is implicitly implied by the reuslts of \citet{COLT'21:Hsieh}, as our \pref{lemma:payoff-variation-conversation} shows that in the stationary case dynamic NE-regret is bounded by the individual regret.}

\textbf{Related Work.}~~ Two-player zero-sum game is one of the most fundamental problems in game theory, whose studies date back to the seminal work of~\citet{1928:von-Neumann}. \citet{1999:Schapire-game-theory} discovered the profound connections between zero-sum games and no-regret online learning, and since then there have been extensive studies on designing no-regret algorithms to solve games in the stationary setting~\citep{conf/nips/RakhlinS13,SODA'11:fast-games,NIPS'15FastConv,NIPS'20:faster-hedge,ICLR'21-last-iteration,NIPS'21:general-game}. We refer the reader to~\citep{NIPS'21:general-game} for a more thorough discussion on the literature. Several recent works start considering the problem of learning over a sequence of non-stationary payoffs under different structures, including zero-sum matrix games~\citep{ICML'19:drift-game,NIPS'21:periodic-game}, convex-concave games~\citep{Arxiv'19:online-saddle-point} and strongly monotone games~\citep{OR'21:time-varying-games}. For zero-sum games, \citep{NIPS'21:periodic-game} focuses on the periodic case and proves divergence results for a class of learning algorithms; \citep{ICML'19:drift-game} is the closest to our work, but as mentioned, we argue that their proposed measure (NE-regret) is not always appropriate (see~\pref{sec:measure-regret}).

\textbf{Organization.}~~We formally introduce the problem in \pref{sec:problem-setup}, then present the performance measures in~\pref{sec:measure} and our algorithm in~\pref{sec: algo overview}. Next, we provide theoretical guarantees in~\pref{sec: regret guarantee}, and conclude in~\pref{sec: conclusion}.

\section{Problem Setup and Notations}
\label{sec:problem-setup}
We consider the following problem of two players (called $x$-player and $y$-player) repeatedly playing a zero-sum game for $T$ rounds, with $m$ fixed actions for $x$-player and $n$ fixed actions for $y$-player.
At each round $t \in [T] \triangleq \{1, \ldots, T\}$, the environment first chooses a payoff matrix $A_t \in [-1,1]^{m\times n}$, whose $(i,j)$ entry denotes the loss/reward for $x$-player/$y$-player when they play action $i$ and action $j$ respectively.
Without knowing $A_t$, $x$-player ($y$-player) decides her own mixed strategy (that is, a distribution over actions) $x_t \in \Delta_m$ ($y_t \in \Delta_n$), where $\Delta_k$ denotes the probability simplex $\Delta_k = \{u \in \R_{\geq 0}^k \mid \sum_{i=1}^k u_i =1\}$. 
At the end of this round, $x$-player suffers expected loss $x_t^\T A_t y_t$ and observes the loss vector $A_t y_t$, while $y$-player receives the expected reward $x_t^\T A_t y_t$ and observes the reward vector $x_t^\T A_t$. Note that neither player observes the matrix $A_t$ itself. 

When $A_t$ is fixed for all $t$, this exactly recovers the standard stationary setting considered in for example~\citep{NIPS'15FastConv}.
Having a time-varying $A_t$ allows us to capture various possible sources of non-stationarity in the environment.
In fact, $A_t$ can even be decided by an adaptive adversary who makes the decision knowing the players' algorithm and their decisions in earlier rounds.
Our setting is almost the same as~\citep{ICML'19:drift-game}, except that the feedback they consider is either the entire matrix $A_t$ (stronger than ours) or just one entry of $A_t$ sampled according to $(x_t, y_t)$ (weaker than ours).

For each game matrix $A_t$, define the set of minimax strategies for $x$-player as $\calX_t^*=\argmin_{x\in \Delta_m}\min_{y\in \Delta_n}x^\top A_ty$ and similarly the set of maximin strategies for $y$-player as $\calY_t^*=\argmax_{y\in \Delta_n}\min_{x\in\Delta_m}x^\top A_ty$. It is well-known that any pair $(x_t^*, y_t^*)\in \calX_t^*\times \calY_t^*$ forms a Nash equilibrium of $A_t$ with the following saddle-point property: $x_t^{*\top} A_t y\leq x_t^{*\top} A_ty_t^*\leq x^\top A_ty_t^*$ holds for any $x\in \Delta_m$ and $y\in \Delta_n$.
Throughout the paper, $(x_t^*, y_t^*)$ denotes an arbitrary Nash equilibrium of $A_t$.

\textbf{Notations.}~~ For a real-valued matrix $A \in \R^{m\times n}$, its infinity norm is defined as $\|A\|_\infty \triangleq \max_{i,j}|A_{ij}|$. We use $\mathbf{1}_N$ and $\mathbf{0}_N$ to denote the all-one and all-zero vectors of length $N$. 
For conciseness, we often hide polynomial dependence on the size of the game (that is, $m$ and $n$) in the $\O(\cdot)$-notation.
The $\otil(\cdot)$-notation further omits logarithmic dependence on $T$. We sometimes write $\min_{x \in \Delta_m}$ ($\min_{y\in \Delta_n}$) simply as $\min_{x}$ ($\min_{y}$) when there is no confusion. 

\section{How to Measure the Performance?}
\label{sec:measure}

With the learning protocol specified, the next pressing question is to determine what the goal is when designing algorithms for the two players.
When $A_t$ is fixed, most studies consider minimizing individual regret for each player and some form of convergence to a Nash equilibrium of the fixed game as the two primary goals.
While minimizing individual regret is still naturally defined when $A_t$ is changing over time, it is less clear what other desirable game-theoretic guarantees are in this case.
In \pref{sec:measure-regret}, we formally discuss three performance measures that we think are reasonable for this problem.
Then in \pref{sec:measure-non-stationarity}, we further discuss how to measure the non-stationarity of the sequence $\{A_t\}_{1:T}$ that will play a role in how well the players can do under some of the performance measures.

\subsection{Performance Measures}
\label{sec:measure-regret}
\textbf{\circled{1} Individual Regret.}~ 
The first measure we consider is the standard individual regret.
For $x$-player, this is defined as
\begin{equation}
    \label{eq:individual-regret}
    \Reg_T^x \triangleq \sum_{t=1}^Tx_t^\top A_ty_t - \min_{x\in \Delta_m}\sum_{t=1}^Tx^\top A_ty_t,
\end{equation}
that is, the difference between her total loss and that of the best fixed strategy (assuming the same behavior from the opponent).
Similarly, the regret for $y$-player is defined as $\Reg_T^y \triangleq  \max_{y\in \Delta_n}\sum_{t=1}^Tx_t^\top A_ty - \sum_{t=1}^T x_t^\top A_ty_t$. 
Achieving sublinear (in $T$) individual regret implies that on average each player performs almost as well as their best fixed strategy,
and this is arguably the most standard and basic goal for online learning problems.

\textbf{\circled{2} Duality Gap.}~~ 
For a game matrix $A_t$, the duality gap of a pair of strategy $(x_t, y_t)$ is defined as $\max_{y\in\Delta_n}x_t^\top A_ty - \min_{x\in \Delta_{m}}x^\top A_ty_t$.
It is always nonnegative since $\max_{y\in\Delta_n}x_t^\top A_ty  \geq x_t^\top A_ty_t^* \geq x_t^{*\top} A_ty_t^* \geq x_t^{*\top} A_ty_t \geq \min_{x\in \Delta_{m}}x^\top A_ty_t$,
and it is zero if and only if $(x_t, y_t)$ is a Nash equilibrium of $A_t$.
Thus, the duality gap measures how close $(x_t, y_t)$ is to the equilibria in some sense.
We thus naturally use the cumulative duality gap:
\begin{equation}
    \label{eq:duality gap}
    \gap \triangleq \sum_{t=1}^T \Big(\max_{y\in\Delta_n}x_t^\top A_ty- \min_{x\in \Delta_{m}}x^\top A_ty_t \Big),
\end{equation}
as another performance measure.
When $A_t$ is fixed, this measure is considered in~\citep{ICLR'21-last-iteration} for example.

\textbf{\circled{3} Dynamic Nash Equilibrium (NE)-Regret.}~~ 
Before introducing this last measure, we first review what \citet{ICML'19:drift-game} proposed as the goal for this problem, that is, ensuring small Nash Equilibrium (NE)-regret, defined as
\begin{equation}
\label{eq:NE-regret}
    \nereg \triangleq \Bigg|\sum_{t=1}^Tx_t^\top A_ty_t - \min_{x\in \Delta_m}\max_{y\in \Delta_n}\sum_{t=1}^Tx^\top A_ty\Bigg|.
\end{equation}
In words, this is the difference between the cumulative loss of $x$-player (or equivalently the cumulative reward of $y$-player) and the minimax value of the cumulative payoff matrix ($\sum_{t=1}^T A_t$).
While this might appear to be a reasonable generalization of individual regret for a central controller who decides $x_t$ and $y_t$ jointly, we argue below that this measure is in fact often inappropriate for two reasons.

The first reason is in fact already hinted in~\citep{ICML'19:drift-game}: they proved that
no algorithm can always ensure sublinear NE-regret and simultaneously sublinear individual regret for both players.
Given that minimizing individual regret selfishly is a natural impulse and the standard goal for each player, NE-regret can only make sense when both players are controlled by a centralized algorithm.

The second reason is perhaps more profound.
Consider the following two-phase example: when $t \leq T/2$, $A_t = {\scriptsize \begin{pmatrix} 1 & -1 \\ -1 & 1\end{pmatrix}}$; when $t > T/2$, $A_t = {\scriptsize \begin{pmatrix} 1 & -1 \\ 1 & -1\end{pmatrix}}$.\footnote{The same example is in fact also used by~\citet{ICML'19:drift-game} to prove the incompatibility of individual regret and NE-regret.}
It is straightforward to verify that: when $t\leq T/2$, the equilibrium for $A_t$ is the uniform distribution for both players, leading to game value $0$;
when $t> T/2$, the equilibrium is such that $y$-player always picks the first column, leading to game value $1$;
and the equilibrium for the cumulative game matrix $\sum_{t=1}^T A_t = {\scriptsize \begin{pmatrix} T & -T \\ 0 & 0\end{pmatrix}}$ is $x$-player picking the second row while $y$-player picking the first column, leading to game value $0$.
To sum up, even if both players play perfectly in each round using the corresponding equilibrium, their NE-regret is still $|T/2-0|=T/2$, which is a vacuous bound linear in $T$! 

Motivated by the observations above, we propose a variant of NE-regret as the third performance measure, called \emph{dynamic NE-regret}:\footnotemark
\begin{equation*}
\dnereg \triangleq \Bigg| \sum_{t=1}^Tx_t^\top A_ty_t - \sum_{t=1}^T\min_{x\in \Delta_m}\max_{y\in \Delta_n}x^\top A_ty \Bigg|.
\end{equation*}
Compared to NE-regret, here we move the minimax operation inside the summation, making it the cumulative difference between $x$-player's loss and the minimax game value in each round. 
In other words, similarly to duality gap, dynamic NE-regret provides yet another way to measure in each round, how close $(x_t, y_t)$ is to the equilibria of $A_t$ from the game value perspective.

\footnotetext{In fact, a preprint by~\citet{Arxiv'19:online-saddle-point} also considers a similar measure for general convex-concave problem, but we believe that their results are incorrect. Specifically, they claim (in their Theorem~4.3) that an $\Ot(\sqrt{T})$ bound is always achievable for dynamic NE-regret, but this is clearly impossible because when $A_t$ always has identical columns (so $y$-player does not play any role), dynamic NE-regret becomes the dynamic regret~\citep{ICML'03:zinkvich} for $x$-player, which is well-known to be $\Omega(T)$ in the worst case.
}

The connection between NE-regret and Dynamic NE-regret is on the surface analogous to that between standard regret and dynamic regret~\citep{ICML'03:zinkvich} (see \pref{appendix:dynNE-NE} for definitions and more related discussions).
However, while dynamic regret is always no less than standard regret, \emph{Dynamic NE-regret could be smaller than NE-regret} --- simply consider our earlier two-phase example: the perfect players (who always play an equilibrium) clearly have $0$ dynamic NE-regret, but their NE-regret is $T/2$ as discussed.
This example also shows that dynamic NE-regret is more reasonable compared to NE-regret.
Moreover, as will become clear soon, dynamic NE-regret is compatible with individual regret (and also duality gap), in the sense there are algorithms that provably perform well under all these measures.

We conclude this section with the following two remarks.
\begin{remark}[Comparisons of the three measures]
\label{remark:measure}
Both individual regret and dynamic NE-regret are bounded by duality gap (see proofs in \pref{appendix:relation}), but the latter could be much larger. On the other hand, individual regret and dynamic NE-regret are generally incomparable.
\end{remark}

\begin{remark}[Other possibilities]
The three measures we consider are by no mean the only possibilities. 
Another reasonable one is the tracking error $\sum_{t=1}^T (\norm{x_t - x_{t}^*}_1 + \norm{y_t - y_{t}^*}_1)$ that directly measures the distance between $(x_t, y_t)$ and the equilibrium $(x_t^*, y_t^*)$ (assuming unique equilibrium for simplicity). 
This is considered in~\citep{Arxiv'19:online-saddle-point, balasubramanian2021zeroth} (for different problems).
However, tracking error bounds are in fact not well studied even when $A_t$ is fixed --- the best known results still depend on some problem-dependent constant that can be arbitrarily large~\citep{daskalakis2019last, ICLR'21-last-iteration}. 
Deriving tracking error bounds in our setting is thus beyond the scope of this paper.
Note that in many optimization studies, one often only cares about finding a point that is close to the optimal solution in terms of their function value instead of their absolute distance.
Our dynamic NE-regret and duality gap are both in this same sprite by looking at the game value instead of the actual distance as in tracking error.
\end{remark}

\subsection{Non-stationarity Measures}
\label{sec:measure-non-stationarity}
For duality gap and dynamic NE-regret, it is not difficult to see that if $A_t$ changes drastically over time, then no meaningful guarantees are possible.
This is similar to dynamic regret in standard online learning problems, where guarantees are always expressed in terms of some non-stationarity measure of the environment and are meaningful only when the non-stationarity is reasonably small.
In our setting, we consider the following three different ways to measure non-stationarity of the sequence $\{A_t\}_{t=1}^T$.

\textbf{Variation of Nash Equilibria.}~~Recall the notation $\calX_t^*\times \calY_t^*$, the set of Nash equilibria for matrix $A_t$.
Define the variation of Nash equilibria as: 
\begin{align*}
    \varNE_T\triangleq \min_{\forall t, (x_t^*,y_t^*) \in \calX_t^*\times \calY_t^*} \sum_{t=2}^T\left(\|x_t^*-x_{t-1}^*\|_1+\|y_t^*-y_{t-1}^*\|_1\right),
\end{align*}
which quantifies the drift of the Nash equilibria of the game matrices in $\ell_1$-norm.

\textbf{Variation/Variance of Game Matrices.}~~The path-length variation and the variance of $\{A_t\}_{t=1}^T$ are respectively defined as
\begin{align*}
    \varA_T \triangleq \sum_{t=2}^T\|A_t-A_{t-1}\|_{\infty}^2,~~~\variaA_T\triangleq \sum_{t=1}^T\|A_t-\bar{A}\|_{\infty},
\end{align*}
where $\bar{A}=\frac{1}{T}\sum_{t=1}^TA_t$ is the averaged game matrix. 

Clearly, $P_T$, $V_T$, and $W_T$ are all $\Theta(T)$ in the worst case, and $0$ when $A_t$ is fixed over time.
For dynamic regret and duality gap, the natural goal is to enjoy sublinear bounds whenever (some of) these non-stationarity measures are sublinear (which we indeed achieve).

We conclude by pointing out some connections between these non-stationarity measures.
First, $V_T \leq 4W_T$ holds but the former could be much smaller.
Second, $P_T$ is generally not comparable with $V_T$ and $W_T$, and there are examples where $P_T = 0$ and $V_T = W_T = \Theta(T)$, or $P_T = \Theta(T)$ and $V_T = W_T = \O(1)$.
We defer all details to~\pref{appendix:discussion-measure}.
\section{Proposed Algorithm}
\label{sec: algo overview}

In this section, we present our proposed algorithm for time-varying games, which provably achieves favorable guarantees under all three performance measures.
To illustrate the ideas behind our algorithm design, we first review how \citep{NIPS'15FastConv} achieves fast convergence results for a fixed game, followed by a detailed discussion on how to generalize their idea and overcome the difficulties brought by time-varying games.
For conciseness, throughout the section we focus on the $x$-player; how the $y$-player should behave is completely symmetric. 

For a fixed game $A_t = A$, \citet{NIPS'15FastConv} proposed that each player should deploy an online learning algorithm that satisfies a specific property called ``Regret bounded by Variation in Utilities'' (RVU).
More specifically, an online learning algorithm proposes $x_t \in \Delta_m$ at the beginning of round $t$, and then receives a loss vector $g_t \in \R^m$ and suffers loss $\inner{x_t, g_t}$.
Its regret against a comparator $u \in \Delta_m$ after $T$ rounds is naturally $\sum_{t=1}^T \inner{x_t - u, g_t}$, and the RVU property states that this should be bounded by
$\alpha + \beta \sum_{t=2}^T\|g_t-g_{t-1}\|_{\infty}^2 - \gamma\sum_{t=2}^T\|x_t-x_{t-1}\|_1^2$ for some parameters $\alpha, \beta, \gamma > 0$.\footnote{Without loss of generality, we here focus on $(\|\cdot\|_1,\|\cdot\|_{\infty})$ norm, and it is straightforward to generalize the argument to general primal-dual norm as in~\citep{NIPS'15FastConv}.}
To see why RVU property is useful, consider $x$-player deploying such an algorithm with $g_t$ set to $A_t y_t = A y_t$.
Then her regret is further bounded as
$\alpha  + \beta \sum_{t=2}^T \norm{Ay_t - Ay_{t-1}}_\infty^2 - \gamma \sum_{t=2}^T \norm{x_t - x_{t-1}}_1^2 \leq \alpha  + \beta \sum_{t=2}^T \norm{y_t - y_{t-1}}_1^2 - \gamma \sum_{t=2}^T \norm{x_t - x_{t-1}}_1^2$. 
Therefore, as long as $y$-player also deploys the same algorithm, by symmetry, the sum of their regret is at most $\alpha  + (\beta - \gamma) (\sum_{t=2}^T \norm{x_t - x_{t-1}}_1^2 + \sum_{t=2}^T \norm{y_t - y_{t-1}}_1^2)$, which can be simply bounded by (a constant) $\alpha$ as long as $\beta \leq \gamma$.
Many useful guarantees can then be obtained as a corollary of the fact that the sum of regret is small.

In our setting where $A_t$ is changing over time, our first observation is that instead of the sum of the two players' regret, what we need to control is the sum of their dynamic regret~\citep{ICML'03:zinkvich}, which plays an important role when deriving guarantees for all the three measures (\emph{including individual regret}).
Specifically, for an online learning algorithm producing $x_t$ and receiving $g_t$, its dynamic regret against a sequence of comparators $u_1, \ldots, u_T \in \Delta_m$ is defined as $\sum_{t=1}^T \inner{x_t - u_t, g_t}$.
Generalizing RVU, we naturally introduce the following ``Dynamic Regret bounded by Variation in Utilities'' (\drvu) property.
\begin{definition}[\drvu~Property]
    \label{def: drvu}
    Denote by $\calA(\eta)$ an online learning algorithm with a parameter $\eta >0$. We say that it satisfies the \emph{Dynamic Regret bounded by Variation in Utilities} property (abbreviated as \drvu($\eta$)) with parameters $\alpha, \beta, \gamma >0$, if its dynamic regret $\sum_{t=1}^T\inner{x_t-u_t,g_t}$ on any loss sequence $g_1,\ldots,g_T$ with respect to any comparator sequence $u_1,\ldots,u_T$ is bounded by 
    \begin{equation*}
        \frac{\alpha}{\eta} (1 + P_T^u) + \eta\beta \sum_{t=1}^T\|g_t-g_{t-1}\|_\infty^2 - \frac{\gamma}{\eta}\sum_{t=2}^T\|x_t-x_{t-1}\|_1^2,
    \end{equation*}
    where $P_T^u \triangleq \sum_{t=2}^T \norm{u_t - u_{t-1}}_1$ is the path-length of the comparator sequence. 
\end{definition}
Compared to RVU, \drvu\xspace naturally replaces the first constant term in the regret bound with a term depending on the path-length of the comparator sequence.
We also add another step size parameter $\eta$ (whose role will become clear soon).
Recent studies in dynamic regret~\citep{NIPS'20:sword,JMLR:sword++} show that variants of optimistic Online Mirror Descent (such as Optimistic Gradient Descent and Optimistic Hedge) indeed satisfy \drvu\xspace with $\alpha, \beta, \gamma = \tilde{\Theta}(1)$; see \pref{appendix:DRVU-property} for formal statements and proofs.

Now, if $x$-player deploys an algorithm satisfying \drvu\xspace and feeds it with loss vector $g_t = A_t y_t$ (and similarly $y$-player does the same), we can indeed prove a desired guarantee for each of the three performance measures.
However, the tuning of $\eta$ will require the knowledge of the unknown parameters $P_T, V_T, W_T$ and, perhaps more importantly, be different for each different measures.
To obtain an adaptive algorithm that performs well under all three measures without any prior knowledge, we further propose a two-layer structure with a meta-algorithm learning over and combining decisions of a set of base-learners, each of which satisfies \drvu($\eta$) but with a different step size $\eta$.
While this idea of ``learning over learning algorithms'' is not new in online learning, we will discuss below what extra difficulties show up in our case and how we address them.

\subsection{Base-learners}\label{sec:base-learners}
Define $N=\lfloor \frac{1}{2}\log_2 T \rfloor + 1$.
Our algorithm maintains $N+m$ base-learners:
for $i \in [N]$, the $i$-th base-learner is any algorithm that satisfies $\drvu(\eta_i^x)$,
where $\eta_i^x = \frac{2^{i-1}}{L\sqrt{T}}$ and 
\begin{equation}\label{eq:L_def}
L=\max\big\{4,\sqrt{16c\beta}, \sqrt{8c\beta/\gamma}\big\}
\end{equation}
($\beta$ and $\gamma$ are the parameters from \drvu\xspace and $c = \Ot(1)$ is a constant whose exact value can be found in the proof);
the last $m$ base-learners are dummy learners, with the $(j+N)$-th one always outputting the basis vector $e_j \in \Delta_m$ (that is, always choosing the $j$-th action).
We note that the dummy base-learners are important in controlling the duality gap (but not the other two measures).
We let $\calS_x\triangleq\calS_{1,x}\cup\calS_{2,x}$ with $\calS_{1,x} = [N]$ and $\calS_{2,x} = \{N+1, \ldots, N+m\}$ denote the set of indices of base-learners.

At round $t$, each base-learner $i$ submits her decision $x_{t,i} \in \Delta_m$ to the meta-algorithm, who decides the final decision $x_t$.
Upon receiving the feedback $A_t y_t$, the meta-algorithm sends the same (as the loss vector $g_t$) to each base-learner $i \in \calS_{1,x}$ (no updates needed for the dummy base-learners).

\subsection{Meta-algorithm}
With all the decisions $\{x_{t,i}\}_{i\in \calS_x}$ collected from the base-learners, the meta-algorithm outputs the final decision $x_t = \sum_{i\in \calS_x} p_{t,i} x_{t,i}$,\footnote{Note the slight abuse of notations here: while $p_{t,i}$ represents the $i$-th entry of vector $p_t$, $x_{t,i}$ is \emph{not} the $i$-th entry of $x_t$.} where $p_t \in \Delta_{|\calS_x|}$ is a distribution over the base-learners updated according to a version of Optimistic Online Gradient Descent (OOGD)~\citep{conf/nips/RakhlinS13,NIPS'15FastConv}:
\begin{equation}
\label{eq:optimistic-ogd}
\begin{split}
    p_{t} &= \argmin_{p \in \Delta_{|\calS_x|}} \left\{ \epsilon_t^x \langle p,m_t^x \rangle + \|p-\wh{p}_{t}\|_2^2\right\},\\
    \wh{p}_{t+1} &= \argmin_{p \in \Delta_{|\calS_x|}} \left\{ \epsilon_t^x \langle p,\ell_t^x \rangle + \|p-\wh{p}_t\|_2^2\right\}.
\end{split}
\end{equation}
Here, $\epsilon_t^x >0$ is a time-varying learning rate, $\{\wh{p}_{t}\}_{t=1,2,\ldots}$ is an auxiliary sequence (starting with $\wh{p}_{1}$ as the uniform distribution) updated via projected gradient descent using some loss vector sequence $\ell_1^x, \ell_2^x, \ldots \in\R^{|\calS_x|}$, and $p_t$ is updated via projected gradient descent from the distribution $\wh{p}_{t}$ and using a loss predictor $m_t^x \in \R^{|\calS_x|}$.
It remains to specify what $\ell_t^x$ and $m_t^x$ are (the tuning of the learning rate will be specified in the final algorithm).

\begin{algorithm}[!t]
     \caption{Algorithm for the $x$-player}
     \label{alg:x-player}
     \textbf{Input}: a base-algorithm $\mathcal{A}(\eta)$ satisfying DRVU($\eta$). 
     
     \textbf{Initialize}: 
     a set of base-learners $\calS_x$ as described in \pref{sec:base-learners},
     $\wh{p}_1 = \frac{1}{|\calS_x|} \mathbf{1}_{|\calS_x|}$, learning rate $\epsilon_1^x = \frac{1}{L}$ (\textit{c.f.} \pref{eq:L_def}).  
     
     \For{$t=1,\dots,T$}{
        Receive $x_{t,i}\in\Delta_m$ from each base-learner $i \in \calS_x$.
        
        Compute $m_t^x$ based on \pref{eq:correction-optimism} and $p_t$ based on \pref{eq:optimistic-ogd}.  
        
        Play the final decision $x_t = \sum_{i\in \calS_x} p_{t,i} x_{t,i}$.
        
        Suffer loss $x_t^\T A_t y_t$ and observe the loss vector $A_t y_t$.
        
        Compute $\ell_{t}^x$ based on \pref{eq:correction-loss} and $\wh{p}_{t+1}$ based on \pref{eq:optimistic-ogd}.  
        
        Update $\epsilon_{t+1}^x= \nicefrac{1}{\sqrt{L^2+\sum_{s=2}^t\|A_ty_t-A_{t-1}y_{t-1}\|_\infty^2}}$.
        
        Send $A_t y_t$ as the feedback to each base-learner.
     }
\end{algorithm}

Since base-learner $i$ predicts $x_{t,i}$ and receives loss vector $A_t y_t$,
it is natural to set its loss $\ell_{t,i}^x$ as $x_{t,i}^\T A_t y_t$ from the meta-algorithm's perspective. In light of standard OOGD, $m_t^x$ should then be set to $x_{t,i}^\T A_{t-1} y_{t-1}$, meaning that the last loss vector $A_{t-1} y_{t-1}$ is used to predict the current one (that is unknown yet when computing $p_t$).
However, this setup leads to the following issue.
When applying \drvu($\eta_i^x$) to this base-learner, we see that a negative term related to $\|x_{t,i}-x_{t-1,i}\|_1^2$ and a positive term related to $\|y_t-y_{t-1}\|_1^2$ arise (the latter is from $\|A_ty_t - A_{t-1} y_{t-1}\|_\infty^2 \leq 2 \norm{A_t - A_{t-1}}_{\infty}^2 + 2\|y_t-y_{t-1}\|_1^2$, with the first term only related to the non-stationarity of game matrices).
By symmetry, $y$-player contributes a positive term $\|x_t-x_{t-1}\|_1^2$, which now cannot be canceled by $\|x_{t,i}-x_{t-1,i}\|_1^2$, unlike the case with only one learner for each player discussed earlier.

To address this issue, we propose to add a stability correction term to both $\ell_t^x$ and $m_t^x$. Concretely, they are defined as $\ell_{1,i}^{x} = x_{1,i}^\T A_1 y_1$ and $m_{1,i}^x = 0, \forall i$, and for all $t \geq 2$:
\begin{align}
    \ell_{t,i}^{x} = {} & x_{t,i}^\T A_t y_t + \lambda \norm{x_{t,i} - x_{t-1,i}}_1^2,  \label{eq:correction-loss}\\
    m_{t,i}^{x} = {} & x_{t,i}^\T A_{t-1} y_{t-1} + \lambda \norm{x_{t,i} - x_{t-1,i}}_1^2,  \label{eq:correction-optimism}
\end{align}
where $\lambda = \frac{\gamma L}{2}$ ($\gamma$ is the parameter from \drvu).
From a technical perspective, this introduces to the regret a negative term $\sum_{i\in \calS_x} p_{t,i} \norm{x_{t,i} - x_{t-1,i}}_1^2$, and a positive term $\norm{x_{t,i} - x_{t-1,i}}_1^2$ which can be canceled by the aforementioned negative term from  \drvu($\eta_i^x$).
To see why the extra negative term is useful, notice that the troublesome term $\|x_t-x_{t-1}\|_1^2$ from \drvu($\eta_i^x$) can be bounded as 
\begin{align*}
    & \norm{x_t - x_{t-1}}_1^2 = \left\lVert\sum_{i\in \calS_x} p_{t,i} x_{t,i} - \sum_{i\in \calS_x} p_{t-1,i} x_{t-1,i}\right\rVert_1^2 \nonumber \\
    & \leq 2\sum_{i\in \calS_x} p_{t,i} \norm{x_{t,i} - x_{t-1,i}}_1^2 + 2\norm{p_t - p_{t-1}}_1^2,
\end{align*}
where the first term can exactly be canceled by the extra negative term introduced by the correction term, and the second term can in fact also be canceled in a standard way since the meta-algorithm itself can be shown to satisfy RVU.
This explains the design of our correction terms from a technical level.
Intuitively, injecting this correction term guides the meta-algorithm to bias toward the more stable base-learners, hence also stabilizing the final decision $x_t$. 

We note that a similar technique was used in analyzing gradient-variation dynamic regret for online convex optimization~\citep{JMLR:sword++}. Our approach is different from theirs in the sense that there is only one player in their setting and the correction term is used to cancel the additional gradient variation introduced by the variation of her own decision. In contrast, in our setting the correction term is used to cancel the opponent's gradient variation.

To summarize, our final algorithm (for the $x$-player) is presented in~\pref{alg:x-player}. We also include the symmetric version for the $y$-player in \pref{alg:y-player} (\pref{appendix:y-alg}) for completeness. We emphasize again that this is a parameter-free algorithm that does not require any prior knowledge of the environment.
\section{Theoretical Guarantees and Analysis}
\label{sec: regret guarantee}
In this section, we first provide the guarantees of our algorithm under each of the three performance measures, and then highlight several key ideas in the analysis, with the full proofs deferred to~\pref{appendix:main-proofs}.
Recall that our guarantees are all expressed in terms of the non-stationarity measures $P_T$, $V_T$, and $W_T$, defined in \pref{sec:measure-non-stationarity}.
Also, to avoid showing the cumbersome dependence on the \drvu\xspace parameters ($\alpha$, $\beta$, $\gamma$) in all our bounds, we will simply assume that they are all $\tilde{\Theta}(1)$, which, as mentioned earlier and shown in \pref{appendix:DRVU-property}, is indeed the case for standard algorithms.

\subsection{Performance Guarantees}
We state our results for each performance measure separately below, but emphasize again that they hold simultaneously.
First, we show the individual regret bound. 
\begin{theorem}[Individual Regret]
\label{thm:individual-regret}
When the $x$-player uses~\pref{alg:x-player}, irrespective of $y$-player's strategies, we have 
\[
    \Reg_T^x = \sum_{t=1}^Tx_t^\top A_ty_t - \min_x\sum_{t=1}^Tx^\top A_ty_t = \Ot(\sqrt{T}).
\]
Furthermore, if $x$-player follows~\pref{alg:x-player} and $y$-player follows~\pref{alg:y-player}, then individual regret satisfies:
\begin{align*}
  \max\left\{\Reg_T^x, \Reg_T^y\right\} = \Ot\Big(\sqrt{1+V_T+\min\{P_T,W_T\}}\Big).
\end{align*}
\end{theorem}

The first statement of \pref{thm:individual-regret} provides a robustness guarantee for our algorithm --- no matter how non-stationary the game matrices are and no matter how the opponent behaves, following our algorithm always ensures $\Ot(\sqrt{T})$ individual regret, the standard worst-case regret bound.
On the other hand, when both players follow our algorithm, their individual regret could be even smaller depending on the non-stationarity.
In particular, as long as $V_T+\min\{P_T,W_T\} = o(T)$ (that is, not the worst case scenario), our bound becomes $o(\sqrt{T})$.
Also note that $P_T$ and $W_T$ are generally incomparable (see~\pref{appendix:discussion-measure}), but our bound achieves the minimum of them, thus achieving the best of both worlds.

When the game matrix is fixed, we have $P_T=V_T=W_T=0$, immediately leading to the following corollary. 
\begin{corollary}
\label{cor:individual-regret-stationary}
When $x$-player follows~\pref{alg:x-player} and $y$-player follows~\pref{alg:y-player}, if $A_t = A$ for all $t \in [T]$, then $\max\left\{\Reg_T^x, \Reg_T^y\right\} =\Ot(1)$.
\end{corollary}
The best known individual regret bound for learning in a fixed two-player zero-sum game is $\O(1)$~\citep{COLT'21:Hsieh}. Our result matches theirs up to logarithmic factors. 

The next theorem presents the dynamic NE-regret bound.
\begin{theorem}[Dynamic NE-Regret]
\label{thm:dynamic-NE}
When $x$-player follows~\pref{alg:x-player} and $y$-player follows~\pref{alg:y-player}, we have the following dynamic NE-regret bound:
\begin{align*}
    & \dnereg = \left|\sum_{t=1}^Tx_t^\T A_ty_t-\sum_{t=1}^T\min_{x\in\Delta_m}\max_{y\in \Delta_n} x^\top A_t y\right| \\
    & = \Ot\big(\min\{\sqrt{(1+V_T)(1+P_T)} + P_T, 1 + W_T\}\big).
\end{align*}
\end{theorem}

Similarly, our dynamic NE-regret bound is $o(T)$ as long as $P_T$ \emph{or} $W_T$ is $o(T)$. When the game matrix is fixed, we again obtain the following direct corollary by noticing $P_T=V_T=W_T=0$ in this case.
\begin{corollary}
\label{cor:dynamic-NE-regret-stationary}
When $x$-player follows~\pref{alg:x-player} and $y$-player follows~\pref{alg:y-player}, if $A_t = A$ for all $t \in [T]$, then $\dnereg = \Ot(1)$.
\end{corollary}
In fact, when the game is fixed, dynamic NE-regret degenerates to NE-regret of~\citep{ICML'19:drift-game} as $\sum_{t=1}^T\min_{x}\max_{y}x^\top A y = \min_{x}\max_{y} \sum_{t=1}^Tx^\top A y$.
Their algorithm would achieve $\otil(\sqrt{T})$ (dynamic) NE-regret in this case.
A better $\order(1)$ result is implicitly implied by the aforementioned work of~\citet{COLT'21:Hsieh}, as we show (in \pref{lemma:payoff-variation-conversation}) that (dynamic) NE-regret is bounded by the individual regret in this stationary case.
Our result again matches theirs up to logarithmic factors.

The last theorem provides an upper bound for duality gap.
\begin{theorem}[Duality Gap]
\label{thm:duality-gap-changing}
When $x$-player follows~\pref{alg:x-player} and $y$-player follows~\pref{alg:y-player}, we have
\begin{align*}
    & \gap = \sum_{t=1}^T \max_{y\in\Delta_n}x_t^\top A_ty- \sum_{t=1}^T \min_{x\in \Delta_{m}}x^\top A_ty_t \\
    & =\Ot\Big(\min\{T^{\frac{3}{4}} \big( 1 +Q_T \big)^{\frac{1}{4}}, 
    T^{\frac{1}{2}}(1+Q_T^{\frac{3}{2}}+P_TQ_T)^{\frac{1}{2}}\} \Big),
\end{align*}
where $Q_T \triangleq V_T + \min\{P_T, W_T\}$.
\end{theorem}
Once again the bound is $o(T)$ whenever $Q_T = o(T)$, and it implies the following corollary.
\begin{corollary}
\label{cor:duality-gap-stationary}
When $x$-player follows~\pref{alg:x-player} and $y$-player follows~\pref{alg:y-player}, if $A_t = A$ for all $t \in [T]$, then $\gap=\Ot(\sqrt{T})$.
\end{corollary}
Notably, the best known result of duality gap for a fixed game is $\O(\sqrt{T})$~\citep{ICLR'21-last-iteration}, and our result again matches theirs up to logarithmic factors. 

\subsection{Key Ideas for Analysis}
We now highlight some key components and novelty of our analysis.
As mentioned in~\pref{sec: algo overview}, to bound all the three metrics, the key is to bound the sum of the two players' dynamic regret, which further requires controlling the stability of the strategies between consecutive rounds. The following key lemma shows how such stability is controlled by the non-stationarity measures of $\{A_t\}_{t=1}^T$.
\begin{lemma}
\label{lemma:stability-main}
When $x$-player follows~\pref{alg:x-player} and $y$-player follows~\pref{alg:y-player},
we have both $\sum_{t=2}^T \norm{x_t - x_{t-1}}_1^2$ and $\sum_{t=2}^T \norm{y_t - y_{t-1}}_1^2$ bounded by
\begin{align*}
    \Ot\Big(\min\big\{\sqrt{(1+V_T)(1+P_T)} + P_T, 1+W_T\big\}\Big).
\end{align*}
\end{lemma}
This lemma implies an $\Ot(1)$ stability bound when the game is fixed, which is first proven in~\citep{COLT'21:Hsieh} where both players run the Optimistic Hedge algorithm with an adaptive learning rate. Our result generalizes theirs but requires a novel analysis due to both the time-varying matrices and the two-layer structure of our algorithm.
As another note, this lemma also highlights another difference of our method compared to~\citep{JMLR:sword++} --- as mentioned in \pref{sec: algo overview} our algorithm shares some similarity with theirs, but no explicit stability bound is proven or required in their problem, while stability is crucial for our whole analysis. We next present the proof sketch for~\pref{lemma:stability-main}. More details can be found in~\pref{appendix:key-lemmas}. 

\textit{Proof Sketch.}~~We show in~\pref{lemma:NE-variation-dynamic-regret} that the sum of the two players' dynamic regret (against a sequence $u_1, \ldots, u_T \in \Delta_m$ for $x$-player and a sequence $v_1, \ldots, v_T \in \Delta_n$ for $y$-player) can be bounded by
\begin{align*}
    &\sum_{t=1}^T\big(x_t^\top A_tv_t- u_t^\top A_ty_t\big) = \Ot\Bigg(\frac{1+P_T^u+P_T^v}{\eta}+\eta (1+V_T)\Bigg)\\
    &\qquad-\Omega\Bigg(\sum_{t=1}^T(\|x_t-x_{t-1}\|_1^2+\|y_t-y_{t-1}\|_1^2)\Bigg),
\end{align*}
for \emph{any} step size $0<\eta\leq \otil(1)$. Here, $P_T^u\triangleq \sum_{t=2}^T\|u_t-u_{t-1}\|_1$ and $P_T^v \triangleq\sum_{t=2}^T\|v_t-v_{t-1}\|_1$ are the path-length of comparators. Then,~\pref{lemma:stability-main} can be proven by taking different choices of $\eta$ and the comparator sequence. 
For example, consider picking $(u_t,v_t)=(x_t^*,y_t^*)$. Since the saddle point property ensures $x_t^\top A_ty_t^*-x_t^{*\top}A_ty_t\geq0$, 
rearranging and picking the optimal $\eta$ thus gives the first bound $\otil(\sqrt{(1+V_T)(1+P_T)}+P_T)$ on the stability. 
To prove the second bound, pick $(u_t,v_t)=(\bar{u}^*, \bar{v}^*)$ where $(\bar{u}^*, \bar{v}^*)$ is a Nash equilibrium of the averaged game matrix. Then, we have $P_T^u=P_T^v=0$ and $\sum_{t=1}^Tx_t^\top A_tv_t-\sum_{t=1}^Tu_t^\top A_ty_t\geq -\order(W_T)$. 
Rearranging, picking the optimal $\eta$, and using $V_T\leq \order(W_T)$ then proves the $\Ot(1 + W_T)$ bound. \qed

We finally briefly mention two more new ideas when bounding the duality gap.
First, we apply a reduction from general dynamic regret that competes with any comparator sequence to its worst-case variant, which in some place helps bound the duality gap by the aforementioned stability.
Second, we show how the extra set of ``dummy'' base-learners enables the meta-algorithm to have a direct control on the duality gap. 
We refer the reader to~\pref{appendix:proof_duality-gap} for more details.
\section{Discussions and Future Directions}\label{sec: conclusion}

Our work is among the first few to study learning in time-varying games, and we believe that our proposed performance measures and algorithm are important first steps in this direction.
Our results can also be directly extended to general convex-concave games over a bounded convex domain (details omitted).
We also conduct experiments with synthetic data to show the effectiveness of our algorithm compared to a single base-learner; see \pref{appendix:experiments}.

One missing part in our work is the tightness of each bound --- even though they match the best known results for a fixed game, it is unclear whether they can be further improved in the general case.
We leave this as a future direction.
Another interesting direction would be to consider extending the results to time-varying multi-player general-sum games.

\bibliography{ref}
\bibliographystyle{icml2022}

\newpage
\appendix
\onecolumn
\section{Algorithm for $y$-player}\label{appendix:y-alg}
For completeness, in this section, we show the algorithm run by $y$-player as follows. Our algorithm for $y$-player maintains $N+n$ base-learners: for $i\in [N]$, the $i$-th base-learner is any algorithm that satisfies $\drvu(\eta_i^y)$ where $\eta_i^y=\frac{2^{i-1}}{L\sqrt{T}}$ and $L$ is defined in~\pref{eq:L_def}; the last $n$ base-learners are dummy learners, in which the $(j+N)$-th one always outputting the basis vector $e_j\in \Delta_n$. Let $\calS_y\triangleq\calS_{1,y}\cup\calS_{2,y}$ with $\calS_{1,y}=[N]$ and $\calS_{2,y}=\{N+1,\ldots,N+n\}$ denote the set of indices of base-learners.

At round $t$, each base-learner $j$ submits her decision $y_{t,j}\in \Delta_{n}$ to the meta-algorithm, who decides the final decision $y_t$. After receiving the feedback $A_t^\top x_t$, the meta-algorithm sends this feedback to each base-learner $j\in \calS_{1,y}$.

The meta-algorithm of $y$-player performs the following update:
\begin{equation}
\label{eq:optimistic-ogd-y}
\begin{split}
    q_{t} &= \argmin_{q \in \Delta_{|\calS_y|}} \left\{ \epsilon_{t}^y\langle q,m_t^y \rangle + \|q-\wh{q}_{t}\|_2^2\right\},\\
    \wh{q}_{t+1} &= \argmin_{q \in \Delta_{|\calS_y|}} \left\{ \epsilon_{t}^y\langle q,\ell_t^y \rangle + \|q-\wh{q}_t\|_2^2\right\},
\end{split}
\end{equation}
where $\epsilon_{t}^y$ is the dynamic learning rate for the $y$-player. The loss vector $\ell_{t}^y\in \Delta_{|\calS_y|}$ and loss predictor vector $m_t^y\in\Delta_{|\calS_y|}$ is defined as follows: for any $j\in \calS_y$,
\begin{align}
    \ell_{t,j}^y = -y_{t,j}^\top A_t^\top x_t + \lambda\|y_{t,j}-y_{t-1,j}\|_1^2,\label{eq:correction-loss-y}\\
    m_{t,j}^y= -y_{t,j}^\top A_{t-1}^\top x_{t-1} + \lambda\|y_{t,j}-y_{t-1,j}\|_1^2.\label{eq:correction-optimism-y}
\end{align}
The full pseudo code of the algorithm run by $y$-player is shown in~\pref{alg:y-player}.

\begin{algorithm}[h]
     \caption{Algorithm for the $y$-player}
     \label{alg:y-player}
     \textbf{Input}: a base-algorithm $\mathcal{A}(\eta)$ satisfying DRVU($\eta$). 
     
     \textbf{Initialize}: 
     a set of base-learners $\calS_y$ as described in \pref{appendix:y-alg},
     $\wh{p}_1 = \frac{1}{|\calS_y|} \mathbf{1}_{|\calS_y|}$, learning rate $\epsilon_1^y = \frac{1}{L}$ (\textit{c.f.} \pref{eq:L_def}).  
     
     \For{$t=1,\dots,T$}{
        Receive $y_{t,j}\in\Delta_n$ from each base-learner $j \in \calS_y$.
        
        Compute $m_t^y$ based on \pref{eq:correction-optimism-y} and $q_t$ based on \pref{eq:optimistic-ogd-y}.  
        
        Play the final decision $y_t = \sum_{j\in \calS_y} p_{t,j} y_{t,j}$.
        
        Suffer loss $-x_t^\T A_t y_t$ and observe the loss vector $-A_t ^\top x_t$.
        
        Compute $\ell_{t}^y$ based on \pref{eq:correction-loss-y} and $\wh{q}_{t+1}$ based on \pref{eq:optimistic-ogd-y}.  
        
        Update $\epsilon_{t+1}^y= 1/\sqrt{L^2+\sum_{s=2}^t\|A_t^\top x_t-A_{t-1}^\top x_{t-1}\|_\infty^2}$.
        
        Send $-A_t^\top x_t$ as the feedback to each base-learner.
     }
\end{algorithm}
\section{Discussions on Performance Measure}
\label{appendix:discussion-measure}
In this section, we include more discussions on the performance measures presented in~\pref{sec:measure-regret}.

\subsection{Relationship between Dynamic NE-Regret and NE-Regret}
\label{appendix:dynNE-NE}
Before discussing the relationship between dynamic NE-regret and NE-regret for the game setting, we first review the notion of dynamic regret and static regret for the online convex optimization (OCO) setting. Then we show that in contrast to the case in OCO that the worst-case dynamic regret is always larger than static regret, in the online game setting, dynamic NE-regret is not necessarily larger than the standard NE-regret due to the different structure of the minimax operation.

\paragraph{Dynamic Regret for OCO.} OCO can be regarded as an iterative game between the player and the environment. At each round $t \in [T]$, the player makes the decision $x_t$ from a convex feasible domain $\X \subseteq \R^d$ and simultaneously the environment chooses the loss function $f_t : \X \mapsto \R$, then the player suffers an instantaneous loss $f_t(x_t)$ and observe the full information about the loss function. The standard regret measure is defined as the difference between the cumulative loss of the player and that of the best action in hindsight:
\begin{equation}
	\label{eq:OCO-static-regret}
	\textnormal{Reg}_T = \sum_{t=1}^T f_t(x_t) - \min_{x\in \X} \sum_{t=1}^T f_t(x).
\end{equation}
Note that the measure only competes with a single fixed decision over the time. A stronger measure proposed for OCO problems is called \emph{general dynamic regret}~\citep{ICML'03:zinkvich,NIPS'18:Zhang-Ader,NIPS'20:sword,JMLR:sword++}, defined as
\begin{equation}
	\label{eq:OCO-dynamic-regret-general}
	\textnormal{D-Reg}_T(u_1,\ldots,u_T) = \sum_{t=1}^T f_t(x_t) - \sum_{t=1}^T f_t(u_t),
\end{equation}
which benchmarks the player's performance against an arbitrary sequence of comparators $u_1,\ldots,u_T \in \X$. The measure is also studied in the prediction with expert advice setting~\citep{conf/nips/Cesa-BianchiGLS12,COLT'15:Luo-AdaNormalHedge,NIPS'16:Wei-non-stationary-expert}. 
We emphasize that one of the key tools to achieve our results for time-varying games is to derive a favorable bound for the above general dynamic regret for each player. See~\pref{lemma:NE-variation-dynamic-regret} for the details of our derived bound.

In addition, there is a variant of the above general dynamic regret called the \emph{worst-case dynamic regret}, defined as
\begin{equation}
	\label{eq:OCO-dynamic-regret-worst-case}
	\textnormal{D-Reg}_T^* = \textnormal{D-Reg}_T(x_1^*,\ldots,x_T^*) = \sum_{t=1}^T f_t(x_t) - \sum_{t=1}^T f_t(x_t^*),
\end{equation}
where $x_t^* \in \argmin_{x\in \X} f_t(x)$ is the minimizer of the online loss function $f_t$. The worst-case dynamic regret is extensively studied in the literature~\citep{OR'15:dynamic-function-VT,ICML'16:Yang-smooth,UAI'20:simple,L4DC'21:sc_smooth}. It is worth noting that both standard regret in~\pref{eq:OCO-static-regret} and the worst-case dynamic regret in~\pref{eq:OCO-dynamic-regret-worst-case} are special cases of the general dynamic regret in~\pref{eq:OCO-dynamic-regret-general}. In fact, by choosing the comparators as $u_1 = \ldots = u_T \in \argmin_{x \in \X} \sum_{t=1}^T f_t(x)$, the general dynamic regret recovers the standard static regret; and by choosing the comparators as $u_t = x_t^* \in \argmin_{x\in \X} f_t(x)$ for $t \in [T]$, the general dynamic regret recovers the worst-case dynamic regret.

Notice that the worst-case dynamic regret in~\pref{eq:OCO-dynamic-regret-worst-case} is strictly larger than the static regret in~\pref{eq:OCO-static-regret}, whereas the general dynamic regret in~\pref{eq:OCO-dynamic-regret-general} is not necessarily larger than the static regret due to the flexibility of the comparator sequence.

\paragraph{Dynamic NE-Regret of Online Two-Player Zero-Sum Game.} In this part, we aim to show that, different from the relationships between the (worst-case) dynamic regret and static regret in OCO setting, dynamic NE-regret is not necessarily larger than the NE-regret in the game setting. For a better readability, we here restate the definitions of NE-regret and dynamic NE-regret. Specifically, NE-regret is defined as the absolute value of the difference between the learners’ cumulative payoff and the minimax value of the time-averaged payoff matrix, namely,
\begin{equation}
\label{eq:NE-regret-appendix}
    \nereg \triangleq \left|\sum_{t=1}^Tx_t^\top A_ty_t - \min_{x\in \Delta_m}\max_{y\in \Delta_n}\sum_{t=1}^Tx^\top A_ty\right|.
\end{equation}
The dynamic NE-regret proposed by this paper is defined as absolute value of the difference between the cumulative payoff of the two players against the sum of the minimax game value at each round, namely,
\begin{equation}
\label{eq:dynamic-NE-regret-appendix}
    \dnereg \triangleq \left| \sum_{t=1}^Tx_t^\top A_ty_t - \sum_{t=1}^T\min_{x\in \Delta_m}\max_{y\in \Delta_n}x^\top A_ty \right|.
\end{equation}
Comparing to the original NE-regret in~\pref{eq:NE-regret-appendix}, we here move the minimax operation inside the summation of the benchmark. The operation is similar to that of the worst-case dynamic regret in~\pref{eq:OCO-dynamic-regret-worst-case}, which moves the minimization operation inside the summation of the benchmark compared to the standard static regret in~\pref{eq:OCO-static-regret}. However, the important point to note here is: worst-case dynamic regret is always no smaller than the static regret in online convex optimization setting, whereas the dynamic NE-regret is not necessarily larger than the NE-regret. Recall the example of two-phase online games in~\pref{sec:measure-regret}: the online matrix is set as $A_t = {\scriptsize \begin{pmatrix} 1 & -1 \\ -1 & 1\end{pmatrix}}$ when $t\leq T/2$, and set as $A_t = {\scriptsize \begin{pmatrix} 1 & -1 \\ 1 & 1\end{pmatrix}}$ when $t > T/2$. In this case, when both players are indeed using the Nash equilibrium strategy at each round, they will suffer $0$ dynamic NE-regret, while still incur a linear NE-regret as $\nereg=|T/2-0|=T/2$.

\subsection{Relationships among Individual Regret,  Duality Gap, and Dynamic NE-Regret}
\label{appendix:relation}
In this subsection, we discuss the relationship among the three performance measures considered in this work: individual regret, duality gap, and dynamic NE-regret. As mentioned in~\pref{sec:measure-regret}, both the individual regret and the dynamic NE-regret are bounded by duality gap. In the following, we present a formal statement and provide the proof.
\begin{proposition}\label{prop: dnereg-ireg-dualgap}
    Consider any strategy sequence $\{x_t\}_{t=1}^T$ and $\{y_t\}_{t=1}^T$, where $x_t\in \Delta_m$, $y_t\in \Delta_n$, $t\in [T]$. We have 
    \begin{equation}
        \textnormal{Reg}_T^x \leq \textnormal{Dual-Gap}_T,~ \textnormal{Reg}_T^y\leq \textnormal{Dual-Gap}_T, \text{ and }~  \textnormal{DynNE-Reg}_T\leq \textnormal{Dual-Gap}_T,
    \end{equation}
    where all these measures are defined in~\pref{sec:measure-regret}.
\end{proposition}

\begin{proof}
First, we show that $\iregx\leq \gap$ as follows.
\begin{align*}
    \iregx&=\sum_{t=1}^Tx_t^\top A_ty_t-\min_{x\in \Delta_m}\sum_{t=1}^Tx^\top A_ty_t\\ 
    &\leq \sum_{t=1}^T\max_{y\in \Delta_n}x_t^\top A_ty-\min_{x\in \Delta_m}\sum_{t=1}^Tx^\top A_ty_t\\
    & \leq \sum_{t=1}^T\max_{y\in \Delta_n}x_t^\top A_ty-\sum_{t=1}^T\min_{x\in \Delta_m}x^\top A_ty_t = \gap.
\end{align*}
The inequality of $\iregy \leq \gap$ can be obtained in the same way as shown above. 

For the relationship between $\dnereg$ and $\gap$, actually we have
\begin{align}
    \sum_{t=1}^Tx_t^\top A_ty_t-\sum_{t=1}^T\min_{x\in \Delta_m}\max_{y\in \Delta_n}x^\top A_ty &\leq \sum_{t=1}^T\max_{y\in \Delta_{n}} x_t^\top A_ty-\sum_{t=1}^T\min_{x\in \Delta_m}\max_{y\in \Delta_n}x^\top A_ty \nonumber\\
    &= \sum_{t=1}^T\max_{y\in \Delta_{n}} x_t^\top A_ty-\sum_{t=1}^Tx_t^{*\top} A_ty_t^* \tag{$(x_t^*,y_t^*)\in \calX_t^*\times\calY_t^*$}\nonumber\\
    &\leq \sum_{t=1}^T\max_{y\in \Delta_{n}} x_t^\top A_ty-\sum_{t=1}^Tx_t^{*\top} A_ty_t\nonumber\\
    &\leq \sum_{t=1}^T\max_{y\in \Delta_{n}} x_t^\top A_ty-\sum_{t=1}^T\min_{x\in \Delta_m}x^\top A_ty_t. \label{eq:ne-1}
\end{align}
In addition,
\begin{align}\label{eq:ne-2}
    \sum_{t=1}^T\min_{x\in \Delta_m}\max_{y\in \Delta_n}x^\top A_ty - \sum_{t=1}^Tx_t^\top A_ty_t &\leq \sum_{t=1}^Tx_t^{*\top} A_ty_t^* - \sum_{t=1}^T\min_{x\in\Delta_m}x^\top A_ty_t\nonumber\\
    &\leq \sum_{t=1}^T x_t^\top A_ty_t^*- \sum_{t=1}^T\min_{x\in\Delta_m}x^\top A_ty_t\nonumber\\
    &\leq \sum_{t=1}^T \max_{y\in \Delta_n}x_t^\top A_ty- \sum_{t=1}^T\min_{x\in\Delta_m}x^\top A_ty_t.
\end{align}
Combining~\pref{eq:ne-1} and~\pref{eq:ne-2} shows that $\dnereg\leq \gap$.
\end{proof}

\section{Discussions on Non-Stationarity Measure}
\label{sec: appendix_non_stationary}
In the following, we present more discussions on the relationships among all three non-stationarity measures ($P_T$, $V_T$, and $W_T$) proposed in~\pref{sec:measure-non-stationarity}.

\paragraph{Comparison between Nash non-stationarity $\varNE_T$ and game matrix non-stationarity $\variaA_T$, $\varA_T$.} Here we present two specific cases to show that the non-stationarity on Nash equilibrium is not comparable to the one on game matrix. 
\begin{compactitem}
    \item Case 1. Let $A_t = {\scriptsize \begin{pmatrix} 1 & -1 \\ -1 & 1\end{pmatrix}}$. Consider the time-varying games with $A_t=\frac{1}{T} A$ when $t$ is odd and $A_t=\frac{T-1}{T}A$ when $t$ is even. Notice that all $A_t$'s have the same (unique) Nash equilibrium $x_t^*=y_t^*=(\frac{1}{2}, \frac{1}{2})$ in this case, which implies that the path-length of Nash equilibria is $\varNE_T=0$. By contrast, the other two non-stationarity measures related to game payoff matrices are large, concretely, $\variaA_T=T(\frac{1}{2}-\frac{1}{T})=\Theta(T)$ and $\varA_T=T\cdot \frac{(T-2)^2}{T^2}=\Theta(T)$. \vspace{2mm}
    \item Case 2. Let $A' = {\scriptsize \begin{pmatrix} 1 & 1 \\ 1 & 1\end{pmatrix}}$ and $E = {\scriptsize \begin{pmatrix} \epsilon & \epsilon \\ -\epsilon & -\epsilon\end{pmatrix}}$ for some $\epsilon>0$. Consider $A_t=A'+(-1)^t E$. In this case, we have $x_t^*=(1,0)$ when $t$ is odd and $x_t^*=(0,1)$ when $t$ is even; $y_t^* = (\frac{1}{2},\frac{1}{2})$ for all rounds. Then the path-length of Nash equilibria is large, $\varNE_T=\Theta(T)$. By contrast, the other two measures can be small, specifically,  $\variaA_T=\Theta(T\epsilon)=\order(1)$ and $\varA_T=\Theta(T\epsilon^2) = \order(1/T)$ when choosing $\epsilon=\order(1/T)$.
\end{compactitem}

\paragraph{Comparison between two game matrix non-stationarity measures $\varA_T$ and $\variaA_T$.} Here we show the relationship between two non-stationarity measures regarding game matrix. First, we have $\varA_T\leq \order(\variaA_T)$ as $\sum_{t=2}^T\|A_t-A_{t-1}\|_{\infty}^2\leq 2\sum_{t=2}^T(\|A_t-\bar{A}\|_{\infty}^2+\|A_{t-1}-\bar{A}\|_{\infty}^2)\leq \order(W_T)$. Indeed, $\varA_T$ can be much smaller than $\variaA_T$ in some cases, for instance when $A_t=\frac{t}{T}A$ with $A' = {\scriptsize \begin{pmatrix} 1 & -1 \\ -1 & 1\end{pmatrix}}$, $\varA_T=T\cdot \frac{1}{T^2}=\Theta(\frac{1}{T})$ whereas $\variaA_T=\sum_{t=1}^T|\frac{t}{T}-\frac{T+1}{2T}|=\Theta(T)$.
\section{Verifying DRVU Property}
\label{appendix:DRVU-property}

In this section, we present two instantiations of Optimistic Online Mirror Descent (Optimistic OMD)~\citep{conf/nips/RakhlinS13} and prove that both of them satisfy the DRVU property in~\pref{def: drvu} with $\alpha,\beta,\gamma=\wt{\Theta}(1)$.

Consider the general protocol of online linear optimization over the linear function sequence $\{f_1,\ldots,f_T\}$ with $f_t(x)=\inner{x,g_t}$ over a convex feasible set $\calX\subseteq\mathbb{R}^d$. Optimistic OMD is a generic algorithmic framework parametrized by a sequence of optimistic vectors $M_1,\ldots,M_T \in \R^d$ and a regularizer $\psi$ that is $1$-strongly convex with respect to a certain norm $\|\cdot\|$. Optimistic OMD starts from an initial point $x_1 \in \X$ and then makes the following two-step update at each round:
\begin{equation}
	\label{eq:OMD}
	\begin{split}
	x_t = {} & \argmin_{x \in \X} \eta_t \langle M_t, x\rangle + D_{\psi}(x,\xh_t),\\
	\xh_{t+1} = {} & \argmin_{x \in \X} \eta_t \langle  g_t, x\rangle + D_{\psi}(x,\xh_t).
	\end{split}
\end{equation}
In above, $\eta_t > 0$ is the step size at round $t$, and $D_\psi(\cdot,\cdot)$ is the Bregman divergence induced by the regularizer $\psi$. \citet{JMLR:sword++} prove the following general result for the dynamic regret of Optimistic OMD, and we present the proof in \pref{appendix:proof-general-dynamic-OMD} for completeness.
\begin{theorem}[Theorem 1 of~\citet{JMLR:sword++}]
\label{thm:general-dynamic-regret-OMD}
The dynamic regret of Optimistic OMD whose update rule is specified in~\pref{eq:OMD} is bounded by
\begin{align*}
\sum_{t=1}^{T} \inner{x_t, g_t} - \sum_{t=1}^{T} \inner{u_t, g_t} \leq \sum_{t=1}^{T}\eta_t \norm{g_t - M_t}_*^2 {} & +  \sum_{t=1}^{T} \frac{1}{\eta_t} \Big( \Div{\u_t}{\xh_t} - \Div{\u_t}{\xh_{t+1}}\Big) \\
{} & -  \sum_{t=1}^{T} \frac{1}{\eta_t}\Big( \Div{\xh_{t+1}}{x_t} + \Div{x_t}{\xh_{t}} \Big),
\end{align*}
which holds for any comparator sequence $\u_1,\ldots,\u_T \in \X$.
\end{theorem}

Note that the theorem is very general due to the flexibility in choosing the comparator sequence $u_1,\ldots,u_T$ and the regularizer $\psi$. In the following, we present two instantiations of Optimistic OMD: Optimistic Hedge with a fixed-share update~\citep{herbster1998tracking, cesa2012mirror} and Optimistic Online Gradient Descent~\citep{COLT'12:variation-Yang}, and then we use the above general theorem to prove that the two algorithms indeed satisfy the \drvu~property defined in~\pref{def: drvu}.

\subsection{Optimistic Hedge with a Fixed-share Update}
\label{appendix:optimistic-Hedge}
In this subsection, we show that Optimistic Hedge with a fixed-shared update indeed satisfies the \drvu~property.

Consider the following online convex optimization with linear loss functions: for $t=1,\ldots,T$, an online learning algorithm proposes $x_t \in \Delta_m$ at the beginning of round $t$, and then receives a loss vector $g_t \in \R^m$ and suffer loss $\inner{g_t, x_t}$. 

We first present the algorithmic procedure. Optimistic Hedge with a fixed-share update starts from an initial distribution $\tilde{x}_1 \in \Delta_m$ and updates according to
\begin{equation}
    \label{eq:optimistic-hedge-fixed-share}
    \begin{split}
    \x_{t} = {} & \argmin_{\x \in \Delta_m} \eta \langle\x, h_{t}\rangle + D_{\psi}(\x,\tilde{\x}_{t}),\vspace{2mm}\\
    \xh_{t+1} = {} & \argmin_{\x \in \Delta_m} \eta \langle\x, g_{t}\rangle + D_{\psi}(\x,\tilde{\x}_{t}), \vspace{2mm}\\
    \tilde{\x}_{t+1} = {} & (1-\xi)\xh_{t+1} + \frac{\xi }{m} \mathbf{1}_m,
    \end{split}
\end{equation}
where $\eta > 0$ is a fixed step size, $\psi(x) = \sum_{i=1}^m x_i \log x_i$ is the negative-entropy regularizer, $D_{\psi}(\cdot,\cdot)$ is the induced Bregman divergence, and $0 \leq \xi \leq 1$ is the fixed-share coefficient. The first step updates by the optimistic vector $h_t \in \R^m$ that serving as a guess of the next-round loss, the second step updates by the received loss $g_t \in \R^m$, and the final step admits a fixed-share update. We have the following result on the dynamic regret of Optimistic Hedge with a fixed-share update.
\begin{lemma}
\label{lemma:optimistic-hedge-dynamic}
Set the fixed-share coefficient as $\xi = 1/T$. The dynamic regret of Optimistic Hedge with a fixed-share update is at most
\begin{equation}
    \label{eq:optimitsic-Hedge-dynamic}
    \sum_{t=1}^T \inner{g_t,x_t} - \sum_{t=1}^T \inner{g_t,u_t} \leq \frac{(3 + \log(mT)) (1+P_T^u)}{\eta} + \eta\sum_{t=1}^T\norm{g_t-h_t}_{\infty}^2  -\frac{1}{4\eta}\norm{\x_t-\x_{t-1}}_1^2,,
\end{equation}
where $u_1,\ldots,u_T \in \Delta_m$ is any comparator sequence and $P_T = \sum_{t=2}^T \norm{u_t - u_{t-1}}_1$ denotes the path-length of comparators. Therefore, when choosing the optimism as $h_t = g_{t-1}$, the algorithm satisfies the DRVU($\eta$) property with parameters $\alpha = 3 + \log(mT)$, $\beta = 1$, and $\gamma = \frac{1}{4}$.
\end{lemma}

\begin{proof}
First we note that the chosen regularizer, $\psi(x)=\sum_{i=1}^mx_i\log x_i$, is  $1$-strongly convex in $\|\cdot\|_1$, because for any $x,x'\in\Delta_m$ it holds that
\begin{align*}
    \psi(x)-\psi(x')-\inner{\psi(x'),x-x'}=\sum_{i=1}^mx_i\log\frac{x_i}{x'_i}\geq \frac{1}{2}\|x-x'\|_1^2,
\end{align*}
where the last inequality is by Pinsker's inequality. Therefore, we can apply the general result of~\pref{thm:general-dynamic-regret-OMD} with $f_t(x) = \inner{g_t,x}$ and $M_t = h_t$ and achieve the following result,
\begin{align} 
    \sum_{t=1}^T\inner{g_t,\x_t} - \sum_{t=1}^T \inner{g_t,u_t}\leq \eta\sum_{t=1}^T\norm{g_t-h_t}_{\infty}^2 &+  \frac{1}{\eta}\sum_{t=1}^T\left(\Div{u_t}{\tilde{\x}_t} -  \Div{u_t}{\xh_{t+1}}\right)\notag\\
    & - \frac{1}{\eta}\sum_{t=1}^T\left( \Div{\xh_{t+1}}{\x_t} + \Div{\x_t}{\tilde{\x}_t}\right).\label{eq:proof-fixshare-decomposition}
\end{align}
We now evaluate the right-hand side. For the second term $\sum_{t=1}^T(\Div{u_t}{\tilde{\x}_t} -  \Div{u_t}{\xh_{t+1}})$, we have
\begin{align}
    &\Div{u_t}{\tilde{\x}_t} -  \Div{u_t}{\xh_{t+1}}\notag\\
    &= \sum_{i=1}^m u_{t,i} \log\frac{u_{t,i}}{\tilde{\x}_{t,i}} - \sum_{i=1}^m u_{t,i}\log\frac{u_{t,i}}{\xh_{t+1,i}}=\sum_{i=1}^m u_{t,i} \log \frac{\xh_{t+1,i}}{\tilde{\x}_{t,i}}\notag\\
    &= \left(\sum_{i=1}^m u_{t,i} \log \frac{1}{\tilde{\x}_{t,i}} - \sum_{i=1}^m u_{t-1,i} \log \frac{1}{\xh_{t,i}} \right) +\left(\sum_{i=1}^m u_{t-1,i} \log \frac{1}{\xh_{t,i}} - \sum_{i=1}^m u_{t,i}\log \frac{1}{\xh_{t+1,i}}\right).\label{eq:proof-path-A}
\end{align}
Notice that the first term above can be further upper bounded by
\begin{align*}
    &\sum_{i=1}^m u_{t,i} \log \frac{1}{\tilde{\x}_{t,i}} - \sum_{i=1}^m u_{t-1,i} \log \frac{1}{\xh_{t,i}}\\
    & = \sum_{i=1}^m (u_{t,i} - u_{t-1,i}) \log \frac{1}{\tilde{\x}_{t,i}} + \sum_{i=1}^m u_{t-1,i} \log \frac{\xh_{t,i}}{\tilde{\x}_{t,i}}\\
    & \leq \log \frac{m}{\xi}\cdot\norm{\u_t-\u_{t-1}}_1 + \log \frac{1}{1-\xi},
\end{align*}
where the inequality comes from the fixed-share update procedure, where we have $\log \frac{1}{\tilde{\x}_{t,i}}\leq \log \frac{m}{\xi}$ and $\log\frac{\xh_{t,i}}{\tilde{\x}_{t,i}}\leq \log\frac{1}{1-\xi}$ for any $i\in[m]$ and $t\in[T]$. Then, taking summation of~\pref{eq:proof-path-A} over $t=2$ to $T$ and combining the fact that $\left(\Div{u_1}{\tilde{\x}_1} -  \Div{u_1}{\xh_{2}}\right) = \sum_{i=1}^m u_{1,i}\log \frac{\xh_{2,i}}{\tilde{\x}_{1,i}}$ and $\wt{x}_{1,i}\geq \frac{\xi}{m}$ for any $i\in[m]$, we get
\begin{align}
    &\frac{1}{\eta}\sum_{t=1}^T\left(\Div{u_t}{\tilde{\x}_t} -  \Div{u_t}{\xh_{t+1}}\right) \notag\\
    &\leq \frac{1}{\eta}\left(\log{\frac{m}{\xi}} \sum_{t=2}^T \norm{u_t-u_{t-1}}_1 + (T-1) \log \frac{1}{1-\xi} + \sum_{i=1}^m u_{1,i}\log\frac{1}{\xh_{2,i}} + \sum_{i=1}^m u_{1,i}\log \frac{\xh_{2,i}}{\tilde{\x}_{1,i}}\right)\notag\\
    &\leq  \frac{1}{\eta}\left(\log{\frac{m}{\xi}} \bigg(1+\sum_{t=2}^T \norm{u_t-u_{t-1}}_1\bigg) + (T-1) \log \frac{1}{1-\xi} \right)\label{eq:proof-path-term}.
\end{align}
Next, we proceed to analyze the negative term , i.e., the third term of the right-hand side  in~\pref{eq:proof-fixshare-decomposition}. Indeed,
\begin{align}
     &\sum_{t=2}^T\left( \Div{\xh_t}{\x_{t-1}} + \Div{\x_t}{\tilde{\x}_t}\right)\notag\\
    & \geq \frac{1}{2} \sum_{t=2}^T\left( \norm{\xh_t-x_{t-1}}_1^2 + \norm{\x_t-\tilde{\x}_t}_1^2\right) \tag{Pinsker's inequality}\\
    &\geq  \frac{1}{4}\sum_{t=2}^T\left(\norm{x_t-x_{t-1} + \xh_t - \tilde{\x}_t}_1^2\right) \tag{$\norm{x}_1^2 + \norm{y}_1^2 \geq \frac{1}{2} \norm{x + y}_1^2$}\\
    &= \frac{1}{4}\sum_{t=2}^T\norm{x_t-x_{t-1} + \xi\big(\xh_t - \frac{1}{m}\mathbf{1}_m\big)}_1^2\tag{due to the fixed-share update}\\
    &\geq \frac{1}{4}\sum_{t=2}^T\norm{\x_t-\x_{t-1}}_1^2 - \sum_{t=2}^T\frac{\xi}{2}\norm{\x_t-\x_{t-1}}_1\left\Vert\xh_{t}-\frac{1}{m}\mathbf{1}_m\right\Vert_1\tag{$\norm{a-b}_1^2\geq \norm{a}_1^2 - 2\norm{a}_1\cdot\norm{b}_1$}\\
    &\geq \frac{1}{4}\sum_{t=2}^T\norm{\x_t-\x_{t-1}}_1^2 - {2\xi(T-1)}.\label{eq:proof-negative-term}
\end{align}
Substituting~\pref{eq:proof-path-term} and~\pref{eq:proof-negative-term} into the general dynamic regret upper bound in~\pref{eq:proof-fixshare-decomposition}, we achieve
\begin{align*}
    \sum_{t=1}^T\inner{g_t,\x_t - u_t} & \leq \eta\sum_{t=1}^T\norm{g_t-h_t}_{\infty}^2 + \frac{1}{\eta} \log \frac{m}{\xi} \cdot (1 + P_T^u) + \frac{1}{\eta} (T-1) \log \frac{1}{1-\xi} + \frac{2\xi}{\eta} (T-1) - \frac{1}{4\eta} \sum_{t=2}^T \norm{x_t - x_{t-1}}_1^2\\
    & \leq \eta\sum_{t=1}^T\norm{g_t-h_t}_{\infty}^2 + \frac{1}{\eta} \left(3 + \log(mT) (1+P_T^u)\right) - \frac{1}{4\eta} \sum_{t=2}^T \norm{x_t - x_{t-1}}_1^2,
\end{align*}
where the last step holds because we set $\xi= \frac{1}{T}$ and
\begin{align*}
    \frac{1}{\eta} (T-1) \log \frac{1}{1-\xi} + \frac{2\xi}{\eta} (T-1) = \frac{1}{\eta} (T-1) \log\left(1 + \frac{\xi}{1-\xi}\right) + \frac{2\xi}{\eta} (T-1) \leq \frac{1}{\eta} (T-1) \frac{\xi}{1-\xi} + \frac{2\xi}{\eta} (T-1) \leq \frac{3}{\eta}.
\end{align*}
When choosing the optimism as $h_t = g_{t-1}$, we then have
\begin{align*}
\sum_{t=1}^T\inner{g_t,\x_t - u_t} \leq \eta\sum_{t=1}^T\norm{g_t-g_{t-1}}_{\infty}^2 + \frac{3 + \log(mT)}{\eta} (1+P_T^u) - \frac{1}{4\eta} \sum_{t=2}^T \norm{x_t - x_{t-1}}_1^2,
\end{align*}
which verifies the DRVU property of~\pref{def: drvu}, with $\alpha = 3 + \log(mT)$, $\beta = 1$, and $\gamma = \frac{1}{4}$. This ends the proof.
\end{proof}

\subsection{Optimistic Online Gradient Descent}
\label{appendix:optimistic-OGD}
In this subsection, we show that Optimistic Online Gradient Descent (Optimistic OGD) with a fixed-shared update indeed satisfies the \drvu~property.

Consider the following online convex optimization with linear loss functions: for $t=1,\ldots,T$, an online learning algorithm proposes $x_t \in \Delta_m$ at the beginning of round $t$, and then receives a loss vector $g_t \in \R^m$ and suffer loss $\inner{g_t, x_t}$. 

We first present the algorithmic procedure. Optimistic OGD starts from an initial distribution $\hat{x}_1 \in \Delta_m$ and updates according to
\begin{equation}
    \label{eq:optimistic-OGD}
    \begin{split}
    \x_{t} = {} & \argmin_{\x \in \Delta_m} \eta \langle\x, h_{t}\rangle + D_{\psi}(\x,\xh_{t}),\vspace{2mm}\\
    \xh_{t+1} = {} & \argmin_{\x \in \Delta_m} \eta \langle\x, g_{t}\rangle + D_{\psi}(x,\xh_{t}),
    \end{split}
\end{equation}
where $\eta > 0$ is a fixed step size and $\psi(x) = \frac{1}{2} \norm{x}_2^2$ is the Euclidean regularizer and $D_{\psi}(\cdot,\cdot)$ is the induced Bregman divergence. Compared to~\pref{eq:optimistic-hedge-fixed-share}. The first step updates by the optimistic vector $h_t \in \R^m$ that serving as a guess of the next-round loss, the second step updates by the received loss $g_t \in \R^m$. We note that Optimistic OGD does not require a fixed-share mixing operation to achieve dynamic regret. 

Then, we have the following result on the dynamic regret of Optimistic OGD.
\begin{lemma}
\label{lemma:optimistic-OGD-dynamic}
The dynamic regret of Optimistic OGD is at most
\begin{equation}
    \label{eq:optimitsic-OGD-dynamic}
    \sum_{t=1}^T \inner{g_t,x_t} - \sum_{t=1}^T \inner{g_t,u_t} \leq \frac{ (m +2) P_T^u}{\eta} + \frac{\eta m}{2} \sum_{t=2}^T \norm{g_t - h_t}_{\infty}^2 - \frac{1}{4 \eta m} \sum_{t=2}^T \norm{x_t - x_{t-1}}_1^2 + \O(1),
\end{equation}
where $u_1,\ldots,u_T \in \Delta_m$ is any comparator sequence and $P_T = \sum_{t=2}^T \norm{u_t - u_{t-1}}_1$ denotes the path-length of comparators. Therefore, when choosing the optimism as $h_t = g_{t-1}$, the algorithm satisfies the DRVU($\eta$) property with parameters $\alpha = m + 2$, $\beta = \frac{m}{2}$, and $\gamma = \frac{1}{4m}$.
\end{lemma}

\begin{proof}
From the general result of~\pref{thm:general-dynamic-regret-OMD}, we have the following dynamic regret bound for Optimistic OGD:
\[
    \sum_{t=1}^T \inner{g_t, x_t - u_t} \leq \frac{\eta}{2}\sum_{t=1}^T  \norm{g_t - h_t}_2^2 + \frac{1}{2\eta}\sum_{t=1}^T \left( \norm{\xh_t - u_t}_2^2 - \norm{\xh_{t+1} - u_t}_2^2\right) - \frac{1}{2\eta}\sum_{t=1}^T \left( \norm{\xh_{t+1} - x_t}_2^2 + \norm{\x_{t} - \xh_t}_2^2\right).
\]
Besides, we have
\begin{align*}
    & \sum_{t=1}^T \left(\norm{\xh_t - u_t}_2^2 - \norm{\xh_{t+1} - u_t}_2^2\right) \\
    & \leq \sum_{t=2}^T \norm{\xh_t - u_t}_2^2 - \sum_{t=2}^T \norm{\xh_{t} - u_{t-1}}_2^2 + D^2\\
    & = \sum_{t=2}^T \left( \norm{u_t - u_{t-1}}_2 \cdot \norm{\xh_t + \xh_t - u_t - u_{t-1}}_2 \right) + D^2 \\
    & \leq \sum_{t=2}^T \left( \norm{u_t - u_{t-1}}_2 \cdot \norm{\xh_t + \xh_t - u_t - u_{t-1}}_1 \right) + D^2 \\
    & \leq D^2 + 4 \sum_{t=2}^T \norm{u_t - u_{t-1}}_2\\
    & \leq D^2 + 4 \sum_{t=2}^T \norm{u_t - u_{t-1}}_1,
\end{align*}
where we introduce the notation $D = \sup_{x,y \in \Delta_m} \norm{x - y}_2$, and it can be verified that $D \leq \sqrt{2m}$. Further we have
\begin{align*}
    \sum_{t=1}^T \left(\norm{\xh_{t+1} - x_t}_2^2 + \norm{x_t - \xh_t}_2^2\right) & \geq \sum_{t=2}^T \left(\norm{\xh_t - x_{t-1}}_2^2 + \norm{x_t - \xh_t}_2^2\right) \geq \frac{1}{2} \sum_{t=2}^T \norm{x_t - x_{t-1}}_2^2 \geq \frac{1}{2m}\sum_{t=2}^T \norm{x_t - x_{t-1}}_1^2.
\end{align*}
Combining the above three inequalities, we achieve that
\begin{align*}
 \sum_{t=1}^T \inner{g_t,x_t - u_t} & \leq \frac{\eta}{2} \sum_{t=1}^T \norm{g_t - h_t}_2^2 + \frac{1}{\eta}\left(m + 2 P_T^u\right) - \frac{1}{4\eta m} \sum_{t=2}^T \norm{x_t - x_{t-1}}_1^2\\
 & \leq \frac{\eta m}{2} \sum_{t=1}^T \norm{g_t - h_t}_\infty^2 + \frac{m + 2}{\eta}\left(1 + P_T^u\right) - \frac{1}{4\eta m} \sum_{t=2}^T \norm{x_t - x_{t-1}}_1^2.
\end{align*}
Therefore, choosing the optimism as $h_t = g_{t-1}$, we then verify the DRVU property of~\pref{def: drvu} for Optimistic OGD, with $\alpha = m + 2$, $\beta = \frac{m}{2}$, and $\gamma = \frac{1}{4m}$. This ends the proof.
\end{proof}

\subsection{Proof of~\pref{thm:general-dynamic-regret-OMD}}
\label{appendix:proof-general-dynamic-OMD}
\begin{proof}
We decompose the instantaneous dynamic regret into three terms and bound each one respectively. Specifically,
\begin{align*}
    f_t(x_t) - f_t(\u_t) \leq  \inner{\nabla f_t(x_t),x_t - \u_t} =  \inner{\nabla f_t(x_t) - M_t, x_t - \xh_{t+1}} + \inner{M_t,x_t - \xh_{t+1}} + \inner{\nabla f_t(x_t),\xh_{t+1} - \u_t}.
\end{align*}

The first term can be controlled by~\pref{lemma:stability-OMD}, which guarantees that the OMD update satisfies $\norm{x_t - \xh_{t+1}} \leq \eta_t \norm{\nabla f_t(x_t) - M_t}_*$ and thus,
\begin{equation*}
    \inner{\nabla f_t(x_t) - M_t,x_t - \xh_{t+1}} \leq  \norm{\nabla f_t(x_t) - M_t}_* \cdot \norm{x_t - \xh_{t+1}} \leq \eta_t \norm{\nabla f_t(x_t) - M_t}_*^2.
\end{equation*}

We now analyze the remaining two terms on the right-hand side. By the Bregman proximal inequality in~\pref{lemma:bregman-divergence} and the OMD update step $\x_{t} =\argmin_{\x \in \X} \eta_t \innerp{M_t}{\x} + \Div{\x}{\xh_{t}}$, we have 
\begin{equation*}
    \inner{M_t,x_t - \xh_{t+1}} \leq \frac{1}{\eta_t} \Big( \Div{\xh_{t+1}}{\xh_t} - \Div{\xh_{t+1}}{x_t} - \Div{\x_{t}}{\xh_t}\Big).
\end{equation*}
Similarly, the OMD update step $\xh_{t+1} = \argmin_{\x \in \X}\eta_t \innerp{\nabla f_{t}(\x_{t})}{\x} + \Div{\x}{\xh_{t}}$ implies
\begin{equation*}
    \inner{\nabla f_t(x_t),\xh_{t+1} - \u_t} \leq \frac{1}{\eta_t} \Big( \Div{\u_t}{\xh_t} - \Div{\u_t}{\xh_{t+1}} - \Div{\xh_{t+1}}{\xh_t}\Big).
\end{equation*}
Combining the above three inequalities yields an upper bound for the instantaneous dynamic regret:
\begin{equation}
    \label{eq:instantaneous-dynamic-regret}
    f_t(x_t) - f_t(u_t) \leq \eta_t \norm{\nabla f_t(x_t) - M_t}_*^2 + \frac{1}{\eta_t}\Big( \Div{\u_t}{\xh_t} - \Div{\u_t}{\xh_{t+1}} - \Div{\xh_{t+1}}{x_t} - \Div{\x_{t}}{\xh_t}\Big).
\end{equation}
Taking the summation over all iterations completes the proof.
\end{proof}
\section{Proofs for~\pref{sec: regret guarantee}}
\label{appendix:main-proofs}

In this section, we provide the proofs for the main results presented in~\pref{sec: regret guarantee}, including individual regret of~\pref{thm:individual-regret}, dynamic NE-regret of~\pref{thm:dynamic-NE}, and duality gap of~\pref{thm:duality-gap-changing}.

\subsection{Proof of \pref{thm:individual-regret} (Individual Regret)}
\label{appendix:proof-individual-regret}
\begin{proof}
In the following, we focus on the individual regret of $x$-player, and the result for $y$-player can be proven in a similar way. The proof is split into three parts. 
\begin{compactitem}
    \item[(1)] First, we prove the $\otil(\sqrt{T})$ individual regret bound for $x$-player no matter whether the $y$-player follows the strategy suggested by~\pref{alg:y-player}.
    \item[(2)] Second, we prove the $\otil(\sqrt{1+V_T+P_T})$ bound, which depends on $V_T$ (the variation of the payoff matrices) and $P_T$ (the path-length of the Nash equilibrium sequence).
    \item[(3)] Finally, we prove the $\otil(\sqrt{1+V_T+W_T})$ bound, which depends on $V_T$ and $W_T$, the variation and variance of the payoff matrices.
\end{compactitem}

Our analysis is mainly based on the general dynamic regret bound proven in~\pref{lemma:NE-variation-dynamic-regret}. Specifically, using~\pref{eq:dynamic-regret-x-general}, setting a fixed comparator, i.e., $u_1=\ldots=u_T \in \argmin_{x\in \Delta_m} \sum_{t=1}^T x^\T A_t \y_t$ (then the path-length $P_T^u=0$), and also dropping the last three negative terms in the regret upper bound, for any $i \in \calS_{1,x}$ we have
\begin{align*}
    \sum_{t=1}^T \x_t^\T A_t \y_t - \min_{x}\sum_{t=1}^T x^\T A_t \y_t \leq \Ot\left( \frac{1}{\eta_i^x} + \eta_i^x \sum_{t=2}^T \norm{A_t-A_{t-1}}_\infty^2 + \eta_i^x \sum_{t=2}^T\|y_t-y_{t-1}\|_1^2\right).
\end{align*}
In the following, we further bound the right-hand side in three different ways to achieve different individual regret bounds.

\paragraph{The $\otil(\sqrt{T})$ robustness bound.} First of all, we prove that for $x$-player, her individual regret against $y$-player's actions is at most $\Ot(\sqrt{T})$, which holds even when $y$-player does not follow strategies suggested by~\pref{alg:y-player}. Note that $\sum_{t=2}^T \norm{A_t- A_{t-1}}_\infty^2 + \sum_{t=2}^T \norm{y_t- y_{t-1}}_1^2 \leq \order(T)$. Therefore, we have
\begin{align*}
    \sum_{t=1}^T \x_t^\T A_t \y_t - \min_{x}\sum_{t=1}^T x^\T A_t \y_t & \leq\Ot\left(\frac{1}{\eta_i^x} + \eta_i^x T \right) \leq \Ot(\sqrt{T}),
\end{align*}
where the last inequality is achieved by taking $i = i^\dagger$ such that $\eta_{i^\dagger} = \Theta(1/\sqrt{T})$. Note that the choice is viable due to the configuration of the step size pool. Similarly, we can also attain an $\Ot(\sqrt{T})$ robustness bound for $y$-player. 

Next, we demonstrate two adaptive bounds of the individual regret when both players follow our prescribed strategy (namely, $x$-player is using~\pref{alg:x-player} and $y$-player is using~\pref{alg:y-player}). They are both in the worst case $\otil(\sqrt{T})$ but can be much smaller if the sequence of online payoff matrices is less non-stationary.

\paragraph{The $\otil(\sqrt{1+V_T+P_T})$ bound.} We first consider the individual regret bound that scales with the variation of Nash equilibria denoted by $\varNE_T\triangleq \min_{\forall t,(x_t^*,y_t^*) \in \calX_t^*\times \calY_t^*} \sum_{t=2}^T\left(\|x_t^*-x_{t-1}^*\|_1+\|y_t^*-y_{t-1}^*\|_1\right)$ and $V_T=\sum_{t=2}^T\norm{A_t-A_{t-1}}_{\infty}^2$. According to ~\pref{lemma:stability-NE-variation}, which proves the stability of the dynamics with respect to $P_T$ and $V_T$ when both players are following the suggested strategy, we have $\sum_{t=2}^T\|y_t-y_{t-1}\|_1^2\leq \Ot(\sqrt{(1+V_T)(1+P_T)}+P_T)$. Therefore, we achieve
\begin{align*}
    \sum_{t=1}^Tx_t^\top A_ty_t-\min_{x\in \Delta_m}x^\top A_ty_t &\leq \otil\left(\frac{1}{\eta_i^x}+\eta_i^x V_T + \eta_i^x\left(\sqrt{(1+V_T)(1+P_T)}+P_T\right)\right)\\
    &\leq \otil\left(\frac{1}{\eta_i^x}+\eta_i^x(1+V_T+P_T)\right) \tag{by AM-GM inequality} \\
    &\leq \otil\left(\sqrt{1+V_T+P_T}\right),
\end{align*}
where in the last inequality, we choose $i = i^\ddagger$ such that $\eta_{i^\ddagger}^x\in [\frac{1}{2}\eta_*^x, 2\eta_*^x]$ where $\eta_*^x=\min\{\frac{1}{\sqrt{1+V_T+P_T}},\frac{1}{L}\}$. The choice of $\eta_{i^\ddagger}^x$ is also viable due to the configuration of the step size pool. 

\paragraph{The $\otil(\sqrt{1+V_T+W_T})$ bound.} We then consider the individual regret bound that scales with $V_T$ and the variance of the game matrices denoted by $W_T = \sum_{t=1}^T \norm{A_t - \bar{A}}_\infty$ with $\bar{A}=\frac{1}{T}\sum_{t=1}^TA_t$ being the averaged game matrix. Then according to~\pref{lemma:stability-payoff-variation}, which proves the stability of the dynamics with respect to $W_T$ when both players are following the suggested strategy, we have $\sum_{t=2}^T\|y_t-y_{t-1}\|_1^2\leq \Ot(1 + W_T)$. Therefore, we achieve
\begin{align*}
    \sum_{t=1}^Tx_t^\top A_ty_t-\min_{x\in \Delta_m}x^\top A_ty_t \leq \otil\left(\frac{1}{\eta_i^x}+\eta_i^x(V_T + 1 + W_T)\right) \leq \otil\left(\sqrt{1+V_T+W_T}\right),
\end{align*}
where in the last inequality, we choose $i = i^*$ such that $\eta_{i^*}^x\in [\frac{1}{2}\eta_*^x, 2\eta_*^x]$ where $\eta_*^x=\min\{\frac{1}{\sqrt{1+V_T+W_T}},\frac{1}{L}\}$. The choice of $\eta_{i^*}^x$ is also viable due to the configuration of the step size pool. 

Combining above three upper bounds finishes the proof of~\pref{thm:individual-regret}.
\end{proof}

\subsection{Proof of \pref{thm:dynamic-NE} (Dynamic NE-Regret)}
\label{appendix:proof_dynamic-NE-regret}

\begin{proof}
The proof for the dynamic NE-regret measure consists of two parts. We first prove the $\otil(\sqrt{(1+V_T)(1+P_T)}+P_T)$ bound and then prove the $\otil(1+W_T)$ bound. 

\paragraph{The $\otil(\sqrt{(1+V_T)(1+P_T)}+P_T)$ bound.} Let $(x_t^*, y_t^*)$ be the Nash equilibrium of the online game matrix $A_t$. We consider the upper bound in terms of the non-stationarity measure $P_T \triangleq \sum_{t=2}^T\left(\|x_t^*-x_{t-1}^*\|_1+\|y_t^*-y_{t-1}^*\|_1\right)$.\footnote{Strictly speaking, the path-length non-stationarity measure is $P_T \triangleq \min_{\forall t,(x_t^*,y_t^*) \in \X_t^* \times \Y_t^*} \sum_{t=2}^T\left(\|x_t^*-x_{t-1}^*\|_1+\|y_t^*-y_{t-1}^*\|_1\right)$ as the Nash equilibrium of each round may not unique. Fortunately, our analysis holds for any Nash equilibrium, so we can in particular take the sequence of Nash equilibria making $\sum_{t=2}^T\left(\|x_t^*-x_{t-1}^*\|_1+\|y_t^*-y_{t-1}^*\|_1\right)$ smallest possible. The quantity $P_T$ is only used in the analysis, and our algorithm does not require any prior knowledge about it.} As $(x_t^*, y_t^*)$ is the Nash equilibrium of $A_t$, the inequality $x_t^{*\T} A_ty\leq x_t^{*\top}A_ty_t^*\leq x^\top A_ty_t^*$ holds for any $x \in \Delta_m$ and $y \in \Delta_n$. We notice that
\begin{align*}
    &\min_x\max_y x^\top A_ty= x_t^{*\T} A_ty_t^* \geq x_t^{*\T} A_ty_t \geq \min_{x} x^\top A_ty_t, \\
    &\min_x\max_y x^\top A_ty= x_t^{*\T} A_t y_t^* \leq x_t^\T A_t y_t^* \leq \max_{y} x_t^\top A_t y.
\end{align*}

Therefore, we have
\begin{align*}
    \sum_{t=1}^Tx_t^\top A_ty_t - \sum_{t=1}^T\min_x\max_y x^\top A_ty \leq {} & \sum_{t=1}^T x_t^\top A_ty_t-\sum_{t=1}^Tx_t^{*\top}A_ty_t,\\
    -\sum_{t=1}^Tx_t^\top A_ty_t + \sum_{t=1}^T\min_x\max_y x^\top A_ty \leq {} & -\sum_{t=1}^T x_t^\top A_ty_t+\sum_{t=1}^Tx_t^\top A_ty_t^*,
\end{align*} 
which means that the dynamic NE-regret is upper bounded by the maximum of the following two dynamic regret bounds:
\begin{equation*}
    \left|\sum_{t=1}^Tx_t^\top A_ty_t - \sum_{t=1}^T\min_x\max_y x^\top A_ty \right| \leq \max\left\{\sum_{t=1}^Tx_t^\top A_ty_t-\sum_{t=1}^Tx_t^{*\top}A_ty_t, - \sum_{t=1}^Tx_t^\top A_ty_t+\sum_{t=1}^Tx_t^\top A_ty_t^*\right\}.
\end{equation*}

Moreover, according to the general dynamic regret analysis in~\pref{lemma:NE-variation-dynamic-regret} with the choice of $\{(u_t,v_t)\}_{t=1}^T=\{(x_t^*,y_t^*)\}_{t=1}^T$ and dropping the three negative terms, we have
\begin{align*}
    \sum_{t=1}^T x_t^\top A_ty_t - \sum_{t=1}^T x_t^{*\top}A_ty_t \leq \Ot\left( \frac{1 + P_T^x}{\eta_i^x} + \eta_i^x V_T + \eta_i^x \sum_{t=2}^T \norm{\y_t - \y_{t-1}}_1^2\right).
\end{align*}
where $P_T^x = \sum_{t=2}^T \norm{x_t^* - x_{t-1}^*}_1$ denotes the path-length of the Nash equilibria of the $x$-player. According to~\pref{lemma:stability-NE-variation}, we have $\sum_{t=2}^T\|y_t-y_{t-1}\|_1^2\leq \Ot(\sqrt{(1+V_T)(1+P_T)}+P_T)$, so using AM-GM inequality achieves
\begin{align*}
    \sum_{t=1}^Tx_t^\top A_ty_t-\sum_{t=1}^Tx_t^{*\top}A_ty_t&\leq \otil\left(\frac{1+P_T^x}{\eta_i^x}+\eta_i^xV_T+\eta_i^x(\sqrt{(1+V_T)(1+P_T)}+P_T)\right)\\
    &\leq \otil\left(\frac{1+P_T}{\eta_i^x}+\eta_i^x(1+V_T+P_T)\right) \\
    &\leq \otil\left(\sqrt{(1+P_T)(1+V_T+P_T)} + P_T\right)\\
    &\leq \otil\left(\sqrt{(1+P_T)(1+V_T)} + P_T\right).
\end{align*}
where the last inequality is by choosing $i=i^*$ such that $\eta_{i^*}^x\in [\frac{1}{2}\eta_*^x, 2\eta_*^x]$ where $\eta_*^x=\min\left\{\sqrt{\frac{1+P_T}{1+V_T+P_T}},\frac{1}{L}\right\}$.

\paragraph{The $\otil(1+W_T)$ bound.} According to~\pref{lemma:payoff-variation-conversation}, we have the NE-regret is bounded by the maximum of two individual regret upper bounds plus the variance of the game matrices.
\begin{equation*}
    \left|\sum_{t=1}^Tx_t^\top A_ty_t - \sum_{t=1}^T\min_x\max_y x^\top A_ty \right| \leq \max\left\{\sum_{t=1}^Tx_t^\top A_ty_t - \sum_{t=1}^T x^{*\T} A_t y_t, \sum_{t=1}^T x_t^{\T} A_t y^* - \sum_{t=1}^Tx_t^\top A_ty_t\right\} + 2W_T.
\end{equation*}
Using the $\Ot(\sqrt{1+V_T+W_T})$ individual regret bound proven in~\pref{thm:individual-regret} and the fact that $V_T\leq\order(W_T)$, we have
\begin{align*}
    &\left|\sum_{t=1}^Tx_t^\top A_ty_t - \sum_{t=1}^T\min_x\max_y x^\top A_ty \right| \leq \otil\left(\sqrt{1+V_T + W_T} + W_T\right)\leq \otil\left(1+W_T\right).
\end{align*}

To summarize, combining the above two upper bounds for dynamic NE-regret, we finally achieve the following guarantee:
\begin{align*}
    \left|\sum_{t=1}^Tx_t^\top A_ty_t - \sum_{t=1}^T\min_x\max_y x^\top A_ty \right| \leq \Ot\left(\min\left\{\sqrt{(1+V_T)(1+P_T)} + P_T, 1+W_T\right\}\right),
\end{align*}
which completes the proof of~\pref{thm:dynamic-NE}.
\end{proof}

\subsection{Proof of \pref{thm:duality-gap-changing} (Duality Gap)}
\label{appendix:proof_duality-gap}
\begin{proof}
In~\pref{thm:duality-gap-changing}, there are indeed two different upper bounds for duality gap of our approach, as restated below.
\[
    \sum_{t=1}^T x_t^\T A_t \bar{y}_t^* - \sum_{t=1}^T \bar{x}_t^{*\T} A_t y_t \leq\Ot\Big(\min\{T^{\frac{3}{4}} \big( 1 + Q_T \big)^{\frac{1}{4}}, T^{\frac{1}{2}}(1+Q_T^{\frac{3}{2}}+P_TQ_T)^{\frac{1}{2}}\} \Big),
\]
where $Q_T \triangleq V_T + \min\{P_T, W_T\}$ is introduced to simplify the notation. Now we will prove the two bounds respectively.

\paragraph{The $\Ot(T^{\frac{3}{4}} \big( 1 + V_T + \min\{P_T,W_T\} \big)^{\frac{1}{4}})$ bound.} For convenience of the following proof, we introduce the notation of $f_t(x)\triangleq x^\T A_t y_t$, and then the best response is essentially the minimizer of the function, namely, $\bar{x}_t^* = \argmin_{x\in\Delta_m} f_t(x)$. We now investigate the following worst-case dynamic regret of the $x$-player, 
\[
    \sum_{t=1}^T x_t^\T A_t y_t - \sum_{t=1}^T \bar{x}_t^{*\T} A_t y_t = \sum_{t=1}^T f_t(x_t) - \sum_{t=1}^T f_t(\bar{x}_t^{*}),
\]
which benchmarks the cumulative loss of $x$-player's actions with the best response at each round. We decompose the quantity into the following two terms:
\begin{align*}
   \sum_{t=1}^T f_t(x_t) - \sum_{t=1}^T f_t(\bar{x}_t^{*}) = \underbrace{\sum_{t=1}^T f_t(x_t) - \sum_{t=1}^T f_t(u_t)}_{\term{i}} + \underbrace{\sum_{t=1}^T f_t(u_t) - \sum_{t=1}^T f_t(\bar{x}_t^{*})}_{\term{ii}},
\end{align*}
where we insert a term $\sum_{t=1}^T f_t(u_t)$ as an anchor quantity. Notably, this comparator sequence $\{u_t\}_{t=1}^T$ can be arbitrarily set without affecting the above equation. In particular, we choose it  as a piecewise-stationary comparator sequence such that $\I_1,\I_2,\ldots,\I_K$ is an even partition of the total horizon $[T]$ with $\abs{\I_k} = \Delta$ for $k = 1,\ldots,K$ (for simplicity, suppose the time horizon $T$ is divisible by epoch length $\Delta$), and for any $t \in \I_k$, $u_t \triangleq \argmin_{x\in\Delta_m} \sum_{t\in\I_k} f_t(x)$. Then, following the general dynamic regret bound proven in~\pref{lemma:NE-variation-dynamic-regret}, for this particular comparator sequence and for any $i \in \calS_{1,x}$, we have the following upper bound for term (i):
\begin{align*}
    & \sum_{t=1}^T f_t(x_t) - \sum_{t=1}^T f_t(u_t)\\
    &\leq \order\left(\frac{\alpha(1+P_T^u)}{\eta_i^x}\right)+\eta_i^x\cthree \beta\sum_{t=2}^T\|A_t-A_{t-1}\|_{\infty}^2+\eta_i^x\cthree \beta\sum_{t=2}^T\|y_t-y_{t-1}\|_{1}^2+\left(\lambda-\frac{\gamma}{\eta_i^x}\right)\sum_{t=2}^T\|x_{t,i}-x_{t-1,i}\|_1^2 \\
    &\qquad - L\sum_{t=1}^T(\|p_t-\wh{p}_{t+1}\|_2^2+\|p_t-\wh{p}_t\|_2^2)-\lambda\sum_{t=1}^T\sum_{i\in \calS_x}p_{t,i}\|x_{t,i}-x_{t-1,i}\|_1^2+ \otil(1)\\
    & \leq \Ot\left( \frac{1 + P_T^u}{\eta_i^x} + \eta_i^x \sum_{t=2}^T \norm{A_t - A_{t-1}}_\infty^2 \right)+\eta_i^x \cthree \beta \sum_{t=2}^T \norm{\y_t - \y_{t-1}}_1^2 \\
    & \qquad  - \frac{L}{2} \sum_{t=2}^T \norm{\p_t - \p_{t-1}}_1^2  - \lambda \sum_{t=2}^T\sum_{i=1}^N p_{t,i} \norm{\x_{t,i} - \x_{t-1,i}}_1^2 \tag{$\lambda-\frac{\gamma}{\eta_i^x}=\frac{\gamma L}{2}-\frac{\gamma}{\eta_i^x}\leq 0$}\\
    & \leq \Ot\left( \frac{1 + P_T^u}{\eta_i^x} + \eta_i^x \sum_{t=2}^T \norm{A_t - A_{t-1}}_\infty^2 \right)+2\eta_i^x \cthree \beta \sum_{t=2}^T \norm{q_t - q_{t-1}}_1^2 + 2\eta_i^x \cthree \beta \sum_{t=2}^T\sum_{j\in \calS_y} q_{t,j}\norm{y_{t,j} - y_{t-1,j}}_1^2\\
    & \qquad  - \frac{L}{2} \sum_{t=2}^T \norm{\p_t - \p_{t-1}}_1^2  - \lambda \sum_{t=2}^T\sum_{i=1}^N p_{t,i} \norm{\x_{t,i} - \x_{t-1,i}}_1^2 \tag{by~\pref{eq:y-split}} \\
    & \leq \Ot\left( \frac{T/\Delta}{\eta_i^x} + \eta_i^x \sum_{t=2}^T \norm{A_t - A_{t-1}}_\infty^2 \right)+2\eta_i^x \cthree \beta \sum_{t=2}^T \norm{q_t - q_{t-1}}_1^2 + 2\eta_i^x \cthree \beta \sum_{t=2}^T\sum_{j\in \calS_y} q_{t,j}\norm{y_{t,j} - y_{t-1,j}}_1^2\\
    & \qquad  - \frac{L}{2} \sum_{t=2}^T \norm{\p_t - \p_{t-1}}_1^2  - \lambda \sum_{t=2}^T\sum_{i=1}^N p_{t,i} \norm{\x_{t,i} - \x_{t-1,i}}_1^2,
\end{align*}
where the last step holds because $P_T^u = \sum_{t=2}^T \norm{u_t - u_{t-1}}_1 = \O(K) = \O(T/\Delta)$ by the specific construction of the comparator sequence.

Moreover, in~\pref{lemma:function-variation} we present a general result to relate the function-value difference between the sequence of piecewise minimizers and the sequence of each-round minimizers, so term (ii) can be well upper bounded as follows.
$$\sum_{t=1}^T f_t(u_t) - \sum_{t=1}^T f_t(\bar{x}_t^{*}) \leq 2 \Delta \sum_{t=2}^{T} \norm{A_t y_t - A_{t-1}y_{t-1}}_{\infty}.$$ 
Combining the above two inequalities, we get the worst-case dynamic regret for the $x$-player: for any $i \in \calS_{1,x}$, we have
\begin{align*}
    & \sum_{t=1}^T x_t^\T A_t y_t - \sum_{t=1}^T \bar{x}_t^{*\T} A_t y_t\\
    & \leq \Ot\left( \frac{T/\Delta}{\eta_i^x} + \eta_i^x V_T\right) +2\eta_i^x \cthree \beta \sum_{t=2}^T \norm{q_t - q_{t-1}}_1^2 + 2\eta_i^x \cthree \beta \sum_{t=2}^T\sum_{j\in \calS_y} q_{t,j}\norm{y_{t,j} - y_{t-1,j}}_1^2\\
    &\qquad + 2 \Delta \sum_{t=2}^{T} \norm{A_t y_t - A_{t-1}y_{t-1}}_{\infty} - \frac{L}{2} \sum_{t=2}^T \norm{\p_t - \p_{t-1}}_1^2  - \lambda \sum_{t=2}^T\sum_{i\in \calS_x} p_{t,i} \norm{\x_{t,i} - \x_{t-1,i}}_1^2 .
\end{align*}

Similarly, we can also obtain the worst-case dynamic regret for the $y$-player: for any $j \in \calS_{1,y}$, we have
\begin{align*}
    & \sum_{t=1}^T x_t^\T A_t y_t - \sum_{t=1}^T \bar{x}_t^{*\T} A_t y_t\\
    & \leq \Ot\left( \frac{T/\Delta}{\eta_j^y} + \eta_j^y V_T\right) +2\eta_j^y \cthree \beta \sum_{t=2}^T \norm{p_t - p_{t-1}}_1^2 + 2\eta_j^y \cthree \beta \sum_{t=2}^T\sum_{i\in \calS_x} p_{t,i}\norm{x_{t,i} - x_{t-1,i}}_1^2\\
    &\qquad + 2 \Delta \sum_{t=2}^{T} \norm{x_t^\T A_t - x_{t-1}^\T A_{t-1}}_{\infty} - \frac{L}{2} \sum_{t=2}^T \norm{q_t - q_{t-1}}_1^2  - \lambda \sum_{t=2}^T\sum_{j\in \calS_y} q_{t,j} \norm{y_{t,j} - y_{t-1,j}}_1^2 .
\end{align*}
Combining the two dynamic regret bounds yields: for any $i \in \calS_{1,x}$ and any $j\in \calS_{1,y}$,
\begin{align}
    & \sum_{t=1}^T x_t^{\T} A_t \bar{y}_t^* - \sum_{t=1}^T \bar{x}_t^{*\T} A_t y_t \nonumber\\
    & \leq \Ot\left( \frac{T/\Delta}{\eta_i^x} + \frac{T/\Delta}{\eta_j^y} + \eta_i^x V_T + \eta_j^y V_T\right) \nonumber\\
    & \qquad + 2 \Delta \sum_{t=2}^{T} \norm{A_t y_t - A_{t-1}y_{t-1}}_{\infty} + 2 \Delta \sum_{t=2}^{T} \norm{x_t^\T A_t - x_{t-1}^\T A_{t-1}}_{\infty} \nonumber\\
    & \qquad + \left(2\eta_j^y\cthree \beta-\frac{L}{2}\right)\sum_{t=2}^T\|q_t-q_{t-1}\|_1^2 + \left(2\eta_i^x\cthree \beta-\frac{L}{2}\right)\sum_{t=2}^T\|p_t-p_{t-1}\|_1^2 \nonumber\\
    & \qquad + (2\eta_j^y\cthree \beta-\lambda)\sum_{t=2}^T\sum_{j\in \calS_y}q_{t,j}\|y_{t,j}-y_{t-1,j}\|_1^2 + (2\eta_i^x\cthree \beta-\lambda)\sum_{t=2}^T\sum_{i\in \calS_x}p_{t,i}\|x_{t,i}-x_{t-1,i}\|_1^2 \nonumber\\
    & \leq \Ot\left( \frac{T/\Delta}{\eta_i^x} + \frac{T/\Delta}{\eta_j^y} + \eta_i^x V_T + \eta_j^y V_T\right)+ 2 \Delta \sum_{t=2}^{T} \norm{A_t y_t - A_{t-1}y_{t-1}}_{\infty} + 2 \Delta \sum_{t=2}^{T} \norm{x_t^\T A_t - x_{t-1}^\T A_{t-1}}_{\infty}, \label{eq:duality-gap-inter}
\end{align}
where the last inequality is because $2\eta_j^y\cthree \beta-\frac{L}{2}\leq 0$, $2\eta_i^x\cthree \beta-\frac{L}{2}\leq 0$, $2\eta_i^x\cthree \beta-\gamma\leq 0$, $2\eta_i^x\cthree \beta-\gamma\leq 0$ based on the fact that $\eta_{j}^y\leq \frac{1}{L}$, $\eta_{i}^x\leq \frac{1}{L}$ and $L=\max\{4, \sqrt{16\cthree \beta}, \sqrt{\frac{8\cthree \beta}{\gamma}}\}$ and $\lambda=\frac{\gamma L}{2}$.

Next, we bound the last two terms in the right-hand side of~\pref{eq:duality-gap-inter}. Indeed,
\begin{align*}
    & \sum_{t=2}^{T} \norm{A_t y_t - A_{t-1}y_{t-1}}_{\infty} + \sum_{t=2}^{T} \norm{x_t^\T A_t - x_{t-1}^\T A_{t-1}}_{\infty}\\
    & \leq \sqrt{T} \left(\sum_{t=2}^{T} \norm{A_t y_t - A_{t-1}y_{t-1}}_{\infty}^2 + \sum_{t=2}^{T} \norm{x_t^\T A_t - x_{t-1}^\T A_{t-1}}_{\infty}^2\right)^{\frac{1}{2}}\\
    & \leq \Ot\left( \sqrt{T(1+ V_T + \min\{P_T,W_T\})}\right),
\end{align*}
where the first inequality is by Cauchy-Schwarz inequality and the second inequality uses the gradient-variation bound in~\pref{lemma:gradient-variation}.
Plugging this into~\pref{eq:duality-gap-inter} and choosing $\eta_i^x, \eta_j^y\in [\frac{1}{2}\eta_*, 2\eta_*]$ with $\eta_*=\min\left\{\frac{1}{L}, \sqrt{\frac{T}{\Delta (1+V_T)}}\right\}$, we have
\begin{align*}
    \sum_{t=1}^T x_t^{\T} A_t \bar{y}_t^* - \sum_{t=1}^T \bar{x}_ t^{*\T} A_t y_t  & \leq \Ot\left( \sqrt{\frac{T(1+V_T)}{\Delta}} + \frac{T}{\Delta} + \Delta\sqrt{T(1+ V_T + \min\{P_T, W_T\})} \right) \\
    & \leq \Ot\left(\frac{T}{\Delta} + \Delta\sqrt{T(1+ V_T + \min\{P_T, W_T\})} \right) \\
    & = \Ot\left(T^{\frac{3}{4}}\big( 1 + V_T + \min\{P_T, W_T\} \big)^{\frac{1}{4}} \right).
\end{align*} 
where the last inequality is by setting the epoch length $\Delta$ optimally. We remark that the above choice of $\eta_i^x, \eta_j^y$ is viable due to the construction of step size pool and the fact $\sqrt{T/(\Delta (1+V_T))} \geq \Theta(1/T)$; besides, the setting of epoch length $\Delta$ is also feasible, and notably the epoch length is only used in the analysis and our algorithm does not require its information.

\paragraph{The $\Ot(T^{\frac{1}{2}}(1+Q_T^{\frac{3}{2}}+P_TQ_T)^{\frac{1}{2}})$ bound.} From the update rule of the meta-algorithm and~\pref{eq:instantaneous-dynamic-regret} proven in~\pref{thm:general-dynamic-regret-OMD} with $\psi(x)=\frac{1}{2}\|x\|_2^2$, we have the following instantaneous regret bound for any $p \in \Delta_{|\calS_x|}$ and $q \in \Delta_{|\calS_y|}$.
\begin{equation}
    \label{eq:OGD-regret}
    \begin{split}
    \inner{p_t - p, \ell_t^x} \leq {} & \epsilon_{t}^x \norm{\ell_t^x - m_t^x}_2^2 +  \frac{1}{\epsilon_{t}^x} \Big(\|\wh{p}_t-p\|_2^2 - \|\wh{p}_{t+1}-p\|_2^2 - \|p_t-\wh{p}_{t+1}\|_2^2-\|p_t-\wh{p}_t\|_2^2\Big),\\
    \inner{q_t - q, \ell_t^y} \leq {} & \epsilon_{t}^y \norm{\ell_t^y - m_t^y}_2^2 +  \frac{1}{\epsilon_{t}^y} \Big(\|\wh{q}_t-q\|_2^2 - \|\wh{q}_{t+1}-q\|_2^2 - \|q_t-\wh{q}_{t+1}\|_2^2-\|q_t-\wh{q}_t\|_2^2\Big).
    \end{split}
\end{equation}
Recall that the feedback loss and optimism are set as follows. For $x$-player, we have $\ell_t^x=X_t^\top A_ty_t+\lambda X_t^\delta$ and $m_t^x=X_t^\top A_{t-1}y_{t-1}+\lambda X_t^\delta$, where $X_t=[x_{t,1};x_{t,2};\ldots,x_{t,|\calS_x|}]$ and $X_t^\delta=[\|x_{t,1}-x_{t-1,1}\|_1^2;\|x_{t,2}-x_{t-1,2}\|_1^2;\ldots;\|x_{t,|\calS_x|}-x_{t-1,|\calS_x|}\|_1^2]$. For $y$-player, we have $\ell_t^y=-Y_t^\top A_t^\top x_t+\lambda Y_t^\delta$, $m_t^x=-Y_t^\top A_{t-1}^\top x_{t-1}+\lambda Y_t^\delta$, $Y_t=[y_{t,1};y_{t,2};\ldots;y_{t,|\calS_y|}]$ and $Y_t^\delta=[\|y_{t,1}-y_{t-1,1}\|_1^2;\|y_{t,2}-y_{t-1,2}\|_1^2;\ldots;\|y_{t,|\calS_y|}-y_{t-1,|\calS_y|}\|_1^2]$. Note that $\calS_x\triangleq\calS_{1,x}\cup\calS_{2,x}$ with $\calS_{1,x} = [N]$ and $\calS_{2,x} = \{N+1, \ldots, N+m\}$; besides, $\calS_y\triangleq\calS_{1,y}\cup\calS_{2,y}$ with $\calS_{1,y} = [N]$ and $\calS_{2,y} = \{N+1, \ldots, N+n\}$. $N=\lfloor \frac{1}{2}\log_2 T \rfloor + 1$.

Then using Cauchy-Schwarz inequality and noticing that the dimensions of $X_t$, $Y_t$ are $\wt{\Theta}(1)$, we have
\begin{equation}
    \label{eq:OGD-optimism}
    \begin{split}    
    \norm{\ell_t^x - m_t^x}_2^2 = {}& \|X_t^\top A_ty_t-X_t^\top A_{t-1}y_{t-1}\|_2^2 \leq c' \left(\|A_t-A_{t-1}\|_\infty^2+\|y_t-y_{t-1}\|_2^2\right),\\
    \norm{\ell_t^y - m_t^y}_2^2 = {}& \|-Y_t^\top A_t^\top x_t + -Y_t^\top A_{t-1}^\top x_{t-1}\|_2^2 \leq c' \left(\|A_t-A_{t-1}\|_\infty^2+\|x_t-x_{t-1}\|_2^2\right),
    \end{split}
\end{equation}
where $c' > 0$ is a universal constant independent with the time horizon and the non-stationarity measures (ignoring the dependence on poly-logarithmic factors in $T$).

In the following, we will specify the choice of the compared weight vectors. Concretely, let $(x_t^*,y_t^*)$ be any Nash equilibrium of the payoff matrix $A_t$. We pick the compared weight distribution $p = p_t^* \in \Delta_{|\calS_x|}$ and $q = q_t^* \in \Delta_{|\calS_y|}$ such that both $p_t^*$ and $q_t^*$ have supports only on the additional dummy base-learners and the supports finally result in a Nash equilibrium of $A_t$, namely, $p_{t,i}^* = q_{t,j}^* = 0$ for $i = 1,\ldots,|\calS_{1,x}|$ and $j = 1,\ldots,|\calS_{1,y}|$, $p_{t,i+|\calS_{1,x}|}^* = x_{t,i}^*$ for $i =1,\ldots,m$ and $q_{t,j+|\calS_{1,y}|}^* = y_{t,j}^*$ for $j = 1,\ldots,n$. By definition, we have  
\begin{equation}
    \label{eq:p-q-observe}
    \inner{p_t-p_t^*, X_t^\top A_ty_t}+\inner{q_t-q_t^*, -Y_t^\top A_tx_t}=-x_t^{*\top}A_ty_t+x_t^\top A_ty_t^*\geq 0.
\end{equation}
Moreover, combining~\pref{eq:OGD-regret} and~\pref{eq:OGD-optimism} yields the following results:
\begin{align*}
    \inner{p_t-p_t^*, X_t^\top A_ty_t} \leq {} & \frac{1}{\epsilon_{t}^x}\left(\|\wh{p}_t-p_t^*\|_2^2 - \|\wh{p}_{t+1}-p_t^*\|_2^2 - \|p_t-\wh{p}_{t+1}\|_2^2-\|p_t-\wh{p}_t\|_2^2\right) \\
    &\qquad+c'\cdot\epsilon_{t}^x\left(\|A_t-A_{t-1}\|_\infty^2+\|y_t-y_{t-1}\|_2^2\right) + \lambda \inner{p_t^*-p_t, X_t^\delta}.
\end{align*}
Notice that in fact we have $\inner{p_t^*, X_t^\delta} =0$ due to the choice of $p_t^*$. More specifically, $p_t^*$ has support only on the additional dummy base-learners and for those dummy learners their stability quantity is zero (namely, $X_{t,i}^\delta = 0$ for $i = |\calS_{1,x}|+1,\ldots,|\calS_{1,x}|+m$). In addition, it is clear that $\inner{p_t,X_t^\delta} \geq 0$, so we have 
\begin{align*}
    & \inner{p_t-p_t^*, X_t^\top A_ty_t} \\
    & \leq \frac{1}{\epsilon_{t}^x}\left(\|\wh{p}_t-p_t^*\|_2^2 - \|\wh{p}_{t+1}-p_t^*\|_2^2 - \|p_t-\wh{p}_{t+1}\|_2^2-\|p_t-\wh{p}_t\|_2^2\right) +c'\cdot\epsilon_{t}^x\left(\|A_t-A_{t-1}\|_\infty^2+\|y_t-y_{t-1}\|_2^2\right) .
\end{align*}
Similarly, we get 
\begin{align*}
    & \inner{q_t-q_t^*, -Y_t^\top A_tx_t}\\
    & \leq \frac{1}{\epsilon_{t}^y}\left(\|\wh{q}_t-q_t^*\|_2^2 - \|\wh{q}_{t+1}-q_t^*\|_2^2 - \|q_t-\wh{q}_{t+1}\|_2^2-\|q_t-\wh{q}_t\|_2^2\right) +c'\cdot\epsilon_{t}^y\left(\|A_t-A_{t-1}\|_\infty^2+\|x_t-x_{t-1}\|_2^2\right).
\end{align*}
Adding the above two inequalities and rearranging the terms, based on~\pref{eq:p-q-observe}, we have
\begin{align}
    &\frac{1}{\epsilon_{t}^x}\left(\|p_t-\wh{p}_{t+1}\|_2^2+\|p_t-\wh{p}_t\|_2^2\right)+\frac{1}{\epsilon_{t}^y}\left(\|q_t-\wh{q}_{t+1}\|_2^2+\|q_t-\wh{q}_t\|_2^2\right) \nonumber\\
    &\leq \frac{1}{\epsilon_{t}^x}\left(\|\wh{p}_t-p_t^*\|_2^2 - \|\wh{p}_{t+1}-p_t^*\|_2^2\right) + c'\cdot\epsilon_{t}^x(\|A_t-A_{t-1}\|_{\infty}^2+\|y_t-y_{t-1}\|_2^2) \nonumber\\
    &\qquad +\frac{1}{\epsilon_{t}^y}\left(\|\wh{q}_t-q_t^*\|_2^2 - \|\wh{q}_{t+1}-q_t^*\|_2^2\right) + c'\cdot\epsilon_{t}^y(\|A_t-A_{t-1}\|_{\infty}^2+\|x_t-x_{t-1}\|_2^2) + \lambda D(\|X_t^\delta\|_2+\|Y_t^\delta\|_2) \nonumber\\
    &\leq \frac{1}{\epsilon_{t}^x}\left(\|\wh{p}_t-p_t^*\|_2^2 - \|\wh{p}_{t+1}-p_{t+1}^*\|_2^2\right) + c'\cdot\epsilon_{t}^x(\|A_t-A_{t-1}\|_{\infty}^2+\|y_t-y_{t-1}\|_2^2) \nonumber\\
    &\qquad +\frac{1}{\epsilon_{t}^y}\left(\|\wh{q}_t-q_t^*\|_2^2 - \|\wh{q}_{t+1}-q_{t+1}^*\|_2^2\right) + c'\cdot\epsilon_{t}^y(\|A_t-A_{t-1}\|_{\infty}^2+\|x_t-x_{t-1}\|_2^2) \nonumber\\
    &\qquad + 2D\left(\frac{1}{\epsilon_{t}^x}+\frac{1}{\epsilon_{t}^y}\right)\cdot(\|p_t^*-p_{t+1}^*\|_1+\|q_t^*-q_{t+1}^*\|_1).   \label{eq:duality-upper-1}
\end{align}
In above, $D = \sqrt{2 (N + \max\{m,n\})} = \tilde{\Theta}(1)$ is another universal constant (ignoring the dependence on logarithmic factors in $T$) serving as the upper bound of $\norm{p - p'}_2$ and $\norm{q - q'}_2$ for any $p, p' \in \Delta_{|\calS_x|}$ and for any $q, q' \in \Delta_{|\calS_y|}$.

Next, we show that the desired duality gap upper bound can be related to the terms on the left-hand side of above inequality, namely, $\frac{1}{\epsilon_{t}^x}\left(\|p_t-\wh{p}_{t+1}\|_2^2+\|p_t-\wh{p}_t\|_2^2\right)+\frac{1}{\epsilon_{t}^y}\left(\|q_t-\wh{q}_{t+1}\|_2^2+\|q_t-\wh{q}_t\|_2^2\right)$. To see this, we first have the following inequalities from the update rule of the meta-algorithm as well as the first-order optimality condition: for any $p'\in \Delta_{|\calS_x|}$ and $q'\in \Delta_{|\calS_y|}$,
\begin{align*}
    (\wh{p}_{t+1}-\wh{p}_t+\epsilon_{t}^xX_t^\top A_ty_t+\epsilon_{t}^x\lambda X_t^\delta)^\top (p'-\wh{p}_{t+1})&\geq 0,\\
    (\wh{q}_{t+1}-\wh{q}_t-\epsilon_{t}^yY_t^\top A_t^\top x_t+\epsilon_{t}^y\lambda Y_t^\delta)^\top (q'-\wh{q}_{t+1})&\geq 0.
\end{align*}
Rearranging the terms and introducing the notations $\wt{x}_t\triangleq\sum_{i\in \calS_x}\wh{p}_{t+1,i}x_{t,i}$ and $\wt{y}_t\triangleq\sum_{j\in\calS_y}\wh{q}_{t+1,j}y_{t,j}$, we then have for any $p'\in \Delta_{|\calS_x|}$ and $q'\in \Delta_{|\calS_y|}$,
\begin{align*}
    &(\wh{p}_{t+1}-\wh{p}_t)^\top(p'-\wh{p}_{t+1})\\
    &\geq \epsilon_{t}^x(\wh{p}_{t+1}-p')^\top (X_t^\top A_ty_t+\lambda X_t^\delta) \\
    &\geq \epsilon_{t}^x(\wh{p}_{t+1}-p')^\top (X_t^\top A_t\wt{y}_t+X_t^\top A_t(y_t-\wt{y}_t)+\lambda X_t^\delta) \\
    &\geq \epsilon_{t}^x(\wh{p}_{t+1}-p')^\top X_t^\top A_t\wt{y}_t- c''\cdot\epsilon_{t}^x\|\wh{p}_{t+1}-p'\|_2\cdot \|q_t-\wh{q}_{t+1}\|_2 + \lambda\epsilon_{t}^x \inner{\wh{p}_{t+1}-p', X_t^\delta},
\end{align*}
and also
\begin{align*}
    &(\wh{q}_{t+1}-\wh{q}_t)^\top(q'-\wh{q}_{t+1})\\
    &\geq \epsilon_{t,y}(\wh{q}_{t+1}-q')^\top (-Y_t^\top A_tx_t+\lambda Y_t^\delta) \\
    &\geq \epsilon_{t}^y(\wh{q}_{t+1}-q')^\top (-Y_t^\top A_t\wt{x}_t-Y_t^\top A_t(x_t-\wt{x}_t)+\lambda Y_t^\delta) \\
    &\geq \epsilon_{t}^y(\wh{q}_{t+1}-q')^\top (-Y_t^\top A_t\wt{x}_t)- c''\cdot\epsilon_{t}^y\|\wh{q}_{t+1}-q'\|_2\cdot \|p_t-\wh{p}_{t+1}\|_2 + \lambda\epsilon_{t}^y \inner{\wh{q}_{t+1}-q', Y_t^\delta},
\end{align*}
where $c'' > 0$ is also a universal constant independent with the time horizon and the non-stationarity measures (ignoring the dependence on poly-logarithmic factors in $T$). Rearranging the terms arrives that
\begin{align*}
    (\wh{p}_{t+1}-p')^\top X_t^\top A_t\wt{y}_t \leq {} & \|p'-\wh{p}_{t+1}\|_2\cdot \otil\left( \frac{1}{\epsilon_{t}^x}\|\wh{p}_{t+1}-\wh{p}_t\|_2+\|q_t-\wh{q}_{t+1}\|_2\right) + \lambda \inner{p'-\wh{p}_{t+1}, X_t^\delta},\\
   (\wh{q}_{t+1}-q')^\top (-Y_t^\top A_t\wt{x}_t) \leq {} & \|q'-\wh{q}_{t+1}\|_2\cdot\otil\left( \frac{1}{\epsilon_{t}^y}\|\wh{q}_{t+1}-\wh{q}_t\|_2+\|p_t-\wh{p}_{t+1}\|_2\right) + \lambda \inner{q'-\wh{q}_{t+1}, Y_t^\delta}.
\end{align*}
Let $(\bar{x}^*_t, \bar{y}^*_t)$ be the corresponding best response for the strategy $(\wt{x}_t, \wt{y}_t)$ with respect to the payoff $A_t$, i.e., $\bar{x}^*_t = \argmin_{x\in\Delta_m} x^\top A_t\wt{y}_t$ and $\bar{y}^*_t = \argmax_{y\in\Delta_n} \wt{x}_t^\top A_ty$. Denote the duality gap bound of $(\wt{x}_t,\wt{y}_t)$ as $\alpha_t(\wt{x}_t, \wt{y}_t)\triangleq\max_{y}\wt{x}_t^\top A_ty - \min_{x} x^\top A_t\wt{y}_t$. Now we pick the comparator vectors $p' \in \Delta_{|\calS_x|}$ and $q' \in \Delta_{|\calS_y|}$ such that both $p'$ and $q'$ have supports only on the additional dummy base-learners and the supports finally form the best response of $\wt{x}_t, \wt{y}_t$, namely, $p_{i}' = q_{j}' = 0$ for $i = 1,\ldots,|\calS_{1,x}|$, $j = 1,\ldots,|\calS_{1,y}|$ and $p_{i+|\calS_{1,x}|}' = \bar{x}_{t,i}^*$ for $i = 1,\ldots,m$ and $q_{j+|\calS_{1,y}|}' = \bar{y}_{t,j}^*$ for $j = 1,\ldots,n$. Due to the construction, we confirm that $\inner{p',X_t^\delta} = 0$ and $\inner{q',Y_t^\delta} = 0$. 

As a result, combining the two inequalities on the above gives the following upper bound for duality gap:
\begin{align*}
    \alpha_t(\wt{x}_t,\wt{y}_t) = {} & (\wh{p}_{t+1}-p')^\top X_t^\top A_t\wt{y}_t + (\wh{q}_{t+1}-q')^\top (-Y_t^\top A_t\wt{x}_t)\\
    \leq {} & \otil\left( \frac{1}{\epsilon_{t}^x}\|\wh{p}_{t+1}-\wh{p}_t\|_2+\|q_t-\wh{q}_{t+1}\|_2 + \frac{1}{\epsilon_{t}^y}\|\wh{q}_{t+1}-\wh{q}_t\|_2+\|p_t-\wh{p}_{t+1}\|_2\right)\\
    \leq {} & \otil\left(\frac{1}{\epsilon_{t}^x}(\|\wh{p}_{t+1}-\wh{p}_t\|_2+\|p_{t+1}-\wh{p}_{t+1}\|_2)+\frac{1}{\epsilon_{t}^y}(\|\wh{q}_{t+1}-\wh{q}_t\|_2+\|q_t-\wh{q}_{t+1}\|_2)\right).
\end{align*}
The first inequality holds because $\|p'-\wh{p}_{t+1}\|_2\leq D = \tilde{\Theta}(1)$ and $\|q'-\wh{q}_{t+1}\|_2\leq D = \tilde{\Theta}(1)$ hold for any $t \in [T]$. The second inequality is obtained by scaling two terms with a factor of $\frac{1}{\epsilon_{t}^x}$ and $\frac{1}{\epsilon_{t}^y}$ respectively, where we notice that $1 \leq \frac{1}{L\epsilon_{t}^x}$ and $1 \leq \frac{1}{L\epsilon_{t}^y}$ are true as $\epsilon_{t}^x \leq \frac{1}{L}$ and $\epsilon_{t}^y\leq \frac{1}{L}$ holds for all $t \in [T]$. 

Then, by Cauchy-Schwarz inequality, we obtain the following upper bound for the square of duality gap bound of $(\wt{x}_t,\wt{y}_t)$:
\begin{align*}
    &\alpha_t^2(\wt{x}_t,\wt{y}_t) \\
    &\leq \otil\left(\Big(\frac{1}{\epsilon_{t}^x}(\|\wh{p}_{t+1}-\wh{p}_t\|_2+\|p_{t+1}-\wh{p}_{t+1}\|_2)+\frac{1}{\epsilon_{t}^y}(\|\wh{q}_{t+1}-\wh{q}_t\|_2+\|q_t-\wh{q}_{t+1}\|_2)\Big)^2\right)\\
    &\leq \otil\left(\Big(\frac{1}{\epsilon_{t}^x}+\frac{1}{\epsilon_{t}^y}\Big)\Big(\frac{1}{\epsilon_{t}^x}(\|\wh{p}_{t+1}-\wh{p}_t\|_2^2+\|p_{t+1}-\wh{p}_{t+1}\|_2^2)+\frac{1}{\epsilon_{t}^y}(\|\wh{q}_{t+1}-\wh{q}_{t}\|_2^2+\|q_t-\wh{q}_{t+1}\|_2^2)\Big)\right)\\
    &\leq \otil\left(\Big(\frac{1}{\epsilon_{t}^x}+\frac{1}{\epsilon_{t}^y} \Big) \Big(\frac{1}{\epsilon_{t}^x}\big(\|\wh{p}_t-p_t^*\|_2^2 - \|\wh{p}_{t+1}-p_{t+1}^*\|_2^2\big) + \epsilon_{t}^x\big(\|A_t-A_{t-1}\|_{\infty}^2+\|y_t-y_{t-1}\|_2^2\big)\Big)\right) \\
    &\qquad +\otil\left(\Big(\frac{1}{\epsilon_{t}^x}+\frac{1}{\epsilon_{t}^y} \Big)\Big(\frac{1}{\epsilon_{t}^y}\big(\|\wh{q}_t-q_t^*\|_2^2 - \|\wh{q}_{t+1}-q_{t+1}^*\|_2^2\big) + \epsilon_{t}^y\big(\|A_t-A_{t-1}\|_{\infty}^2+\|x_t-x_{t-1}\|_2^2\big) \Big)\right)\\
    &\qquad + \otil\left(\Big(\frac{1}{\epsilon_{t}^x}+\frac{1}{\epsilon_{t}^y}\Big)^2\cdot(\|p_t^*-p_{t+1}^*\|_1+\|q_t^*-q_{t+1}^*\|_1)\right).
\end{align*}
Notably, the last step makes use of the inequality in~\pref{eq:duality-upper-1} and $\lambda=\frac{\gamma L}{2}=\wt{\Theta}(1)$. For simplicity, we introduce the notation $\frac{1}{\epsilon_t}\triangleq\frac{1}{\epsilon_{t}^x}+\frac{1}{\epsilon_{t}^y}$. Taking a summation on the squared duality gap over all rounds and using the fact that $\epsilon_{t}^x,\epsilon_{t}^y\leq \otil(1)$ and $\epsilon_{t}^x,\epsilon_{t}^y$ are non-increasing in $t$, we have (omitting all dimension and $\poly(\log T)$ factors)
\begin{align*}
    &\sum_{t=1}^T\alpha_t^2(\wt{x}_t,\wt{y}_t)\\
    &\leq \sum_{t=1}^T\otil\left(\left(\frac{1}{\epsilon_{t+1}^x\cdot\epsilon_{t+1}}-\frac{1}{\epsilon_{t}^x\cdot\epsilon_t}\right)\|\wh{p}_{t+1}-p_{t+1}^*\|_2^2\right)+\sum_{t=1}^T\otil\left(\left(\frac{1}{\epsilon_{t+1}^y\cdot\epsilon_{t+1}}-\frac{1}{\epsilon_{t}^y\cdot\epsilon_t}\right)\|\wh{q}_{t+1}-q_{t+1}^*\|_2^2\right)\\
    &\qquad +\frac{1}{\epsilon_{T}^2}\otil\left(P_T\right)+\frac{1}{\epsilon_T}\sum_{t=2}^T\otil\Big(\|A_t-A_{t-1}\|_{\infty}^2+\|y_t-y_{t-1}\|_2^2+\|x_t-x_{t-1}\|_2^2\Big)\\
    &\leq \otil\left(\frac{1+P_T}{\epsilon_{T+1}^2}\right)+\frac{1}{\epsilon_{T+1}}\sum_{t=2}^T\otil\Big(\|A_t-A_{t-1}\|_{\infty}^2+\|x_t-x_{t-1}\|_2^2+\|y_t-y_{t-1}\|_2^2\Big),
\end{align*}
where the last inequality uses $\|\wh{p}_{t+1}-p_{t+1}^*\|_2\leq \otil(1)$, $\|\wh{q}_{t+1}-q_{t+1}^*\|_2\leq \otil(1)$.
According to~\pref{lemma:stability-NE-variation} and~\pref{lemma:stability-payoff-variation},
\begin{align*}
     \sum_{t=2}^T\|x_t-x_{t-1}\|_{2}^2 + \sum_{t=2}^T\|y_t-y_{t-1}\|_{2}^2  &\leq \otil\left(\min\left\{\sqrt{(1+V_T)(1+P_T)}+P_T, 1 + W_T\right\}\right).
\end{align*}

In addition, according to the definition of $\epsilon_{t}^x$ and $\epsilon_{t}^y$, we have
\begin{align*}
& \frac{1}{\epsilon_{T+1}^x} = \sqrt{L^2 + \sum_{t=2}^T \norm{A_ty_t - A_{t-1}y_{t-1}}_{\infty}^2} \leq \otil\left(\sqrt{1+V_T+\min\{P_T,W_T\}}\right),\\
& \frac{1}{\epsilon_{T+1}^y} = \sqrt{L^2 + \sum_{t=2}^T \norm{x_t^\T A_t - x_{t-1}^\T A_{t-1}}_{\infty}^2}\leq \otil\left(\sqrt{1+V_T+\min\{P_T,W_T\}}\right),
\end{align*}
where the last inequality is due to the gradient-variation bound in~\pref{lemma:gradient-variation}. Therefore, combining all above inequalities can achieve the following result on the squared duality gap:
\begin{align*}
    \sum_{t=1}^T\alpha_t^2(\wt{x}_t,\wt{y}_t) &\leq \otil\Big((1+P_T)(1+V_T + \min\{P_T,W_T\})\Big)+\otil\Big((1+V_T+\min\{W_T,P_T\})^{\frac{3}{2}}\Big) \\
    &=\otil\left((1+V_T+\min\{P_T,W_T\})\left(\sqrt{1+V_T+\min\{P_T,W_T\}}+P_T\right)\right).
\end{align*}
We further introduce the notation $Q_T\triangleq V_t + \min\{P_T,W_T\}$ to simplify the presentation. Then, by Cauchy-Schwarz inequality we have
\begin{equation}
    \label{eq:duality-gap-bound-intermidiate}
    \sum_{t=1}^T\alpha_t(\wt{x}_t,\wt{y}_t)\leq \Ot\left( \sqrt{T(1+Q_T)(\sqrt{1+Q_T} + P_T)}\right) = \otil\left(T^{\frac{1}{2}} (1+Q_T^{\frac{3}{2}}+P_TQ_T)^{\frac{1}{2}}\right).
\end{equation}
We finally transform the above bound back to $\alpha_t(x_t,y_t)$ by noticing that
\begin{align*}
    &\sum_{t=1}^T\alpha_t(x_t,y_t)\\
    &=\sum_{t=1}^T\left(\max_{y\in \Delta_n}x_t^\top A_ty-\min_{x\in \Delta_m}x^\top A_ty_t\right) \\
    &=\sum_{t=1}^T\left(\max_{y\in \Delta_n}\wt{x}_t^\top A_ty-\min_{x\in \Delta_m}x^\top A_t\wt{y}_t\right) + \sum_{t=1}^T\left(\max_{y\in \Delta_n}x_t^\top A_ty-\max_{y\in \Delta_n}\wt{x}_t^\top A_ty\right)+\sum_{t=1}^T\left(\min_{x\in \Delta_m}x^\top A_t\wt{y}_t-\min_{x\in \Delta_m}x^\top A_ty_t\right) \\
    &\leq \sum_{t=1}^T\alpha_t(\wt{x}_t, \wt{y}_t)+\sum_{t=1}^T\otil\left(\|p_t-\wh{p}_{t+1}\|_2+\|q_t-\wh{q}_{t+1}\|_2\right) \\
    &\leq \sum_{t=1}^T\alpha_t(\wt{x}_t, \wt{y}_t)+\otil\left(\sqrt{T\sum_{t=1}^T(\|p_t-\wh{p}_{t+1}\|_2^2+\|q_t-\wh{q}_{t+1}\|_2^2)}\right) \tag{by Cauchy-Schwarz inequality}\\
    &\leq \sum_{t=1}^T\alpha_t(\wt{x}_t, \wt{y}_t)+\otil\left(\sqrt{T\min\{\sqrt{(1+V_T)(1+P_T)}+P_T, 1+W_T\}}\right). \tag{by~\pref{lemma:stability-NE-variation} and~\pref{lemma:stability-payoff-variation}}\\
    &\leq \otil\left(T^{\frac{1}{2}}\left(1+Q_T^{\frac{3}{2}}+P_TQ_T\right)^{\frac{1}{2}}\right)+\otil\left(\sqrt{T(1+Q_T)}\right) \tag{by~\pref{eq:duality-gap-bound-intermidiate} and Cauchy-Schwarz inequality}\\
    &\leq \otil\left(T^{\frac{1}{2}}\left(1+Q_T^{\frac{3}{2}}+P_TQ_T\right)^{\frac{1}{2}}\right).
\end{align*}

To summarize, combining the both types of upper bounds for duality gap, we finally achieve the following guarantee:
\begin{align*}
\sum_{t=1}^T \max_{y\in\Delta_n}x_t^\top A_ty- \sum_{t=1}^T  \min_{x\in \Delta_{m}}x^\top A_ty_t \leq\Ot\Big(\min\{T^{\frac{3}{4}} \big( 1 + Q_T \big)^{\frac{1}{4}}, T^{\frac{1}{2}}(1+Q_T^{\frac{3}{2}}+P_TQ_T)^{\frac{1}{2}}\} \Big),
\end{align*}
which completes the proof of~\pref{thm:duality-gap-changing}.
\end{proof}
\section{Key Lemmas}
\label{appendix:key-lemmas}
This section presents several key lemmas used in proving our theoretical results. 

We first provide an analysis for the general dynamic regret of the meta-base two-layer approach, which serves as one of the key technical tools for proving upper bounds for the three performance measures. The result is shown in~\pref{lemma:NE-variation-dynamic-regret}, and we emphasize that the regret bounds hold for \emph{any} comparator sequence, which is crucial and useful in the subsequent analysis.
\begin{lemma}[General dynamic regret]
\label{lemma:NE-variation-dynamic-regret}
\pref{alg:x-player} guarantees that $x$-player's dynamic regret with respect to any comparator sequence $u_1,\ldots,u_T \in \Delta_m$ is bounded by 
\begin{equation}
\label{eq:dynamic-regret-x-general}
\begin{split}
    &\sum_{t=1}^T \x_t^\T A_t \y_t - \sum_{t=1}^T \u_t^\T A_t \y_t \\
    &\leq \order\left(\frac{\alpha(1+P_T^u)}{\eta_i^x}\right)+\eta_i^x c\beta\sum_{t=2}^T\|A_t-A_{t-1}\|_{\infty}^2+\eta_i^xc\beta\sum_{t=2}^T\|y_t-y_{t-1}\|_{1}^2+\left(\lambda-\frac{\gamma}{\eta_i^x}\right)\sum_{t=2}^T\|x_{t,i}-x_{t-1,i}\|_1^2 \\
    &\qquad - L\sum_{t=1}^T(\|p_t-\wh{p}_{t+1}\|_2^2+\|p_t-\wh{p}_t\|_2^2)-\lambda\sum_{t=2}^T\sum_{i\in \calS_x}p_{t,i}\|x_{t,i}-x_{t-1,i}\|_1^2+ \otil(1),
\end{split}
\end{equation}
for a specific $c=\tilde{\Theta}(1)$ and any compared base-learner's index $i \in \calS_{1,x}$. 

Similarly, \pref{alg:y-player} guarantees that $y$-player's dynamic regret with respect to any comparator sequence $v_1,\ldots,v_T \in \Delta_n$ is at most 
\begin{equation}
\label{eq:dynamic-regret-y-general}
\begin{split}
    & -\sum_{t=1}^T \x_t^\T A_t \y_t + \sum_{t=1}^T \x_t^\T A_t \v_t \\
    &\leq \order\left(\frac{\alpha(1+P_T^v)}{\eta_j^y}\right)+\eta_j^yc\beta\sum_{t=2}^T\|A_t-A_{t-1}\|_{\infty}^2+\eta_j^yc\beta\sum_{t=2}^T\|x_t-x_{t-1}\|_{1}^2+\left(\lambda-\frac{\gamma}{\eta_j^y}\right)\sum_{t=2}^T\|y_{t,i}-y_{t-1,i}\|_1^2 \\
    &\qquad - L\sum_{t=1}^T(\|q_t-\wh{q}_{t+1}\|_2^2+\|q_t-\wh{q}_t\|_2^2)-\lambda\sum_{t=2}^T\sum_{i\in \calS_x}q_{t,i}\|y_{t,i}-y_{t-1,i}\|_1^2+ \otil(1),
\end{split}
\end{equation}
which also holds for any compared base-learner's index $j \in \calS_{1,y}$.
\end{lemma}

\begin{proof}

We consider the dynamic regret for $x$-player and similar results hold for $y$-player. First, we decompose the dynamic regret for $x$-player into the sum of the meta-regret and base-regret. Specifically, for any $i\in \calS_{1,x}$, we have
\begin{align*}
    \sum_{t=1}^Tx_t^\top A_ty_t - \sum_{t=1}^Tu_tA_ty_t= \underbrace{\sum_{t=1}^Tx_t^\top A_ty_t-\sum_{t=1}^Tx_{t,i}^\top A_ty_t}_{\meta} + \underbrace{\sum_{t=1}^Tx_{t,i}^\top A_ty_t-\sum_{t=1}^Tu_t^\top A_ty_t}_{\base}.
\end{align*}
We now give upper bounds for the meta-regret and base-regret respectively.

First, we consider the meta-regret, which is essentially the static regret with respect to any base-learner with an index $i \in \calS_{1,x}$. 
Recall several notations introduced in the algorithm. For $x$-player, $\ell_t^x=X_t^\top A_ty_t+\lambda X_t^\delta$ and $m_t^x=X_t^\top A_{t-1}y_{t-1}+\lambda X_t^\delta$, where $X_t=[x_{t,1};x_{t,2};\ldots,x_{t,|\calS_x|}]$ and $X_t^\delta=[\|x_{t,1}-x_{t-1,1}\|_1^2;\|x_{t,2}-x_{t-1,2}\|_1^2;\ldots;\|x_{t,|\calS_x|}-x_{t-1,|\calS_x|}\|_1^2]$. Note that $\calS_x\triangleq\calS_{1,x}\cup\calS_{2,x}$ with $\calS_{1,x} = [N]$ and $\calS_{2,x} = \{N+1, \ldots, N+m\}$, where $N=\lfloor \frac{1}{2}\log_2 T \rfloor + 1$.
The notations for $y$-player are similarly defined and we do not restate here for conciseness. 
According to the general result of~\pref{thm:general-dynamic-regret-OMD} with $f_t(p_t)=\inner{p_t, X_t^\top A_ty_t+\lambda X_t^\delta}$, $M_t=X_t^\top A_{t-1}y_{t-1}+\lambda X_t^\delta$, and $u_t=e_i\in\Delta_{|\calS_x|}$ for all $t\in [T]$, we have the following regret bound for the meta-algorithm,
\begin{align*}
    & \sum_{t=1}^T \inner{p_t - e_i, X_t^\top A_ty_t+\lambda X_t^\delta} \\
    & \leq \sum_{t=2}^T\epsilon_t^x \norm{X_t^\top A_ty_t - X_t^\top A_{t-1}y_{t-1}}_2^2 +  \sum_{t=1}^T\frac{1}{\epsilon_t^x} \Big(\|\wh{p}_t-e_i\|_2^2 - \|\wh{p}_{t+1}-e_i\|_2^2 \Big) - \sum_{t=1}^T\frac{1}{\epsilon_t^x} \Big( \|p_t-\wh{p}_{t+1}\|_2^2+\|p_t-\wh{p}_t\|_2^2\Big)+\otil(1)\\
    & \leq c_1 \sum_{t=2}^T  \epsilon_t^x \norm{A_t y_t - A_{t-1} y_{t-1}}_\infty^2 + \frac{1}{\epsilon_{T}^x} \sum_{t=1}^T\Big(\|\wh{p}_t-e_i\|_2^2 - \|\wh{p}_{t+1}-e_i\|_2^2 \Big) - \sum_{t=1}^T\frac{1}{\epsilon_t^x} \Big( \|p_t-\wh{p}_{t+1}\|_2^2+\|p_{t}-\wh{p}_{t}\|_2^2\Big)+\otil(1)\\
    & \leq c_1 \sum_{t=2}^T  \frac{\norm{A_t y_t - A_{t-1} y_{t-1}}_\infty^2}{\sqrt{L^2 + \sum_{s=2}^{t-1} \norm{A_s y_s - A_{s-1} y_{s-1}}_\infty^2}} + \frac{\otil(1)}{\epsilon_{T}^x} - L\sum_{t=1}^T \left(\norm{p_t - \wh{p}_{t+1}}_2^2+\norm{p_t - \wh{p}_{t}}_2^2\right) \tag{by definition of $\epsilon_t^x$ and $\max_{p\in \Delta_{|\calS_x|}}\|p-e_i\|_2^2\leq \otil(1)$}\\
    & \leq c_2 \sqrt{L^2 + \sum_{t=2}^{T} \norm{A_t y_t - A_{t-1} y_{t-1}}_\infty^2} + \otil(1) - L\sum_{t=1}^T(\norm{p_t - \wh{p}_{t+1}}_2^2+\norm{p_t-\wh{p}_t}_2^2),
\end{align*}
where $c_1, c_2=\wt{\Theta}(1)$ and the last step holds by~\pref{lem:self-confident-variant}.

Next, we consider the base-regret. Since the base-algorithm $\calB_i$ satisfies the \drvu~property, the base-regret is upper bounded as follows:
\begin{align*}
    \sum_{t=1}^Tx_{t,i}^\top A_ty_t-\sum_{t=1}^Tu_t^\top A_ty_t &\leq \frac{\alpha (1+P_T^u)}{\eta_i^x}+\eta_i^x\beta\sum_{t=2}^T\|A_ty_t-A_{t-1}y_{t-1}\|_{\infty}^2-\frac{\gamma}{\eta_i^x}\sum_{t=2}^T\|x_{t,i}-x_{t-1,i}\|_1^2.
\end{align*}

Summing up the above two inequalities, we achieve the following dynamic regret guarantee for the $x$-player:
\begin{align*}
    & \sum_{t=1}^T \x_t^\T A_t \y_t - \sum_{t=1}^T \u_t^\T A_t \y_t\\
    &\leq c_2 \sqrt{L^2 + \sum_{t=2}^{T} \norm{A_t y_t - A_{t-1} y_{t-1}}_\infty^2} + \otil(1) - L\sum_{t=1}^T(\norm{p_t - \wh{p}_{t+1}}_2^2+\norm{p_t-\wh{p}_t}_2^2) \\
    &\qquad + \frac{\alpha (1+P_T^u)}{\eta_i^x}+\eta_i^x\beta\sum_{t=2}^T\|A_ty_t-A_{t-1}y_{t-1}\|_{\infty}^2-\frac{\gamma}{\eta_i^x}\sum_{t=2}^T\|x_{t,i}-x_{t-1,i}\|_1^2 \\
    &\qquad + \lambda\sum_{t=2}^T\|x_{t,i}-x_{t-1,i}\|_1^2 - \lambda\sum_{t=2}^T\sum_{i\in \calS_x}p_{t,i}\|x_{t,i}-x_{t-1,i}\|_1^2 + \otil(1)\\
    &\leq \order\left(\frac{\alpha(1+P_T^u)}{\eta_i^x}\right)+\eta_i^x(c_2^2+\beta)\sum_{t=2}^T\|A_ty_t-A_{t-1}y_{t-1}\|_{\infty}^2+\left(\lambda-\frac{\gamma}{\eta_i^x}\right)\sum_{t=2}^T\|x_{t,i}-x_{t-1,i}\|_1^2 \\
    &\qquad - L\sum_{t=1}^T(\|p_t-\wh{p}_{t+1}\|_2^2+\|p_t-\wh{p}_t\|_2^2)-\lambda\sum_{t=2}^T\sum_{i\in \calS_x}p_{t,i}\|x_{t,i}-x_{t-1,i}\|_1^2 + \otil(1)\\
    &\leq \order\left(\frac{\alpha(1+P_T^u)}{\eta_i^x}\right)+\eta_i^xc\beta\sum_{t=2}^T\|A_t-A_{t-1}\|_{\infty}^2+\eta_i^xc\beta\sum_{t=2}^T\|y_t-y_{t-1}\|_{1}^2+\left(\lambda-\frac{\gamma}{\eta_i^x}\right)\sum_{t=2}^T\|x_{t,i}-x_{t-1,i}\|_1^2 \\
    &\qquad - L\sum_{t=1}^T(\|p_t-\wh{p}_{t+1}\|_2^2+\|p_t-\wh{p}_t\|_2^2)-\lambda\sum_{t=2}^T\sum_{i\in \calS_x}p_{t,i}\|x_{t,i}-x_{t-1,i}\|_1^2+ \otil(1),
\end{align*}
where $c=\wt{\Theta}(1)$. This proves~\pref{eq:dynamic-regret-x-general}. Repeating the above analysis for $y$-player proves~\pref{eq:dynamic-regret-y-general}.
\end{proof}

The following lemma presents the stability lemma in terms of the   non-stationarity measure $P_T$ (that is, the NE variation). We give the stability upper bounds from the aspects of meta-algorithm and final decisions. For simplicity, we assume that the $\drvu$ parameters $(\alpha,\beta,\gamma)$ are all $\wt{\Theta}(1)$, which is indeed the case for standard algorithms as proven in~\pref{appendix:DRVU-property}.
\begin{lemma}[NE-variation stability]
\label{lemma:stability-NE-variation}
Suppose that $x$-player follows~\pref{alg:x-player} and $y$-player follows~\pref{alg:y-player}. Then, the following inequalities hold simultaneously.
In the meta-algorithm aspect, we have
\begin{align}
\label{eq:stability-NE-variation-meta}
\sum_{t=1}^T\|p_t-\wh{p}_{t+1}\|_2^2 &\leq \otil\left(\sqrt{(1+V_T)(1+P_T)}+P_T\right),\\
\sum_{t=1}^T\|q_t-\wh{q}_{t+1}\|_2^2 &\leq \otil\left(\sqrt{(1+V_T)(1+P_T)}+P_T\right);
\end{align}
in the final decision aspect, we have
\begin{align}
    \label{eq:stability-NE-variation-overall}
    \sum_{t=2}^T \norm{x_t - x_{t-1}}_1^2 \leq \Ot\left(\sqrt{(1+V_T)(1+P_T)} + P_T\right),\\
    \sum_{t=2}^T \norm{y_t - y_{t-1}}_1^2 \leq \Ot\left(\sqrt{(1+V_T)(1+P_T)} + P_T\right).
\end{align}
\end{lemma}

\begin{proof}
Let $(x_t^*,y_t^*)$ denote the Nash equilibrium of the game matrix $A_t$ at round $t$. Based on~\pref{lemma:NE-variation-dynamic-regret} with a choice of $\{u_t\}_{t=1}^T=\{x_t^*\}_{t=1}^T$ and $\{v_t\}_{t=1}^T=\{y_t^*\}_{t=1}^T$, we have
\begin{align*}
    & \sum_{t=1}^T \x_t^\T A_t \y_t^* - \sum_{t=1}^T \x_t^{*\T} A_t \y_t\\
    & = \sum_{t=1}^T \x_t^\T A_t \y_t^* - \sum_{t=1}^T \x_t^\T A_t \y_t + \sum_{t=1}^T \x_t^\T A_t \y_t - \sum_{t=1}^T \x_t^{*\T} A_t \y_t \\
    &\leq \otil\left(\frac{\alpha(1+P_T^x)}{\eta_i^x}+\frac{\alpha(1+P_T^y)}{\eta_j^y}+\beta(\eta_i^x+\eta_j^y)V_T\right) + \eta_i^xc\beta\sum_{t=2}^T\|y_t-y_{t-1}\|_1^2 + \eta_j^yc\beta\sum_{t=2}^T\|x_t-x_{t-1}\|_1^2\\
    &\qquad + \left(\lambda-\frac{\gamma}{\eta_i^x}\right)\sum_{t=2}^T\|x_{t,i}-x_{t-1,i}\|_1^2+\left(\lambda-\frac{\gamma}{\eta_j^y}\right)\sum_{t=2}^T\|y_{t,j}-y_{t-1,j}\|_1^2 \\
    &\qquad -L\sum_{t=1}^T(\|p_t-\wh{p}_{t+1}\|_2^2+\|p_t-\wh{p}_t\|_2^2+\|q_t-\wh{q}_{t+1}\|_2^2+\|q_t-\wh{q}_t\|_2^2) \\
    &\qquad - \lambda\sum_{t=2}^T\sum_{i\in S_x}p_{t,i}\|x_{t,i}-x_{t-1,i}\|_1^2 - \lambda\sum_{t=2}^T\sum_{j\in S_y}q_{t,j}\|y_{t,j}-y_{t-1,j}\|_1^2\\
    &\leq \otil\left(\frac{\alpha(1+P_T^x)}{\eta_i^x}+\frac{\alpha(1+P_T^y)}{\eta_j^y}+\beta(\eta_i^x+\eta_j^y)V_T\right) + 2\eta_i^xc\beta\sum_{t=2}^T\|q_t-q_{t-1}\|_1^2 + 2\eta_i^xc\beta\sum_{t=2}^T\sum_{j\in \calS_y}\|y_{t,j}-y_{t-1,j}\|_1^2\\
    &\qquad+2\eta_j^yc\beta\sum_{t=2}^T\|p_t-p_{t-1}\|_1^2 +
    2\eta_j^yc\beta\sum_{t=2}^T\sum_{i\in \calS_x}\|x_{t,i}-x_{t-1,i}\|_1^2 \\
    &\qquad + \left(\lambda-\frac{\gamma}{\eta_i^x}\right)\sum_{t=2}^T\|x_{t,i}-x_{t-1,i}\|_1^2+\left(\lambda-\frac{\gamma}{\eta_j^y}\right)\sum_{t=2}^T\|y_{t,j}-y_{t-1,j}\|_1^2 \\
    &\qquad -\frac{L}{2}\sum_{t=1}^T(\|p_t-\wh{p}_{t+1}\|_2^2+\|p_t-\wh{p}_t\|_2^2+\|q_t-\wh{q}_{t+1}\|_2^2+\|q_t-\wh{q}_t\|_2^2)-\frac{L}{4}\sum_{t=2}^T(\|p_t-p_{t-1}\|_2^2+\|q_t-q_{t-1}\|_2^2)\\
    &\qquad - \lambda\sum_{t=2}^T\sum_{i\in S_x}p_{t,i}\|x_{t,i}-x_{t-1,i}\|_1^2 - \lambda\sum_{t=2}^T\sum_{j\in S_y}q_{t,j}\|y_{t,j}-y_{t-1,j}\|_1^2\\
    & \leq \otil\left(\frac{\alpha(1+P_T^x)}{\eta_i^x}+\frac{\alpha(1+P_T^y)}{\eta_j^y}+\beta(\eta_i^x+\eta_j^y)V_T\right) + \left(2\eta_i^xc\beta-\frac{L}{4}\right)\sum_{t=2}^T\|q_t-q_{t-1}\|_1^2 +\left(2\eta_j^yc\beta-\frac{L}{4}\right)\sum_{t=2}^T\|p_t-p_{t-1}\|_1^2\\
    &\qquad+ \left(2\eta_i^xc\beta-\lambda\right)\sum_{t=2}^T\sum_{j\in \calS_y}q_{t,j}\|y_{t,j}-y_{t-1,j}\|_1^2+
    \left(2\eta_j^yc\beta-\lambda\right)\sum_{t=2}^T\sum_{i\in \calS_x}p_{t,i}\|x_{t,i}-x_{t-1,i}\|_1^2\\
    &\qquad + \left(\lambda-\frac{\gamma}{\eta_i^x}\right)\sum_{t=2}^T\|x_{t,i}-x_{t-1,i}\|_1^2+\left(\lambda-\frac{\gamma}{\eta_j^y}\right)\sum_{t=2}^T\|y_{t,j}-y_{t-1,j}\|_1^2 \\
    &\qquad -\frac{L}{2}\sum_{t=1}^T(\|p_t-\wh{p}_{t+1}\|_2^2+\|p_t-\wh{p}_t\|_2^2+\|q_t-\wh{q}_{t+1}\|_2^2+\|q_t-\wh{q}_t\|_2^2).
\end{align*} 
According to the choice of $L=\max\left\{4,\sqrt{16c\beta}, \sqrt{\frac{8c\beta}{\gamma}}\right\}=\wt{\Theta}(1)$, $\lambda=\frac{\gamma L}{2}$, and $\eta_i^x,\eta_j^y\leq \frac{1}{L}$, it can be verified that
\begin{equation}\label{eq:negative-coef}
    \begin{split}
    &2\eta_i^xc\beta - \frac{L}{4}\leq \frac{2c\beta}{L} - \frac{L}{4}\leq -\frac{L}{8};\qquad  2\eta_j^yc\beta - \frac{L}{4}\leq \frac{2c\beta}{L} - \frac{L}{4}\leq -\frac{L}{8};\\
    &2\eta_i^xc\beta - \lambda = \frac{2c\beta}{L}-\frac{\gamma L}{2} \leq -\frac{\gamma L}{4};\qquad  2\eta_j^yc\beta - \lambda = \frac{2c\beta}{L}-\frac{\gamma L}{2} \leq -\frac{\gamma L}{4};\\
    &\lambda - \frac{\gamma}{\eta_i^x} = \frac{\gamma L}{2} - \frac{\gamma}{\eta_i^x} \leq \frac{\gamma}{2}\left(L - \frac{2}{\eta_i^x}\right) \leq -\frac{\gamma}{2\eta_i^x},\\ &\lambda - \frac{\gamma}{\eta_j^y} = \frac{\gamma L}{2} - \frac{\gamma}{\eta_j^y} \leq \frac{\gamma}{2}\left(L - \frac{2}{\eta_j^y}\right) \leq -\frac{\gamma}{2\eta_j^y}.
    \end{split}
\end{equation}
In addition, since $(x_t^*,y_t^*)$ is the Nash equilibrium of $A_t$, it follows that
$$x_t^\top A_ty_t^*-x_t^{*\T} A_ty_t\geq x_t^{*\top} A_ty_t^*-x_t^{*\T} A_ty_t^*=0.$$
Therefore, as $\alpha$, $\beta$, $\gamma=\wt{\Theta}(1)$, we have the following inequalities simultaneously.
\begin{align*}
    \sum_{t=2}^T\|p_t-p_{t-1}\|_1^2\leq \frac{8}{L}\cdot \otil\left(\frac{1+P_T^x}{\eta_{i}^x}+\frac{1+P_T^y}{\eta_{j}^y}+(\eta_i^x+\eta_j^y)V_T\right) \leq \otil\left(\sqrt{(1+V_T)(1+P_T)}+P_T\right),
\end{align*}
where the last inequality is because we pick $\eta_i^x$ and $\eta_j^y$ to be the one such that $\eta_i^x\in [\frac{1}{2}\eta_*^x, 2\eta_*^x]$, $\eta_j^y\in [\frac{1}{2}\eta_*^y, 2\eta_*^y]$ where $\eta_*^x=\min\left\{\sqrt{\frac{1+P_T^x}{1+V_T}}, \frac{1}{L}\right\}$ and $\eta_*^y=\min\left\{\sqrt{\frac{1+P_T^y}{1+V_T}}, \frac{1}{L}\right\}$. This is achievable based on the choice of our step size pool. Similarly, we have
\begin{align*}
    &\sum_{t=2}^T\|q_t-q_{t-1}\|_1^2\leq \otil\left(\sqrt{(1+V_T)(1+P_T)}+P_T\right),\\
    &\sum_{t=1}^T\|p_t-\wh{p}_{t+1}\|_1^2\leq \otil\left(\sqrt{(1+V_T)(1+P_T)}+P_T\right),\\
    &\sum_{t=1}^T\|q_t-\wh{q}_{t+1}\|_1^2\leq \otil\left(\sqrt{(1+V_T)(1+P_T)}+P_T\right),\\
    &\sum_{t=2}^T\sum_{i\in\calS_x}p_{t,i}\|x_{t,i}-x_{t-1,i}\|_1^2\leq \otil\left(\sqrt{(1+V_T)(1+P_T)}+P_T\right),\\
    &\sum_{t=2}^T\sum_{j\in\calS_y}q_{t,j}\|y_{t,j}-y_{t-1,j}\|_1^2\leq \otil\left(\sqrt{(1+V_T)(1+P_T)}+P_T\right).
\end{align*}
In addition, note that
\begin{equation}
    \label{eq:x-split}
    \begin{split}
        &\sum_{t=2}^T \norm{x_t - x_{t-1}}_1^2 \\
    &= \sum_{t=2}^T \left\|\sum_{i\in \calS_x}p_{t,i}x_{t,i}-\sum_{i\in \calS_x}p_{t-1,i}x_{t-1,i}\right\|_1^2 \\
    &= \sum_{t=2}^T \left\|\sum_{i\in \calS_x}p_{t,i}(x_{t,i}-x_{t-1,i})-\sum_{i\in \calS_x}(p_{t-1,i}-p_{t,i})x_{t-1,i}\right\|_1^2 \\
    &\leq 2\sum_{t=2}^T\left\|\sum_{i\in \calS_x}p_{t,i}(x_{t,i}-x_{t-1,i})\right\|_1^2+2\sum_{t=2}^T\left\|\sum_{i\in \calS_x}(p_{t,i}-p_{t-1,i})x_{t-1,i}\right\|_1^2 \\
    &\leq 2\sum_{t=2}^T\left(\sum_{i\in \calS_x}p_{t,i}\left\|x_{t,i}-x_{t-1,i}\right\|_1\right)^2+2\sum_{t=2}^T\left(\sum_{i\in \calS_x}|p_{t,i}-p_{t-1,i}|\cdot\left\|x_{t-1,i}\right\|_1\right)^2 \\
    &\leq 2\sum_{t=2}^T\left(\sum_{i\in \calS_x}p_{t,i}\left\|x_{t,i}-x_{t-1,i}\right\|_1\right)^2+2\sum_{t=2}^T\left(\sum_{i\in \calS_x}|p_{t,i}-p_{t-1,i}|\cdot\left\|x_{t-1,i}\right\|_1\right)^2\\
    &\leq 2\sum_{t=2}^T\sum_{i\in \calS_x}p_{t,i}\|x_{t,i}-x_{t-1,i}\|_1^2+2\sum_{t=2}^T\|p_t-p_{t-1}\|_1^2,
    \end{split}
\end{equation}
where the last inequality is by Cauchy-Schwarz inequality and $\|x_{t,i}\|_1=1$ for all $i\in \calS_x$.
Similarly, we have for $y$-player,
\begin{equation}
\label{eq:y-split}
    \sum_{t=2}^T \norm{y_t - y_{t-1}}_1^2 \leq 2\sum_{t=2}^T\sum_{j\in \calS_y}q_{t,j}\|x_{t,j}-x_{t-1,j}\|_1^2+2\sum_{t=2}^T\|q_t-q_{t-1}\|_1^2.
\end{equation}

Based on the above results, we further have
\begin{align*}
    \sum_{t=2}^T \norm{x_t - x_{t-1}}_1^2 \leq &  2 \sum_{t=2}^T \norm{p_t - p_{t-1}}_1^2 + 2 \sum_{t=2}^T \sum_{i=1}^N p_{t,i}\norm{x_{t,i} - x_{t-1,i}}_1^2 \leq \otil\left(\sqrt{(1+V_T)(1+P_T)}+P_T\right),\\
    \sum_{t=2}^T \norm{y_t - y_{t-1}}_1^2 \leq &  2 \sum_{t=2}^T \norm{q_t - q_{t-1}}_1^2 + 2 \sum_{t=2}^T \sum_{i=1}^N q_{t,i}\norm{y_{t,i} - y_{t-1,i}}_1^2 \leq \otil\left(\sqrt{(1+V_T)(1+P_T)}+P_T\right),
\end{align*}
which completes the proof.
\end{proof}

The following lemma shows the relationship between dynamic NE-regret and the individual regret.

\begin{lemma}[Dynamic-NE-regret-to-individual-regret  conversation]
For arbitrary sequences of $\{x_t\}_{t=1}^T$, $\{y_t\}_{t=1}^T$ and $\{A_t\}_{t=1}^T$, where $x_t\in \Delta_m$, $y_t\in \Delta_n$, and $A_t\in\mathbb{R}^{m\times n}$, $\forall t\in[T]$, we have
\label{lemma:payoff-variation-conversation}
\begin{equation}
    \label{eq:payoff-varition-conversation}
    \left|\sum_{t=1}^Tx_t^\top A_ty_t - \sum_{t=1}^T\min_x\max_y x^\top A_ty \right| \leq \max\left\{\sum_{t=1}^Tx_t^\top A_ty_t - \sum_{t=1}^T x^{*\T} A_t y_t, \sum_{t=1}^T x_t^{\T} A_t y^* - \sum_{t=1}^Tx_t^\top A_ty_t\right\} + 2W_T,
\end{equation}
where $W_T = \sum_{t=1}^T\|A_t-\bar{A}\|_{\infty}$ is the variance of the game matrices with $\bar{A}=\frac{1}{T}\sum_{t=1}^TA_t$ being the average game matrix and $(x^*,y^*)$ is a pair of Nash equilibrium of $\bar{A}$. In the special case where $A_t=A$ for all $t\in [T]$, we have dynamic NE-regret bounded by the maximum of the two individual regrets as $W_T=0$.
\end{lemma}
\begin{proof}
Suppose that $\sum_{t=1}^Tx_t^\top A_ty_t - \sum_{t=1}^T\min_x\max_y x^\top A_ty \geq 0$, then the dynamic NE-regret can be upper bounded as
\begin{align*}
    & \left|\sum_{t=1}^Tx_t^\top A_ty_t - \sum_{t=1}^T\min_x\max_y x^\top A_ty \right| \tag{let $(x_t^*,y_t^*)$ be the Nash equilibrium of $A_t$}\\
    & = \sum_{t=1}^Tx_t^\top A_ty_t - \sum_{t=1}^T x_t^{*\T} A_ty_t^*\tag{let $(x^*,y^*)$ be the Nash equilibrium of $\bar{A}$}\\
    & \leq \sum_{t=1}^Tx_t^\top A_ty_t - \sum_{t=1}^T x_t^{*\T} A_t y^* \tag{changing $y_t^*$ to $y^*$  decreases the game value w.r.t. $A_t$}\\
    & \leq \sum_{t=1}^Tx_t^\top A_ty_t - \sum_{t=1}^T x_t^{*\T} \bar{A} y^* + W_T \tag{shifting the payoff matrix from $A_t$ to $\bar{A}$}\\
    & \leq \sum_{t=1}^Tx_t^\top A_ty_t - \sum_{t=1}^T x^{*\T} \bar{A} y^* + W_T \tag{changing $x_t^*$ to $x^*$  decreases the game value w.r.t. $\bar{A}$}\\
    & \leq \sum_{t=1}^Tx_t^\top A_ty_t - \sum_{t=1}^T x^{*\T} \bar{A} y_t + W_T \tag{changing $y^*$ to $y_t$  decreases the game value w.r.t. $\bar{A}$}\\
    & \leq \sum_{t=1}^Tx_t^\top A_ty_t - \sum_{t=1}^T x^{*\T} A_t y_t + 2W_T. \tag{shifting the payoff matrix from $\bar{A}$ to $A_t$}
\end{align*}
Similarly, when $\sum_{t=1}^Tx_t^\top A_ty_t - \sum_{t=1}^T\min_x\max_y x^\top A_ty \leq 0$, we can verify that
\begin{align*}
    & \left|\sum_{t=1}^Tx_t^\top A_ty_t - \sum_{t=1}^T\min_x\max_y x^\top A_ty \right| \tag{let $(x_t^*,y_t^*)$ be the Nash equilibrium of $A_t$}\\
    & = \sum_{t=1}^T x_t^{*\T} A_ty_t^* - \sum_{t=1}^Tx_t^\top A_ty_t \tag{let $(x^*,y^*)$ be the Nash equilibrium of $\bar{A}$} \\
    & \leq \sum_{t=1}^T x^{*\T} A_ty_t^* - \sum_{t=1}^Tx_t^\top A_ty_t \tag{changing $x_t^*$ to $x^*$ increases the game value w.r.t. $A_t$}\\
    & \leq \sum_{t=1}^T x^{*\T} \bar{A}y_t^* - \sum_{t=1}^Tx_t^\top A_ty_t + W_T \tag{shifting the payoff matrix from $A_t$ to $\bar{A}$}\\
    & \leq \sum_{t=1}^T x^{*\T} \bar{A}y^* - \sum_{t=1}^Tx_t^\top A_ty_t + W_T \tag{changing $y_t^*$ to $y^*$  increases the game value w.r.t. $\bar{A}$}\\
    & \leq \sum_{t=1}^T x_t^{\T} \bar{A}y^* - \sum_{t=1}^Tx_t^\top A_ty_t + W_T \tag{changing $x^*$ to $x_t$  increases the game value w.r.t. $\bar{A}$}\\
    & \leq \sum_{t=1}^T x_t^{\T} A_t y^* - \sum_{t=1}^Tx_t^\top A_ty_t + 2W_T. \tag{shifting the payoff matrix from $\bar{A}$ to $A_t$}
\end{align*}
Combining the two cases yields the desired result.
\end{proof}

Next, we present the following stability lemma in terms of the non-stationarity measure $W_T$ (that is, the payoff variance). We give the stability upper bounds from the aspects of meta-algorithm and final decisions. Again, we assume that the $\drvu$ parameters $(\alpha,\beta,\gamma)$ are all $\wt{\Theta}(1)$.
\begin{lemma}[Payoff-variance stability]
\label{lemma:stability-payoff-variation}
Suppose that $x$-player follows~\pref{alg:x-player} and $y$-player follows~\pref{alg:y-player}. Then, the following inequalities hold simultaneously.
In the meta-algorithm aspect, we have
\begin{equation}
\label{eq:stability-var-variation-meta}
\sum_{t=1}^T\|p_t-\wh{p}_{t+1}\|_2^2 \leq \otil\left(1+W_T\right),~~ \sum_{t=1}^T\|q_t-\wh{q}_{t+1}\|_2^2 \leq \otil\left(1+W_T\right);
\end{equation}
in the final decision aspect, we have
\begin{align}
    \label{eq:stability-var-variation-overall}
    \sum_{t=2}^T \norm{x_t - x_{t-1}}_1^2 \leq \otil\left(1+W_T\right),~~ \sum_{t=2}^T \norm{y_t - y_{t-1}}_1^2 \leq \otil\left(1+W_T\right).
\end{align}
\end{lemma}

\begin{proof}
Let $\bar{A}=\frac{1}{T}\sum_{t=1}^TA_t$ denote the average game matrix, and let $(x^*, y^*)$ be the Nash equilibrium of $\bar{A}$. Then according to the saddle point property of $(x^*, y^*)$, we have for any $x_t\in \Delta_m$ and $y_t\in \Delta_n$, $x_t^\T \bar{A} y^* - x^{*\T} \bar{A} y_t \geq 0$. Therefore, based on~\pref{lemma:NE-variation-dynamic-regret} with $u_t=y^*$ and $v_t=x^*$ for all $t\in [T]$, we have
\begin{align*}
    0& \leq \sum_{t=1}^T \x_t^\T \bar{A} \y^* - \sum_{t=1}^T \x^{*\T} \bar{A} \y_t\\
    & \leq \sum_{t=1}^T \x_t^\T \bar{A} \y^* - \sum_{t=1}^T \x^\T \bar{A} \y_t + \sum_{t=1}^T \x_t^\T \bar{A} \y_t - \sum_{t=1}^T \x_t^{*\T} \bar{A} \y_t + 2W_T \\
    &\leq \otil\left(\frac{\alpha}{\eta_i^x}+\frac{\alpha}{\eta_j^y}+\beta(\eta_i^x+\eta_j^y)V_T\right) + \eta_i^xc\beta\sum_{t=2}^T\|y_t-y_{t-1}\|_1^2 + \eta_j^yc\beta\sum_{t=2}^T\|x_t-x_{t-1}\|_1^2\\
    &\qquad + \left(\lambda-\frac{\gamma}{\eta_i^x}\right)\sum_{t=2}^T\|x_{t,i}-x_{t-1,i}\|_1^2+\left(\lambda-\frac{\gamma}{\eta_j^y}\right)\sum_{t=2}^T\|y_{t,j}-y_{t-1,j}\|_1^2 \\
    &\qquad -L\sum_{t=1}^T(\|p_t-\wh{p}_{t+1}\|_2^2+\|p_t-\wh{p}_t\|_2^2+\|q_t-\wh{q}_{t+1}\|_2^2+\|q_t-\wh{q}_t\|_2^2) \\
    &\qquad - \lambda\sum_{t=2}^T\sum_{i\in S_x}p_{t,i}\|x_{t,i}-x_{t-1,i}\|_1^2 - \lambda\sum_{t=2}^T\sum_{j\in S_y}q_{t,j}\|y_{t,j}-y_{t-1,j}\|_1^2 + 2W_T\\
    &\leq \otil\left(\frac{\alpha}{\eta_i^x}+\frac{\alpha}{\eta_j^y}+\beta(\eta_i^x+\eta_j^y)V_T\right) + 2\eta_i^xc\beta\sum_{t=2}^T\|q_t-q_{t-1}\|_1^2 + 2\eta_i^xc\beta\sum_{t=2}^T\sum_{j\in \calS_y}\|y_{t,j}-y_{t-1,j}\|_1^2\\
    &\qquad+2\eta_j^yc\beta\sum_{t=2}^T\|p_t-p_{t-1}\|_1^2 +
    2\eta_j^yc\beta\sum_{t=2}^T\sum_{i\in \calS_x}\|x_{t,i}-x_{t-1,i}\|_1^2 \\
    &\qquad+ \left(\lambda-\frac{\gamma}{\eta_i^x}\right)\sum_{t=2}^T\|x_{t,i}-x_{t-1,i}\|_1^2+\left(\lambda-\frac{\gamma}{\eta_j^y}\right)\sum_{t=2}^T\|y_{t,j}-y_{t-1,j}\|_1^2 \\
    &\qquad -\frac{L}{2}\sum_{t=1}^T(\|p_t-\wh{p}_{t+1}\|_2^2+\|p_t-\wh{p}_t\|_2^2+\|q_t-\wh{q}_{t+1}\|_2^2+\|q_t-\wh{q}_t\|_2^2)-\frac{L}{4}\sum_{t=2}^T(\|p_t-p_{t-1}\|_2^2+\|q_t-q_{t-1}\|_2^2)\\
    &\qquad - \lambda\sum_{t=2}^T\sum_{i\in S_x}p_{t,i}\|x_{t,i}-x_{t-1,i}\|_1^2 - \lambda\sum_{t=2}^T\sum_{j\in S_y}q_{t,j}\|y_{t,j}-y_{t-1,j}\|_1^2 + 2W_T\\
    & \leq \otil\left(\frac{\alpha}{\eta_i^x}+\frac{\alpha}{\eta_j^y}+\beta(\eta_i^x+\eta_j^y)V_T\right) + \left(2\eta_i^xc\beta-\frac{L}{4}\right)\sum_{t=2}^T\|q_t-q_{t-1}\|_1^2 +\left(2\eta_j^yc\beta-\frac{L}{4}\right)\sum_{t=2}^T\|p_t-p_{t-1}\|_1^2\\
    &\qquad+ \left(2\eta_i^xc\beta-\lambda\right)\sum_{t=2}^T\sum_{j\in \calS_y}q_{t,j}\|y_{t,j}-y_{t-1,j}\|_1^2+
    \left(2\eta_j^yc\beta-\lambda\right)\sum_{t=2}^T\sum_{i\in \calS_x}p_{t,i}\|x_{t,i}-x_{t-1,i}\|_1^2\\
    &\qquad + \left(\lambda-\frac{\gamma}{\eta_i^x}\right)\sum_{t=2}^T\|x_{t,i}-x_{t-1,i}\|_1^2+\left(\lambda-\frac{\gamma}{\eta_j^y}\right)\sum_{t=2}^T\|y_{t,j}-y_{t-1,j}\|_1^2 \\
    &\qquad -\frac{L}{2}\sum_{t=1}^T(\|p_t-\wh{p}_{t+1}\|_2^2+\|p_t-\wh{p}_t\|_2^2+\|q_t-\wh{q}_{t+1}\|_2^2+\|q_t-\wh{q}_t\|_2^2) + 2W_T.
\end{align*} 
Based on the choice of $L$ and $\lambda$, we can again verify the condition of ~\pref{eq:negative-coef}, which leads to the following inequalities:
\begin{align*}
    \sum_{t=2}^T\|p_t-p_{t-1}\|_1^2\leq \frac{8}{L}\cdot \otil\left(\frac{1}{\eta_{i}^x}+\frac{1}{\eta_{j}^y}+(\eta_i^x+\eta_j^y)V_T+W_T\right) \leq \otil\left(\sqrt{1+V_T}+W_T\right) \leq \otil\left(1+W_T\right),
\end{align*}
where the second last inequality is because we pick $\eta_i^x$ and $\eta_j^y$ to be the one such that $\eta_i^x\in [\frac{1}{2}\eta_*^x, 2\eta_*^x]$, $\eta_j^y\in [\frac{1}{2}\eta_*^y, 2\eta_*^y]$ where $\eta_*^x=\min\left\{\sqrt{\frac{1}{1+V_T}}, \frac{1}{L}\right\}$ and $\eta_*^y=\min\left\{\sqrt{\frac{1}{1+V_T}}, \frac{1}{L}\right\}$. This is achievable based on the choice of our step size pool. The last inequality holds because of AM-GM inequality and $V_T \leq \O(W_T)$. Similarly, we have
\begin{align*}
    &\sum_{t=2}^T\|q_t-q_{t-1}\|_1^2\leq \otil\left(1+W_T\right),\\
    &\sum_{t=1}^T\|p_t-\wh{p}_{t+1}\|_1^2\leq \otil\left(1+W_T\right),\\
    &\sum_{t=1}^T\|q_t-\wh{q}_{t+1}\|_1^2\leq \otil\left(1+W_T\right),\\
    &\sum_{t=2}^T\sum_{i\in\calS_x}p_{t,i}\|x_{t,i}-x_{t-1,i}\|_1^2\leq \otil\left(1+W_T\right),\\
    &\sum_{t=2}^T\sum_{j\in\calS_y}q_{t,j}\|y_{t,j}-y_{t-1,j}\|_1^2\leq \otil\left(1+W_T\right).
\end{align*}
Based on the above results, we further have
\begin{align*}
    \sum_{t=2}^T \norm{x_t - x_{t-1}}_1^2 \leq &  2 \sum_{t=2}^T \norm{p_t - p_{t-1}}_1^2 + 2 \sum_{t=2}^T \sum_{i=1}^N p_{t,i}\norm{x_{t,i} - x_{t-1,i}}_1^2 \leq \otil\left(1+W_T\right),\\
    \sum_{t=2}^T \norm{y_t - y_{t-1}}_1^2 \leq &  2 \sum_{t=2}^T \norm{q_t - q_{t-1}}_1^2 + 2 \sum_{t=2}^T \sum_{i=1}^N q_{t,i}\norm{y_{t,i} - y_{t-1,i}}_1^2 \leq \otil\left(1+W_T\right),
\end{align*}
which completes the proof.
\end{proof}

Building upon the stability of the decisions of both $x$-player and $y$-player proven in~\pref{lemma:stability-NE-variation} and~\pref{lemma:stability-payoff-variation}, we further show the variation of the (gradient) feedback received by both $x$-player and $y$-player.

\begin{lemma}
\label{lemma:gradient-variation}
Suppose $x$-player follows~\pref{alg:x-player} and $y$-player follows~\pref{alg:y-player}. Then, the gradient variation can be bounded as follows:
\begin{equation}
    \label{eq:gradient-variation}
    \begin{split}
    \sum_{t=2}^T \norm{A_t y_t - A_{t-1}y_{t-1}}_\infty^2 \leq {}& \Ot\Big( 1 + V_T + \min\{P_T, W_T\}\Big),\\
    \sum_{t=2}^T \norm{x_t^\T A_t - x_{t-1}^\T A_{t-1}}_\infty^2 \leq {}& \Ot\Big( 1 + V_T + \min\{P_T, W_T\}\Big).
    \end{split}
\end{equation}
\end{lemma}
\begin{proof}
The gradient variation of the $x$-player can be upper bounded as follows:
\begin{align*}
    \sum_{t=2}^T \norm{A_t y_t - A_{t-1}y_{t-1}}_\infty^2 \leq {}& 2 \sum_{t=2}^T \norm{A_t y_t - A_{t-1}y_{t}}_\infty^2 + 2 \sum_{t=2}^T \norm{A_{t-1} y_t - A_{t-1}y_{t-1}}_\infty^2\\
    \leq {}& 2 \sum_{t=2}^T \norm{A_t - A_{t-1}}_\infty^2 + \order\left( \sum_{t=2}^T \norm{y_t - y_{t-1}}_1^2\right)\\
    \leq {}& \O(V_T) + \Ot\left(\min\{\sqrt{1+V_T + P_T} + P_T,  1+ W_T\}\right)\\
    \leq {} & \Ot\left(\min\{1 +V_T + P_T,  1+ V_T + W_T\}\right)
\end{align*}
where the second last step holds by~\pref{lemma:stability-NE-variation} and~\pref{lemma:stability-payoff-variation}, and the last step makes use of AM-GM inequality. A similar argument can be applied to upper bound $\sum_{t=2}^T \norm{x_t^\T A_t - x_{t-1}^\T A_{t-1}}_\infty^2$. This ends the proof.
\end{proof}

The following lemma establishes a general result to relate the function-value difference between the sequence of piecewise minimizers and the sequence of each-round minimizers.
\begin{lemma}
\label{lemma:function-variation}
Let $A_t\in \mathbb{R}^{m\times n}$ and $y_t\in \Delta_{n}$ for all $t\in [T]$ with $\max_{t\in[T]}\|A_t\|_{\infty}\leq 1$. Let $\I_1,\I_2,\ldots,\I_K$ be an even partition of the total horizon $[T]$ with $\abs{\I_k} = \Delta$ for $k = 1,\ldots,K$ (for simplicity, suppose the time horizon $T$ is divisible by epoch length $\Delta$). Denote $\bar{x}_t^*=\argmin_{x\in\Delta_m}x^\top A_ty_t$ for any $t\in [T]$ and denote $u_t \triangleq \argmin_{x\in\Delta_m} \sum_{\tau\in\I_k} x^\T A_\tau y_\tau$ for any $t \in \I_k$. Then, we have
\begin{equation}
    \label{eq:function-varition}
    \sum_{t=1}^T u_t^\top A_t y_t - \sum_{t=1}^T \bar{x}_t^{*\top} A_t y_t \leq 2 \Delta \sum_{t=2}^{T} \norm{A_t y_t - A_{t-1}y_{t-1}}_{\infty}.
\end{equation}
\end{lemma}

\begin{proof}
The proof follows the analysis of function-variation type worst-case dynamic regret~\citep{UAI'20:simple}. For convenience, we introduce the notation $f_t(x)\triangleq x^\T A_t y_t$ to denote the each-round online function of $x$-player. Then,
\begin{align*}
 & \sum_{i=1}^{T} f_t(u_t) - \sum_{i=1}^{T} f_t(\bar{x}_t^*) = \sum_{k=1}^{K} \sum_{t \in \I_k} \left(f_t(u_t) - f_t(\bar{x}_t^*)\right) \\
 & \leq 2 \Delta \cdot  \sum_{k=1}^{K} \sum_{t \in \I_k} \norm{A_t y_t - A_{t-1}y_{t-1}}_{\infty} = 2 \Delta \sum_{t=2}^{T} \norm{A_t y_t - A_{t-1}y_{t-1}}_{\infty},
\end{align*}
where the inequality holds because of the setting of $\abs{\I_k} = \Delta$ and the following fact about the instantaneous quantity: 
\begin{align*}
    f_t(u_t) - f_t(\bar{x}_{t}^{*}) & \leq f_t(\bar{x}_{t_1}^*) - f_t(\bar{x}_{t}^{*}) \tag*{(denote by $t_1$ the starting time stamp of $\I_k$)} \\
    & = f_t(\bar{x}_{t_1}^*) - f_{t_1}(\bar{x}_{t_1}^*) + f_{t_1}(\bar{x}_{t_1}^*) - f_t(\bar{x}_{t}^{*}) \\
    & \leq f_t(\bar{x}_{t_1}^*) - f_{t_1}(\bar{x}_{t_1}^*) + f_{t_1}(\bar{x}_{t}^*) - f_t(\bar{x}_{t}^{*}) \tag*{(by optimality of $\bar{x}_{t_1}^*$)}\\
    & \leq 2 \sum_{t \in \I_k} \sup_{x} \abs{f_t(x) - f_{t-1}(x)} \\
    & = 2 \sum_{t \in \I_k} \sup_{x} \left\vert x^\T (A_t y_t - A_{t-1}y_{t-1})\right\vert \\
    & \leq 2 \sum_{t \in \I_k} \norm{A_t y_t - A_{t-1}y_{t-1}}_{\infty}.
\end{align*}
Hence we finish the proof.
\end{proof}
\section{Technical Lemmas}
\label{sec:technical_lemmas}
\begin{lemma}[{Bregman proximal inequality~\citep[Lemma 3.2]{OPT'93:Bregman}}]
\label{lemma:bregman-divergence}
Let $\X$ be a convex set in a Banach space. Let $f: \X \mapsto \R$ be a closed proper convex function on $\X$. Given a convex regularizer $\psi:\X \mapsto \R$, we denote its induced Bregman divergence by $D_\psi(\cdot,\cdot)$. Then, any update of the form $\x_k = \argmin_{\x \in \X} \{ f(\x) + D_\psi(\x,\x_{k-1})\}$ satisfies the following inequality for any $\u \in \X$,
\begin{equation}
  f(\x_k) - f(\u) \leq D_\psi(\u, \x_{k-1}) - D_\psi(\u, \x_{k}) - D_\psi(\x_k, \x_{k-1}).
\end{equation}
\end{lemma}

\begin{lemma}[{stability lemma~\citep[Proposition 7]{COLT'12:variation-Yang}}]
\label{lemma:stability-OMD}
Let $\x_*=\argmin_{\x\in\X}\inner{a,\x}+D_\psi(x,c)$ and $\x'_* = \argmin_{\x\in\X}\inner{a',\x}+D_\psi(\x,c)$. When the regularizer $\psi:\X\mapsto\R$ is a 1-strongly convex function with respect to the norm $\| \cdot \|$, we have $\norm{\x_*-\x'_*}\leq\norm{(\nabla \psi(c)-a)-(\nabla \psi(c)-a')}_*= \norm{a-a'}_*$.
\end{lemma}

\begin{lemma}[{variant of self-confident tuning~\citep[Lemma 4.8]{UAI'19:FIRST-ORDER}}]
\label{lem:self-confident-variant}
  Let $a_1, a_2, \ldots, a_T$ be non-negative real numbers. Then
  \begin{equation*}
      \sum_{t=1}^{T} \frac{a_{t}}{\sqrt{1+\sum_{s=1}^{t-1} a_{s}}} \leq  4\sqrt{1+\sum_{t=1}^T a_t} + \max_{t\in[T]} a_t.
  \end{equation*}
\end{lemma}
\section{Experiment}
\label{appendix:experiments}
In this section, we further provide empirical studies on the performance of our proposed algorithm in time-varying games.

We construct an environment such that $P_T=\Theta(\sqrt{T})$, $W_T=\Theta(T^{\frac{3}{4}})$, and $V_T=\Theta(\sqrt{T})$. Under this environment, our theoretical results indicate that $\max\{\Reg^x_T, \Reg^y_T\}\leq \otil(T^{\frac{1}{4}})$, $\nereg\leq\otil(\sqrt{T})$ and $\gap\leq \otil(T^{\frac{7}{8}})$. Our empirical results validate the effectiveness of our algorithm in this environment, and in fact its performance is even better than the theoretical upper bounds, which also encourage us to investigate better theoretical guarantees in the future.

The environment setup is as follows. We set the size of game matrix to be $m \times n$ with $m=2$ and $n = 2$. The total time horizon is set as $T = 2\times 10^6$. Define 
\begin{align*}
    A_0=\begin{pmatrix}
\nicefrac{1}{2} & \nicefrac{1}{2} \\
- \nicefrac{1}{2} & -\nicefrac{1}{2}
\end{pmatrix},~~A_1=\begin{pmatrix}
-1 & -1 \\
1 & 1
\end{pmatrix},~~E=\begin{pmatrix}
\nicefrac{1}{3} & -\nicefrac{1}{2}\\
\nicefrac{1}{3} & -\nicefrac{1}{2}
\end{pmatrix}.
\end{align*}
Set $T_0=2\lfloor{T^{\nicefrac{1}{2}}}\rfloor$. The scheduling of the game matrices is separated into $K=4$ epochs and during each epoch $k$, we have
\begin{align*}
    A_t=\begin{cases}
    A_0+(-1)^t\cdot E, & \mbox{$t\in \left[\frac{(k-1)T}{K}+1,\frac{(k-1)T}{K}+T_0\right]$}, \vspace{2mm}\\
    \left(\frac{1}{2}+\left(\frac{1}{2}-(-1)^t\cdot T^{-\frac{1}{4}}\right)\right)A_1, & \mbox{$t\in \left[\frac{(k-1)T}{K}+T_0+1, \frac{kT}{K}\right]$}.
    \end{cases}
\end{align*}

Specifically, during the first phase of each epoch, when $t$ is even, $A_t=A_0+E$, in which $x$-player's Nash equilibrium is $x_t^*=(0,1)$ and $y_t^*=(1,0)$; when $t$ is odd, $A_t=A_0-E$, in which $x$-player's Nash equilibrium is $x_t^*=(0,1)$ but $y_t^*=(0,1)$. Therefore, during the first phase, the variation of Nash equilibrium is $\Theta(1)$ per consecutive rounds. Also, the variation of the game matrix is also $\Theta(1)$ per consecutive rounds. During the second phase of each epoch, the Nash equilibrium of $A_t$ keeps the same but the variation of the game matrix is $\Theta(T^{-\frac{1}{2}})$ per consecutive rounds. Therefore, over the total $T$ horizon, $P_T=\Theta(T_0)=\Theta(\sqrt{T})$, $V_T=\Theta(\sqrt{T})$. Direct calculation also shows that $W_T=\Theta(T^{\frac{3}{4}})$.

To show the necessity of the two-layer structure, we compare the performance of our two-layer algorithm (\pref{alg:x-player} for $x$-player and~\pref{alg:y-player} for $y$-player) with one single base-learner with a fixed step size chosen specifically to minimize each performance measure. More concretely, we choose the base-learner as optimistic Hedge with a fixed-share update, which satisfies $\drvu(\eta)$ property as we prove in~\pref{appendix:optimistic-Hedge}. As mentioned in~\pref{sec: algo overview}, this base-learner with a specific choice of the step size can indeed achieve a favorable bound for a specific performance measure. Specific to our environment configuration, to achieve the best individual regret bound, the step size needs to be chosen as $\Theta(1/\sqrt{1+P_T+V_T})=\Theta(T^{-\frac{1}{4}})$, while to achieve the best dynamic NE-regret bound, the step size should be chosen as $\Theta(\sqrt{P_T/(1+P_T+V_T)})=\Theta(1)$, which means that the base-learner cannot guarantee the desired bounds for all the three measures simultaneously. 

Concretely, we implement~\pref{alg:x-player} for $x$-player and~\pref{alg:y-player} for $y$-player with $L=4$ and the step size pool $\eta_i=\frac{2^{i-1}}{4\sqrt{T}}$ for both $x$-player and $y$-player. In this case, the number of base-learners (i.e., the size of step size pool) is $N=\lfloor{\frac{1}{2}\log_2 T}\rfloor+1=11$. In addition, we also run our base-learner with each single $\eta_i$ separately. We pick the best step size for each measure respectively and show how these single step size base-learners perform in all the three measure.

\pref{fig:meta-base} plots the experimental results with respect to all the three measures (individual regret, dynamic NE-regret, and duality gap). From the results, we can see that our algorithm's performance on all of the three measure is comparable to (or even better than) the base-learner with the corresponding best step size tuning, while the base-learners specifically tuned for a single measure cannot perform well on all of the three measures simultaneously, which supports our theoretical results proven in~\pref{sec: regret guarantee} and also validate the necessity of a two-layer structure of our proposed algorithm.

\begin{figure}[!t]
\centering
\includegraphics[width=0.33\textwidth]{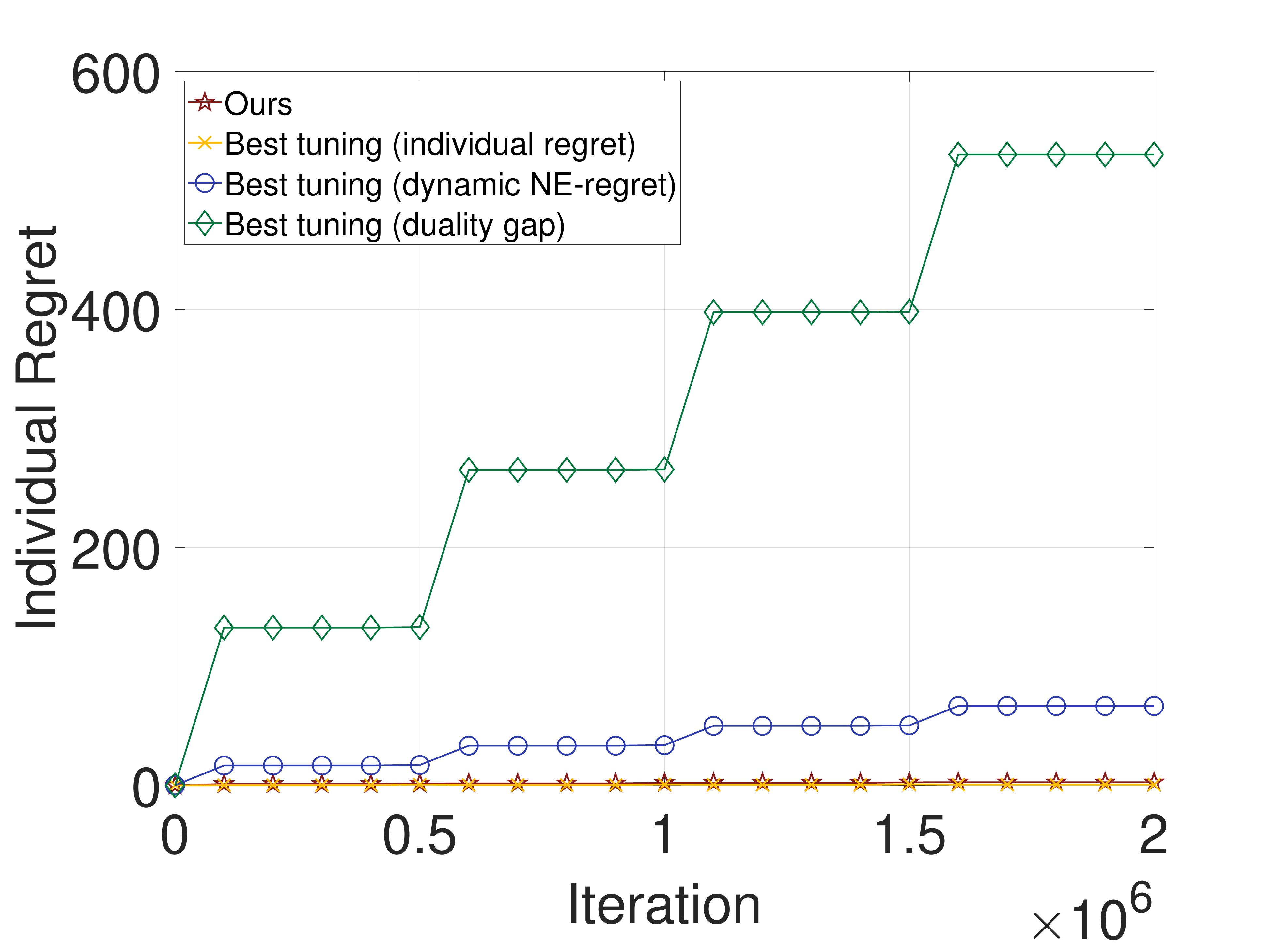}
\includegraphics[width=0.33\textwidth]{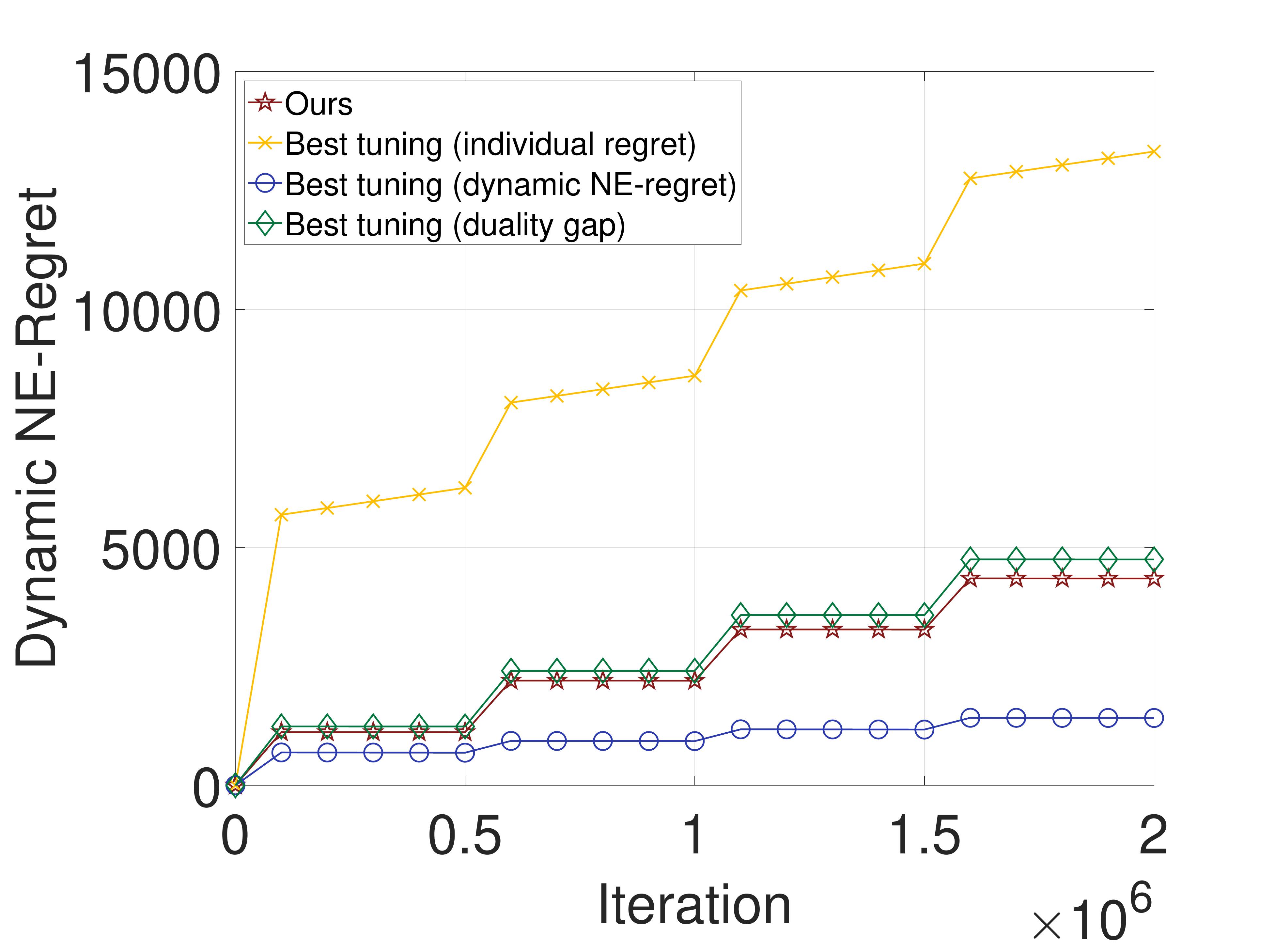}
\includegraphics[width=0.33\textwidth]{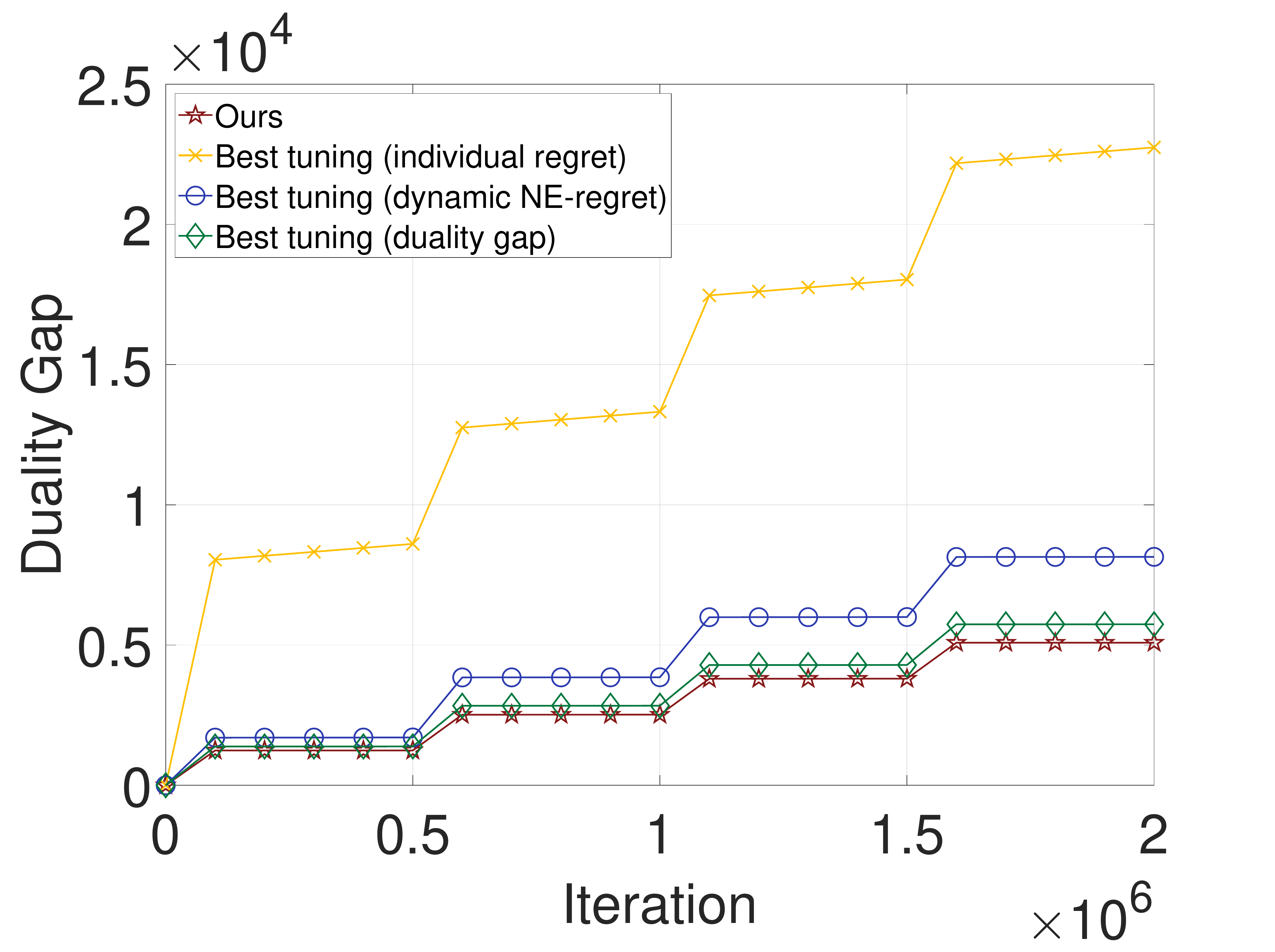}
\vspace{-5mm}
\caption{Empirical results of our algorithm (red line) compared with the base-learners with different step size choices. ``Best tuning (measure)'' represents the curve of the base-learner with the step size choice that performs the best with respect to this ``measure'' compared with all the other step size configurations. The three figures show that our algorithm's performance under all of the three measure is comparable to (or even better than) the base-learner with the corresponding best step size tuning. On the other hand, the base-learner specifically tuned for a single measure cannot perform well in all of the three measures simultaneously. }
\label{fig:meta-base}
\end{figure}

\end{document}